\documentclass[letter,11pt]{article}

\usepackage{amsmath}
\usepackage{amsthm}
\usepackage{amssymb}
\usepackage{graphicx}
\usepackage{subfigure}
\usepackage[numbers]{natbib}
\usepackage{algorithm}
\usepackage{algorithmic}
\usepackage[hidelinks]{hyperref}
\usepackage{authblk}
\usepackage{fullpage}
\usepackage{multirow}
\usepackage{natbib}

\usepackage[shortlabels,inline]{enumitem}
\setlist[enumerate,1]{label={\normalfont(\roman*)},leftmargin=1.6cm}

\newtheorem{theorem}{Theorem}
\newtheorem{corollary}[theorem]{Corollary}
\newtheorem{proposition}[theorem]{Proposition}
\newtheorem{lemma}[theorem]{Lemma}
\newtheorem{definition}[theorem]{Definition}
\newtheorem{remark}{Remark}

\renewcommand{\eqref}[1]{Eq.~(\ref{#1})}
\newcommand{\figref}[1]{Fig.~\ref{#1}}

\title{Two-Layer Feature Reduction for Sparse-Group Lasso via Decomposition of Convex Sets}
\author{Jie Wang}
\author{Jieping Ye}
\affil{Computer Science and Engineering, Arizona State University,
            USA}
\date{}

\begin{document}

\maketitle

\begin{abstract}
Sparse-Group Lasso (SGL) has been shown to be a powerful regression technique for simultaneously discovering group and within-group sparse patterns by using a combination of the $\ell_1$ and $\ell_2$ norms. However, in large-scale applications, the complexity of the regularizers entails great computational challenges. In this \mbox{paper}, we propose a novel \textbf{t}wo-\textbf{l}ayer \textbf{f}eature \textbf{re}duction method (TLFre) for SGL via a decomposition of its dual feasible set. The \mbox{two}-layer reduction is able to \mbox{quickly} identify the inactive groups and the inactive features, respectively, which are guaranteed to be absent from the sparse representation and can be removed from the optimization. Existing feature reduction methods are only applicable for sparse models with one sparsity-inducing regularizer. To our best knowledge, TLFre is \emph{the first one} that is capable of dealing with \emph{multiple} sparsity-inducing \mbox{regularizers}. Moreover, TLFre has a very low computational cost and can be integrated with any existing solvers. We also develop a screening method---called DPC (\textbf{d}ecom\textbf{p}osition of \textbf{c}onvex set)---for the nonnegative Lasso problem. Experiments on both synthetic and real data sets show that TLFre and DPC improve the efficiency of SGL and nonnegative Lasso by several orders of magnitude.
\end{abstract}


\section{Introduction}

Sparse-Group Lasso (SGL) \cite{Friedman,Simon2013} is a powerful regression technique in identifying \mbox{important} \mbox{groups} and features simultaneously. To yield sparsity at both group and individual feature \mbox{levels}, SGL combines the Lasso \cite{Tibshirani1996} and group Lasso \cite{Yuan2006} penalties. In recent years, SGL has found great success in a wide range of applications, including but not limited to machine learning \cite{Vidyasagar2014,Yogatama2014}, signal processing \cite{Sprechmann2011}, bioinformatics \cite{Peng2010} etc. Many research efforts have been devoted to developing efficient solvers for SGL \cite{Friedman,Simon2013,Liu2010,Vincent2014}. However, when the feature dimension is extremely high, the complexity of the SGL regularizers imposes great computational challenges. Therefore, there is an increasingly urgent need for nontraditional techniques to address the challenges posed by the massive volume of the data sources.

Recently, El Ghaoui et al. \cite{ElGhaoui2012} proposed a promising feature reduction method, called \emph{SAFE screening}, to screen out the so-called \emph{inactive} features, which have zero coefficients in the solution, from the optimization. Thus, the size of the data matrix needed for the training phase can be \mbox{significantly} reduced, which may lead to substantial improvement in the efficiency of solving sparse models. \mbox{Inspired} by SAFE, various exact and heuristic feature screening methods have been proposed for many sparse models such as Lasso \cite{Wang2013,Liu2014,Tibshirani11,Xiang2012}, group Lasso \cite{Wang2013,Wang,Tibshirani11}, etc. It is worthwhile to mention that the discarded features by exact feature screening methods such as SAFE \cite{ElGhaoui2012}, DOME \cite{Xiang2012} and EDPP \cite{Wang2013} are guaranteed to have zero coefficients in the solution. However, heuristic feature screening methods like Strong Rule \cite{Tibshirani11} may mistakenly discard features which have \mbox{nonzero} coefficients in the solution. More recently, the idea of exact feature screening has been extended to exact sample screening, which screens out the nonsupport vectors in SVM \cite{Ogawa2013,Wang2014} and LAD \cite{Wang2014}. As a promising data reduction tool, exact feature/sample screening would be of great practical importance because they can effectively reduce the data size without sacrificing the optimality \cite{Ogawa2014}. 

However, all of the existing feature/sample screening methods are only applicable for the sparse models with one sparsity-inducing regularizer. In this paper, we propose an exact two-layer feature screening method, called TLFre, for the SGL problem. The two-layer reduction is able to quickly identify the inactive groups and the inactive features, respectively, which are guaranteed to have zero coefficients in the solution. To the best of our knowledge, TLFre is the first screening method which is capable of dealing with multiple sparsity-inducing regularizers. 

We note that most of the existing exact feature screening methods involve an estimation of the dual optimal solution. The difficulty in developing screening methods for sparse models with multiple sparsity-inducing regularizers like SGL is that the dual feasible set is the sum of simple convex sets. Thus, to determine the feasibility of a given point, we need to know if it is decomposable with respect to the summands, which is itself a nontrivial problem (see Section \ref{section:basics}). One of our major contributions is that we derive an elegant decomposition method of any dual feasible solutions of SGL via the framework of Fenchel's duality (see Section \ref{section:Fenchel_Dual_SGL}). Based on the Fenchel's dual problem of SGL, we motivate TLFre by an in-depth exploration of its geometric properties and the optimality conditions in Section \ref{section:screening_rules_SGL}.
We derive the set of the regularization parameter values corresponding to zero solutions.
To develop TLFre, we need to estimate the upper bounds involving the dual optimal solution. To this end, we first give an accurate estimation of the dual optimal solution via the normal cones. Then, we formulate the estimation of the upper bounds via nonconvex optimization problems. We show that these nonconvex problems admit closed form solutions.

The rest of this paper is organized as follows. In Section \ref{section:basics}, we briefly review some basics of the SGL problem. We then derive the Fenchel's dual of SGL with nice geometric properties under the elegant framework of Fenchel's Duality in Section \ref{section:Fenchel_Dual_SGL}. In Section \ref{section:screening_rules_SGL}, we develop the TLFre screening rule for SGL. To demonstrate the flexibility of the proposed framework, we extend TLFre to the nonnegative Lasso problem in Section \ref{section:nnlasso}. Experiments in Section \ref{section:experiments} on both synthetic and real data sets demonstrate that the speedup gained by the proposed screening rules in solving SGL and nonnegative Lasso can be orders of magnitude.

\textbf{Notation}: Let $\|\cdot\|_1$, $\|\cdot\|$ and $\|\cdot\|_{\infty}$ be the $\ell_1$, $\ell_2$ and $\ell_{\infty}$ norms, respectively. Denote by $\mathcal{B}_1^n$, $\mathcal{B}^n$, and $\mathcal{B}_{\infty}^n$ the unit $\ell_1$, $\ell_2$, and $\ell_{\infty}$ norm balls in $\mathbb{R}^n$ (we omit the superscript if it is clear from the context). For a set $\mathcal{C}$, let $\textup{int}\,\mathcal{C}$ be its interior. If $\mathcal{C}$ is closed and convex, we define the projection operator as
$\mathbf{P}_{\mathcal{C}}(\mathbf{w}):=\textup{argmin}_{\mathbf{u}\in\mathcal{C}}\|\mathbf{w}-\mathbf{u}\|$.
We denote by $\mathbf{I}_{\mathcal{C}}(\cdot)$ the indicator function of $\mathcal{C}$, which is $0$ on $\mathcal{C}$ and $\infty$ elsewhere.
Let $\Gamma_0(\mathbb{R}^n)$ be the class of proper closed convex functions on $\mathbb{R}^n$. For $f\in\Gamma_0(\mathbb{R}^n)$, let $\partial f$ be its subdifferential. The domain of $f$ is the set $\textup{dom}\,f:=\{\mathbf{w}:f(\mathbf{w})<\infty\}$.
For $\mathbf{w}\in\mathbb{R}^n$, let $[\mathbf{w}]_i$ be its $i^{th}$ component. For $\gamma\in\mathbb{R}$, let $\textup{sgn}(\gamma)=\textup{sign}(\gamma)$ if $\gamma\neq0$, and $\textup{sgn}(0)=0$. We define
\begin{align*}
	\textup{SGN}(\mathbf{w})=
	\left\{
	\mathbf{s}\in\mathbb{R}^n:[\mathbf{s}]_i\in
	\begin{cases}
		\textup{sign}([\mathbf{w}]_i),\hspace{2mm}\textup{if}\hspace{1mm}[\mathbf{w}]_i\neq0;\\
		[-1,1],\hspace{7mm}\textup{if}\hspace{1mm}[\mathbf{w}]_i=0.
	\end{cases}
	\right\}
\end{align*}
We denote by $\gamma_+=\max(\gamma,0)$. Then, the shrinkage operator $\mathcal{S}_{\gamma}(\mathbf{w}):\mathbb{R}^n\rightarrow\mathbb{R}^n$ with $\gamma\geq0$ is
\begin{align}\label{def:shrinkage_SGL}
	[\mathcal{S}_{\gamma}(\mathbf{w})]_i=(|[\mathbf{w}]_i|-\gamma)_+\textup{sgn}([\mathbf{w}]_i),\,i=1,\ldots,n.
\end{align}

\section{Basics and Motivation}\label{section:basics}

In this section, we briefly review some basics of SGL.
Let $\textbf{y}\in\mathbb{R}^N$ be the response vector and $\textbf{X}\in\mathbb{R}^{N\times p}$ be the matrix of features. With the group information available, the SGL problem \cite{Friedman} is
\begin{align}\label{prob:SGL0}
	\min_{\beta\in\mathbb{R}^p}\,\frac{1}{2}\left\|\textbf{y}-\sum\nolimits_{g=1}^G\textbf{X}_g\beta_g\right\|^2+\lambda_1\sum\nolimits_{g=1}^G\sqrt{n_g}\|\beta_g\|+\lambda_2\|\beta\|_1,
\end{align}
where $n_g$ is the number of features in the $g^{th}$ group, $\textbf{X}_g\in\mathbb{R}^{N\times n_g}$ denotes the predictors in that group with the corresponding coefficient vector $\beta_g$, and $\lambda_1,\lambda_2$ are positive regularization parameters. Without loss of generality, let $\lambda_1=\alpha\lambda$ and $\lambda_2=\lambda$ with $\alpha>0$. Then, problem (\ref{prob:SGL0}) becomes:
\begin{align}\label{prob:SGL}
	\min_{\beta\in\mathbb{R}^p}\,\frac{1}{2}\left\|\textbf{y}-\sum\nolimits_{g=1}^G\textbf{X}_g\beta_g\right\|^2+\lambda\left(\alpha\sum\nolimits_{g=1}^G\sqrt{n_g}\|\beta_g\|+\|\beta\|_1\right).
\end{align}
By the Lagrangian multipliers method \cite{Boyd04} (see the supplement), the dual problem of SGL is
\begin{align}\label{prob:dual_SGL_Lagrangian}
	\sup_{\theta}\,\left\{\tfrac{1}{2}\|\mathbf{y}\|^2-\tfrac{1}{2}\left\|\tfrac{\mathbf{y}}{\lambda}-\theta\right\|^2: \mathbf{X}_g^T\theta\in\mathcal{D}_g^{\alpha}:=\alpha\sqrt{n_g}\mathcal{B}+\mathcal{B}_{\infty},\,g=1,\ldots,G\right\}.
\end{align}
It is well-known that the dual feasible set of Lasso is the intersection of closed half spaces (thus a polytope); for group Lasso, the dual feasible set is the intersection of ellipsoids. Surprisingly, the geometric properties of these dual feasible sets play fundamentally important roles in most of the existing screening methods for sparse models with one sparsity-inducing regularizer \cite{Wang2014,Liu2014,Wang2013,ElGhaoui2012}.

When we incorporate multiple sparse-inducing regularizers to the sparse models, problem (\ref{prob:dual_SGL_Lagrangian}) indicates that the dual feasible set can be much more complicated.  Although (\ref{prob:dual_SGL_Lagrangian}) provides a geometric description of the dual feasible set of SGL, it is not suitable for further analysis. Notice that, \emph{even the feasibility of a given point $\theta$ is not easy to determine}, since it is nontrivial to tell if $\mathbf{X}_g^T\theta$ can be decomposed into $\mathbf{b}_1+\mathbf{b}_2$ with $\mathbf{b}_1\in \alpha\sqrt{n_g}\mathcal{B}$ and $\mathbf{b}_2\in\mathcal{B}_{\infty}$. Therefore, to develop screening methods for SGL, it is desirable to gain deeper understanding of the sum of simple convex sets.

In the next section, we analyze the dual feasible set of SGL in depth via the Fenchel's Duality Theorem. We show that for each $\mathbf{X}^T_g\theta\in\mathcal{D}_g^{\alpha}$, Fenchel's duality naturally leads to an explicit decomposition $\mathbf{X}^T_g\theta=\mathbf{b}_1+\mathbf{b}_2$, with one belonging to $\alpha\sqrt{n_g}\mathcal{B}$ and the other one belonging to $\mathcal{B}_{\infty}$. This lays the foundation of the proposed screening method for SGL.

\section{The Fenchel's Dual Problem of SGL}\label{section:Fenchel_Dual_SGL}
In Section \ref{subsection:Fenchel_Dual_SGL}, we derive the Fenchel's dual of SGL via Fenchel's Duality Theorem. We then motivate TLFre and sketch our approach in Section \ref{subsection:general_rule_SGL}. In Section \ref{subsection:lambda12_0_solution}, we discuss the \mbox{geometric} properties of the Fenchel's dual of SGL and derive the set of $(\lambda,\alpha)$ leading to zero solutions.

\subsection{The Fenchel's Dual of SGL via Fenchel's Duality Theorem}\label{subsection:Fenchel_Dual_SGL}

To derive the Fenchel's dual problem of SGL, we need the Fenchel's Duality Theorem as stated in Theorem \ref{thm:fenchel_duality}. The conjugate of $f\in\Gamma_0(\mathbb{R}^n)$ is the function $f^*\in\Gamma_0(\mathbb{R}^n)$ defined by
\begin{align}\label{eqn:conjugate}
	f^*(\mathbf{z})=\sup\nolimits_{\mathbf{w}}\,\langle \mathbf{w}, \mathbf{z}\rangle-f(\mathbf{w}).
\end{align}
\begin{theorem}\label{thm:fenchel_duality}\textup{[Fenchel's Duality Theorem]}
	Let $f\in\Gamma_0(\mathbb{R}^N)$, $\Omega\in\Gamma_0(\mathbb{R}^p)$, and $\mathcal{T}(\beta) = \mathbf{y} - \mathbf{X}\beta$ be an affine mapping from $\mathbb{R}^p$ to $\mathbb{R}^N$. Let $p^*, d^*\in[-\infty,\infty]$ be primal and dual values defined, respectively, by the Fenchel problems:
	\begin{align*}
		p^*=\inf\nolimits_{\beta\in\mathbb{R}^p}\,f(\mathbf{y}-\mathbf{X}\beta)+\lambda\Omega(\beta);\hspace{2mm}
		d^*=\sup\nolimits_{\theta\in\mathbb{R}^N}\,-f^*(\lambda\theta)-\lambda\Omega^*(\mathbf{X}^T\theta)+\lambda\langle \mathbf{y},\theta\rangle.
	\end{align*}
	One has $p^*\geq d^*$. If, furthermore, $f$ and $\Omega$ satisfy the condition
	$0\in\textup{int}\,\left(\textup{dom}\,f-\mathbf{y}+\mathbf{X}\textup
	{dom}\,\Omega\right)$,
	then the equality holds, i.e., $p^*=d^*$, and the supreme is attained in the dual problem if finite.
	%
	%
\end{theorem}

We omit the proof of Theorem \ref{thm:fenchel_duality} since it is a slight modification of Theorem 3.3.5 in \cite{Borwein2006}.

Let $f(\mathbf{w})=\frac{1}{2}\|\mathbf{w}\|^2$, and $\lambda\Omega(\beta)$ be the second term in (\ref{prob:SGL}). Then, SGL can be written as
\begin{align}\label{prob:general}
	\min\nolimits_{\beta}\,f(\mathbf{y}-\mathbf{X}\beta)+\lambda\Omega(\beta).
\end{align}
To derive the Fenchel's dual problem of SGL, Theorem \ref{thm:fenchel_duality} implies that we need to find $f^*$ and $\Omega^*$. It is well-known that $f^*(\mathbf{z})=\frac{1}{2}\|\mathbf{z}\|^2$. Therefore, we only need to find $\Omega^*$, where the concept \emph{infimal convolution} is needed:
\begin{definition}\textup{\cite{Bauschke2011}}\label{def:inf_conv}
	Let $h,g\in\Gamma_0(\mathbb{R}^n)$. The infimal convolution of $h$ and $g$ is defined by
	\begin{align}\label{eqn:inf_conv}
	(h\Box g)(\xi)=\inf\nolimits_{\eta}\,h(\eta)+g(\xi-\eta),
	\end{align}
	and it is exact at a point $\xi$ if there exists a $\eta^*(\xi)$ such that 
	\begin{align}\label{eqn:inf_conv_exact}
	(h\Box g)(\xi)=h(\eta^*(\xi))+g(\xi-\eta^*(\xi)).
	\end{align}
	$h\Box g$ is exact if it is exact at every point of its domain, in which case it is denoted by $h\boxdot g$.
\end{definition}

With the infimal convolution, we derive the conjugate function of $\Omega$ in Lemma \ref{lemma:conjugate_Omega_SGL}.

\begin{lemma}\label{lemma:conjugate_Omega_SGL}
Let $\Omega_1^{\alpha}(\beta)=\alpha\sum\nolimits_{g=1}^G\sqrt{n_g}\|\beta_g\|$, $\Omega_2(\beta)=\|\beta\|_1$ and $\Omega(\beta)=\Omega_1^{\alpha}(\beta)+\Omega_2(\beta)$. Moreover, let $\mathcal{C}_g^{\alpha}=\alpha\sqrt{n_g}\mathcal{B}\subset\mathbb{R}^{n_g}$, $g=1,\ldots,G$. Then, the following hold:
\begin{enumerate}	
\item $(\Omega_1^{\alpha})^*(\xi)=\sum\nolimits_{g=1}^G\mathbf{I}_{\mathcal{C}_g^{\alpha}}\left(\xi_g\right),\hspace{5mm} (\Omega_2)^*(\xi)=\sum\nolimits_{g=1}^{G}\mathbf{I}_{\mathcal{B}_{\infty}}\left(\xi_g\right)$,
\item $\Omega^*(\xi)=\left((\Omega_1^{\alpha})^*\boxdot(\Omega_2)^*\right)(\xi)=\sum\nolimits_{g=1}^G\,\mathbf{I}_{\mathcal{B}}\left(\frac{\xi_g-\mathbf{P}_{\mathcal{B}_{\infty}}(\xi_g)}{\alpha\sqrt{n_g}}\right),$
\end{enumerate}	
where $\xi_g\in\mathbb{R}^{n_g}$ is the sub-vector of $\xi$ corresponding to the $g^{th}$ group.
\end{lemma}

To prove Lemma \ref{lemma:conjugate_Omega_SGL}, we first cite the following technical result.

\begin{theorem}\label{thm:infconv_sum}\textup{\cite{Hiriart-Urruty2006}}
	Let $f_1,\cdots,f_k\in\Gamma_0(\mathbb{R}^n)$. Suppose there is a point in $\cap_{i=1}^k\textup{dom}\,f_i$ at which $f_1,\cdots,f_{k-1}$ is continuous. Then, for all $p\in\mathbb{R}^n$:
	\begin{align*}
		(f_1+\cdots+f_k)^*(p)=\min_{p_1+\cdots+p_k=p}[f_1^*(p_1)+\cdots+f_k^*(p_k)].
	\end{align*}
\end{theorem}

We now give the proof of Lemma \ref{lemma:conjugate_Omega_SGL}.
\begin{proof}
	The first part can be derived directly by the definition as follows:
	\begin{align*}
		(\Omega_1^{\alpha})^*(\xi)=&\sup_{\beta}\,\langle\beta,\xi\rangle-\Omega_1(\beta)=\sum_{g=1}^G\alpha\sqrt{n_g}\left(\sup_{\beta_g}\left\langle\beta_g,\frac{\xi_g}{\alpha\sqrt{n_g}}\right\rangle-\|\beta_g\|\right)\\
		=&\sum_{g=1}^G\alpha\sqrt{n_g}\mathbf{I}_{\mathcal{B}}\left(\frac{\xi_g}{\alpha\sqrt{n_g}}\right)=\sum_{g=1}^G\mathbf{I}_{\mathcal{B}}\left(\frac{\xi_g}{\alpha\sqrt{n_g}}\right)=\sum_{g=1}^G\mathbf{I}_{\mathcal{C}_g^{\alpha}}(\xi_g).
	\end{align*}
	\begin{align*}
		(\Omega_2)^*(\xi)=\sup_{\beta}\,\langle\beta,\xi\rangle-\Omega_2(\beta)=\mathbf{I}_{\mathcal{B}_{\infty}}\left({\xi}\right)=\sum_{g=1}^{G}\mathbf{I}_{\mathcal{B}_{\infty}}\left({\xi_g}\right).
	\end{align*}
	
	To show the second part, Theorem \ref{thm:infconv_sum} indicates that we only need to show $(\Omega_1^{\alpha})^*\Box(\Omega_2)^*(\xi)$ is exact (note that $\Omega_1^{\alpha}$ and $\Omega_2$ are continuous everywhere).
	Let us now compute $(\Omega_1^{\alpha})^*\Box(\Omega_2)^*$.
	\begin{align}\label{eqn:infconv_SGL}
		\left((\Omega_1)^*\Box(\Omega_2)^*\right)(\xi)=&\inf_{\eta}\,(\Omega_1)^*(\xi-\eta)+(\Omega_2)^*(\eta)\\ \nonumber
		=&\sum_{g=1}^G\,\inf_{\eta_g}\mathbf{I}_{\mathcal{B}}\left(\frac{\xi_g-\eta_g}{\alpha\sqrt{n_g}}\right)+\mathbf{I}_{\mathcal{B}_{\infty}}\left({\eta_g}\right)\\ \nonumber
		=&\sum_{g=1}^G\,\inf_{\|\eta_g\|_{\infty}\leq 1}\mathbf{I}_{\mathcal{B}}\left(\frac{\xi_g-\eta_g}{\alpha\sqrt{n_g}}\right)
	\end{align}
	To solve the optimization problem in (\ref{eqn:infconv_SGL}), i.e.,
	\begin{align}\label{prob:infconv_SGL}
		\mu_g^*=\inf_{\eta_g}\,\left\{\mathbf{I}_{\mathcal{B}}\left(\frac{\xi_g-\eta_g}{\alpha\sqrt{n_g}}\right): \|\eta_g\|_{\infty}\leq 1\right\},
	\end{align}
	we can consider the following problem
	\begin{align}\label{prob:projection_SGL}
		\nu_g^*=\inf_{\eta_g}\,\left\{\frac{1}{\alpha\sqrt{n_g}}\|\xi_g-\eta_g\|: \|\eta_g\|_{\infty}\leq 1\right\}.
	\end{align}
	We can see that the optimal solution of problem (\ref{prob:projection_SGL}) must also be an optimal solution of problem (\ref{prob:infconv_SGL}).
	Let $\eta^*_g(\xi_g)$ be the optimal solution of (\ref{prob:projection_SGL}). We can see that $\eta^*_g(\xi_g)$ is indeed the projection of $\xi_g$ on $\mathcal{B}_{\infty}$, which admits a closed form solution:
	\begin{align*}
		[\eta_g^*(\xi_g)]_i=[\mathbf{P}_{\mathcal{B}_{\infty}}(\xi_g)]_i=
		\begin{cases}
			1,\hspace{7.5mm}{\rm if}\hspace{1mm}[\xi_g]_i>1,\\
			[\xi_g]_i,\hspace{3mm}{\rm if}\hspace{1mm}|[\xi_g]_i|\leq1,\\
			-1,\hspace{5mm}{\rm if}\hspace{1mm}[\xi_g]_i<-1.
		\end{cases}
	\end{align*}
	Thus, problem (\ref{prob:infconv_SGL}) can be solved as
	\begin{align*}
		\mu^*_g=\mathbf{I}_{\mathcal{B}}\left(\frac{\xi_g-\mathbf{P}_{\mathcal{B}_{\infty}}(\xi_g)}{\alpha\sqrt{n_g}}\right).
	\end{align*}
	Hence, the infimal convolution in \eqref{eqn:infconv_SGL} is exact and Theorem \ref{thm:infconv_sum} leads to
	\begin{align}\label{eqn:infconv_nSGL2}
		\Omega^*(\xi)=\left((\Omega_1^{\alpha})^*\boxdot(\Omega_2)^*\right)(\xi)
		=\sum_{g=1}^G\,\mathbf{I}_{\mathcal{B}}\left(\frac{\xi_g-\mathbf{P}_{\mathcal{B}_{\infty}}(\xi_g)}{\alpha\sqrt{n_g}}\right),
	\end{align}
	which completes the proof.
\end{proof}

Note that $\mathbf{P}_{\mathcal{B}_{\infty}}(\xi_g)$ admits a closed form solution, i.e.,
$[\mathbf{P}_{\mathcal{B}_{\infty}}(\xi_g)]_i=\textup{sgn}\left([\xi_g]_i\right)\min\left(\left|[\xi_g]_i\right|,1\right)$.
Combining Theorem \ref{thm:fenchel_duality} and Lemma \ref{lemma:conjugate_Omega_SGL}, the Fenchel's dual of SGL can be derived as follows.
\begin{theorem}\label{thm:dual_SGL}
For the SGL problem in (\ref{prob:SGL}), the following hold:
\begin{enumerate}	
\item The Fenchel's dual of SGL is given by:
\begin{align}\label{prob:dual_SGL_Fenchel}
\inf_{\theta}\,\left\{\tfrac{1}{2}\|\tfrac{\mathbf{y}}{\lambda}-\theta\|^2-\tfrac{1}{2}\|\mathbf{y}\|^2: \left\|\mathbf{X}_g^T\theta-\mathbf{P}_{\mathcal{B}_{\infty}}(\mathbf{X}_g^T\theta)\right\|\leq\alpha\sqrt{n_g},\,g=1,\ldots,G\right\}.
\end{align}
\item Let $\beta^*(\lambda,\alpha)$ and $\theta^*(\lambda,\alpha)$ be the optimal solutions of problems (\ref{prob:SGL}) and (\ref{prob:dual_SGL_Fenchel}), respectively. Then,
\begin{align}\label{eqn:KKT1_SGL}
\lambda\theta^*(\lambda,\alpha)=&\mathbf{y}-\mathbf{X}\beta^*(\lambda,\alpha),\\ \label{eqn:KKT2_SGL}
\mathbf{X}_g^T\theta^*(\lambda,\alpha)\in&
\alpha\sqrt{n_g}\partial \|\beta_g^*(\lambda,\alpha)\|+\partial \|\beta^*_g(\lambda,\alpha)\|_1,\,g=1,\ldots,G.
\end{align}
\end{enumerate}
\end{theorem}

To show Theorem \ref{thm:dual_SGL}, we need the Fenchel-Young inequality as follows:
\begin{lemma}\label{lemma:FY_inequality}
	\textup{\textbf{[Fenchel-Young inequality]} \cite{Borwein2006}} Any point $\mathbf{z}\in\mathbb{R}^n$ and $\mathbf{w}$ in the domain of a function $h:\mathbb{R}^n\rightarrow(-\infty,\infty]$ satisfy the inequality
	\begin{align*}
	h(\mathbf{w})+h^*(\mathbf{z})\geq\langle \mathbf{w},\mathbf{z}\rangle.
	\end{align*}
	Equality holds if and only if $\mathbf{z}\in\partial h(\mathbf{w})$.
\end{lemma}

We now give the proof of Theorem \ref{thm:dual_SGL}.
\begin{proof}
	We first show the first part. Combining Theorem \ref{thm:fenchel_duality} and Lemma \ref{lemma:conjugate_Omega_SGL}, the Fenchel's dual of SGL can be written as:
	\begin{align*}
	\sup_{\theta}\,-\frac{\lambda^2}{2}\|\theta\|^2-\sum\nolimits_{g=1}^G\,\lambda\mathbf{I}_{\mathcal{B}}\left(\frac{\mathbf{X}_g^T\theta-\mathbf{P}_{\mathcal{B}_{\infty}}(\mathbf{X}_g^T\theta)}{\alpha\sqrt{n_g}}\right)+\lambda\langle\mathbf{y},\theta\rangle,
	\end{align*}
	which is equivalent to problem (\ref{prob:dual_SGL_Fenchel}).
	
	To show the second half, we have the following inequalities by Fenchel-Young inequality:
	\begin{align}\label{ineqn:FY_f_SGL}
	f(\mathbf{y}-\mathbf{X}\beta)+f^*(\lambda\theta)&\geq\langle\mathbf{y}-\mathbf{X}\beta,\lambda\theta\rangle,\\\label{ineqn:FY_Omega_SGL}
	\lambda\Omega(\beta)+\lambda\Omega^*(\mathbf{X}^T\theta)&\geq\lambda\langle\beta,\mathbf{X}^T\theta\rangle.
	\end{align}
	We sum the inequalities in (\ref{ineqn:FY_f_SGL}) and (\ref{ineqn:FY_Omega_SGL}) together and get
	\begin{align}\label{ineqn:weak_duality_SGL}
	f(\mathbf{y}-\mathbf{X}\beta)+\lambda\Omega(\beta)\geq-f^*(\lambda\theta)-\lambda\Omega^*(\mathbf{X}^T\theta)+\lambda\langle\mathbf{y},\theta\rangle.
	\end{align}
	Clearly, the left and right hand sides of inequality (\ref{ineqn:weak_duality_SGL}) are the objective functions of the pair of Fenchel's problems. Because $\textup{dom}\,f=\mathbb{R}^N$ and $\textup{dom}\,\Omega=\mathbb{R}^p$, we have $$0\in\textup{int}\,(\textup{dom}\,f-\mathbf{y}+\mathbf{X}\textup{dom}\,\Omega).$$ Thus, the equality in (\ref{ineqn:weak_duality_SGL}) holds at $\beta^*(\lambda,\alpha)$ and $\theta^*(\lambda,\alpha)$, i.e., 
	\begin{align*}
	f(\mathbf{y}-\mathbf{X}\beta^*(\lambda,\alpha))+\lambda\Omega(\beta^*(\lambda,\alpha))=-f^*(\lambda\theta^*(\lambda,\alpha))-\lambda\Omega^*(\mathbf{X}^T\theta^*(\lambda,\alpha))+\lambda\langle\mathbf{y},\theta^*(\lambda,\alpha)\rangle.
	\end{align*}
	Therefore, the equality holds in both (\ref{ineqn:FY_f_SGL}) and (\ref{ineqn:FY_Omega_SGL}) at $\beta^*(\lambda,\alpha)$ and $\theta^*(\lambda,\alpha)$. By applying Lemma \ref{lemma:FY_inequality} again, we have
	\begin{align*}
	\lambda\theta^*(\lambda,\alpha)&\in\partial f(\mathbf{y}-\mathbf{X}\beta^*(\lambda,\alpha))=\mathbf{y}-\mathbf{X}\beta^*(\lambda,\alpha),\\
	\mathbf{X}^T\theta^*(\lambda,\alpha)&\in\partial \Omega(\beta^*(\lambda,\alpha))=\partial \Omega_1^{\alpha}(\beta^*(\lambda,\alpha))+\partial\Omega_2(\beta^*(\lambda,\alpha)),
	\end{align*}
	which completes the proof.
\end{proof}
\eqref{eqn:KKT1_SGL} and \eqref{eqn:KKT2_SGL} are the so-called KKT conditions \textup{\cite{Boyd04}} and can also be obtained by the Lagrangian multiplier method \textup{(}see \ref{subsection:Lagrangian_supplement} in the supplement\textup{)}.

\begin{remark}\label{remark:shrinkage_feasible_SGL}
	We note that the shrinkage operator can also be expressed by
	\begin{align}\label{eqn:shrinkage}
		\mathcal{S}_{\gamma}(\mathbf{w})= \mathbf{w}-\mathbf{P}_{\gamma\mathcal{B}_{\infty}}(\mathbf{w}),\hspace{2mm}\gamma\geq0.
	\end{align}
	Therefore, problem (\ref{prob:dual_SGL_Fenchel}) can be written more compactly as
	\begin{align}\label{prob:dual_SGL_Shrinkage_Fenchel}
		\inf_{\theta}\,\left\{\tfrac{1}{2}\|\tfrac{\mathbf{y}}{\lambda}-\theta\|^2-\tfrac{1}{2}\|\mathbf{y}\|^2: \left\|\mathcal{S}_{1}(\mathbf{X}_g^T\theta)\right\|\leq\alpha\sqrt{n_g},\,g=1,\ldots,G\right\}.
	\end{align}
\end{remark}

\noindent\textbf{The equivalence between the dual formulations}
	 For the SGL problem, its Lagrangian dual in (\ref{prob:dual_SGL_Lagrangian}) and Fenchel's dual in (\ref{prob:dual_SGL_Fenchel}) are indeed equivalent to each other. We bridge them together by the following lemma.
\begin{lemma}\label{lemma:infconv_sets}\textup{\cite{Bauschke2011}}
	Let $\mathcal{C}_1$ and $\mathcal{C}_2$ be nonempty subsets of $\mathbb{R}^n$. Then $\mathbf{I}_{\mathcal{C}_1}\boxdot\mathbf{I}_{\mathcal{C}_2}=\mathbf{I}_{\mathcal{C}_1+\mathcal{C}_2}$.
\end{lemma}
In view of Lemmas \ref{lemma:conjugate_Omega_SGL} and \ref{lemma:infconv_sets}, and recall that $\mathcal{D}_g^{\alpha}=\mathcal{C}_g^{\alpha}+\mathcal{B}_{\infty}$, we have
\begin{align}\label{eqn:conjugate_Omega2_SGL}
\Omega^*(\xi)=\left((\Omega_1^{\alpha})^*\boxdot(\Omega_2)^*\right)(\xi)=\sum\nolimits_{g=1}^G\left(\mathbf{I}_{\mathcal{C}_g^{\alpha}}\boxdot\mathbf{I}_{\mathcal{B}_{\infty}}\right)(\xi_g)=\sum\nolimits_{g=1}^G\,\mathbf{I}_{\mathcal{D}_g^{\alpha}}(\xi_g).
\end{align}
Combining \eqref{eqn:conjugate_Omega2_SGL} and Theorem \ref{thm:fenchel_duality}, we obtain the dual formulation of SGL in (\ref{prob:dual_SGL_Lagrangian}). Therefore, the dual formulations of SGL in (\ref{prob:dual_SGL_Lagrangian}) and (\ref{prob:dual_SGL_Fenchel}) are the same.

\begin{remark}
	An appealing advantage of the Fenchel's dual in (\ref{prob:dual_SGL_Fenchel}) is that we have a natural decomposition of all points $\xi_g\in\mathcal{D}_g^{\alpha}$:
	$
	\xi_g=\mathbf{P}_{\mathcal{B}_{\infty}}(\xi_g)+\mathcal{S}_1(\xi_g))
	$
	with $\mathbf{P}_{\mathcal{B}_{\infty}}(\xi_g)\in\mathcal{B}_{\infty}$ and $\mathcal{S}_1(\xi_g)\in\mathcal{C}_g^{\alpha}$. As a result, this leads to a convenient way to determine the feasibility of any dual variable $\theta$ by checking if $\mathcal{S}_1(\mathbf{X}_g^T\theta)\in\mathcal{C}_g^{\alpha}$, $g=1,\ldots,G$.
\end{remark}

\subsection{Motivation of the Two-Layer Screening Rules}\label{subsection:general_rule_SGL}

We motive the two-layer screening rules via the KKT condition in \eqref{eqn:KKT2_SGL}. As implied by the name, there are two layers in our method. The first layer aims to identify the inactive groups, and the second layer is designed to detect the inactive features for the remaining groups.

by \eqref{eqn:KKT2_SGL}, we have the following cases by noting  $\partial \|\mathbf{w}\|_1=\textup{SGN}(\mathbf{w})$ and
\begin{align*}
	\partial \|\mathbf{w}\|=
	\begin{cases}
		\left\{\frac{\mathbf{w}}{\|\mathbf{w}\|}\right\},\hspace{13mm}\textup{if}\hspace{1mm}\mathbf{w}\neq0,\\
		\{\mathbf{u}: \|\mathbf{u}\|\leq1\}, \hspace{2.5mm}\textup{if}\hspace{1mm}\mathbf{w}=0.
	\end{cases}\hspace{5mm}
\end{align*}
\textbf{Case 1.} If $\beta_g^*(\lambda,\alpha)\neq0$, we have
\begin{align}\label{eqn:KKT1_Nonzero_SGL}
	[\mathbf{X}_g^T\theta^*(\lambda,\alpha)]_i\in
	\begin{cases}
		\alpha\sqrt{n_g}\frac{[\beta_g^*(\lambda,\alpha)]_i}{\|\beta^*_g(\lambda,\alpha)\|}+\textup{sign}([\beta_g^*(\lambda,\alpha)]_i), \hspace{2mm}\textup{if}\hspace{1mm}[\beta_g^*(\lambda,\alpha)]_i\neq0,\\
		[-1,1],\hspace{46.5mm}\textup{if}\hspace{1mm}[\beta_g^*(\lambda,\alpha)]_i=0.
	\end{cases}
\end{align}
In view of \eqref{eqn:KKT1_Nonzero_SGL}, we can see that
\begin{align}\label{eqn:c1_a_SGL}
	\textup{(a):}\hspace{2mm}&\mathcal{S}_{1}(\mathbf{X}_g^T\theta^*(\lambda,\alpha))=\alpha\sqrt{n_g}\tfrac{\beta^*_g(\lambda_1,\lambda_2)}{\|\beta^*_g(\lambda_1,\lambda_2)\|}\hspace{1mm}\textup{and}\hspace{1mm}\|\mathcal{S}_{1}(\mathbf{X}_g^T\theta^*(\lambda,\alpha))\|=\alpha\sqrt{n_g},\\\label{eqn:c1_b_SGL}
	\textup{(b):}\hspace{2mm}&\textup{If}\hspace{1mm}\left|[\mathbf{X}_g^T\theta^*(\lambda,\alpha]_i\right|\leq 1\hspace{1mm}\textup{then}\hspace{1mm} [\beta^*_g(\lambda,\alpha)]_i=0.
\end{align}
\textbf{Case 2.} If $\beta_g^*(\lambda,\alpha)=0$, we have
\begin{align}\label{eqn:c2_SGL}
	[\mathbf{X}_g^T\theta^*(\lambda,\alpha)]_i\in\alpha\sqrt{n_g}[\mathbf{u}_g]_i+[-1,1],\,\|\mathbf{u}_g\|\leq1.
\end{align}

\textbf{The first layer (group-level) of TLFre} From (\ref{eqn:c1_a_SGL}) in \textbf{Case 1}, we have
\begin{align}\tag{R1}\label{rule1_SGL}
	\left\|\mathcal{S}_{1}(\mathbf{X}_g^T\theta^*(\lambda,\alpha))\right\|<\alpha\sqrt{n_g}\Rightarrow \beta_g^*(\lambda,\alpha)=0.
\end{align}
Clearly, (\ref{rule1_SGL}) can be used to identify the inactive groups and thus a group-level screening rule.

\textbf{The second layer (feature-level) of TLFre}
Let $\mathbf{x}_{g_i}$ be the $i^{th}$ column of $\mathbf{X}_g$. We have $[\mathbf{X}_g^T\theta^*(\lambda,\alpha)]_i=\mathbf{x}_{g_i}^T\theta^*(\lambda,\alpha)$. In view of (\ref{eqn:c1_b_SGL}) and (\ref{eqn:c2_SGL}), we can see that
\begin{align}\tag{R2}\label{rule2_SGL}
	\left|\mathbf{x}_{g_i}^T\theta^*(\lambda,\alpha)\right|\leq1\Rightarrow [\beta^*_g(\lambda,\alpha)]_i=0.
\end{align}
Different from (\ref{rule1_SGL}), (\ref{rule2_SGL}) detects the inactive features and thus it is a feature-level screening rule.

However, we cannot directly apply (\ref{rule1_SGL}) and (\ref{rule2_SGL}) to identify the inactive groups/features because both need to know $\theta^*(\lambda,\alpha)$. Inspired by the SAFE rules \cite{ElGhaoui2012}, we can first estimate a region $\Theta$ containing $\theta^*(\lambda,\alpha)$. Let $\mathbf{X}_g^T\Theta=\{\mathbf{X}_g^T\theta:\theta\in\Theta\}$. Then, (\ref{rule1_SGL}) and (\ref{rule2_SGL}) can be relaxed as follows:
\begin{align}\tag{R1$^*$}\label{rrule1_SGL}
	&\sup\nolimits_{\xi_g}\,\left\{\|\mathcal{S}_{1}(\xi_g)\|:\xi_g\in\Xi_g\supseteq\mathbf{X}_g^T\Theta\right\}<\alpha\sqrt{n_g}\Rightarrow\beta_g^*(\lambda,\alpha)=0,\\\tag{R2$^*$}\label{rrule2_SGL}
	&\sup\nolimits_{\theta}\,\left\{\left|\mathbf{x}_{g_i}^T\theta\right|:\theta\in\Theta\right\}\leq 1\Rightarrow[\beta^*_g(\lambda,\alpha)]_i=0.
\end{align}

Inspired by (\ref{rrule1_SGL}) and (\ref{rrule2_SGL}), we develop TLFre via the following three steps:
\begin{enumerate}
\item[\textbf{Step 1}.] Given $\lambda$ and $\alpha$, we estimate a region $\Theta$ that contains $\theta^*(\lambda,\alpha)$.
\item[\textbf{Step 2}.] We solve for the supreme values in (\ref{rrule1_SGL}) and (\ref{rrule2_SGL}).
\item[\textbf{Step 3}.] By plugging in the supreme values from \textbf{Step 2}, (\ref{rrule1_SGL}) and (\ref{rrule2_SGL}) result in the desired two-layer screening rules for SGL.
\end{enumerate}

\subsection{The Set of Parameter Values Leading to Zero Solutions}\label{subsection:lambda12_0_solution}

In this section, we explore the geometric properties of the Fenchel's dual of SGL in depth---based on which we can derive the set of parameter values such that the primal optimal solutions are 0. We consider the SGL problem in (\ref{prob:SGL}) and (\ref{prob:SGL0}) in Section \ref{sssec:zero_sol_SGL} and \ref{sssec:zero_sol_SGL0}, respectively.

\subsubsection{The Set of Parameter Values Leading to Zero Solutions of Problem (\ref{prob:SGL})}\label{sssec:zero_sol_SGL}
Consider the SGL problem in (\ref{prob:SGL}). For notational convenience, let
\begin{align*}
\mathcal{F}_g^{\alpha}=\{\theta:\|\mathcal{S}_{1}(\mathbf{X}_g^T\theta)\|\leq\alpha\sqrt{n_g}\},\,g=1,\ldots,G.
\end{align*} 
We denote the feasible set of the Fenchel's dual of SGL by $$\mathcal{F}^{\alpha}=\cap_{g=1,\ldots,G}\,\mathcal{F}_g^{\alpha}.$$
In view of problem (\ref{prob:dual_SGL_Fenchel}) [or (\ref{prob:dual_SGL_Shrinkage_Fenchel})], we can see that $\theta^*(\lambda,\alpha)$ is the projection of $\mathbf{y}/\lambda$ on $\mathcal{F}^{\alpha}$, i.e.,
\begin{align}\label{eqn:projection_dual_SGL}
\theta^*(\lambda,\alpha)=\mathbf{P}_{\mathcal{F}^{\alpha}}(\mathbf{y}/\lambda).
\end{align}
Thus, if $\mathbf{y}/\lambda\in\mathcal{F}^{\alpha}$, we have $\theta^*(\lambda,\alpha)=\mathbf{y}/\lambda$. Moreover, by (\ref{rule1_SGL}), we can see that $\beta^*(\lambda,\alpha)=0$ if $\mathbf{y}/\lambda$ is an \emph{interior} point of $\mathcal{F}^{\alpha}$.
Indeed, we have the following stronger result.

\begin{theorem}\label{thm:lambda_alpha_SGL}
For the SGL problem, let $\lambda_{\rm max}^{\alpha}=\max_{g}\,\{\rho_g:\left\|\mathcal{S}_{1}(\mathbf{X}_g^T\mathbf{y}/\rho_g)\right\|=\alpha\sqrt{n_g}\}$.
Then, the following statements are equivalent:

\vspace{2mm}
\begin{enumerate*}[itemjoin=\hfill]
	\item $\dfrac{\mathbf{y}}{\lambda}\in\mathcal{F}^{\alpha}$,
	\item $\theta^*(\lambda,\alpha)=\dfrac{\mathbf{y}}{\lambda}$,
	\item $\beta^*(\lambda,\alpha)=0$,
	\item $\lambda\geq\lambda_{\rm max}^{\alpha}$.
\end{enumerate*}
\end{theorem}

\begin{proof}
The equivalence between (i) and (ii) can be see from the fact that $\theta^*(\lambda,\alpha)=\mathbf{P}_{\mathcal{F}^{\alpha}}(\mathbf{y}/\lambda)$.
	
Next, we show (ii)$\Leftrightarrow$(iii). Let us first show (ii)$\Rightarrow$(iii). We assume that $\theta^*(\lambda,\alpha)=\mathbf{y}/\lambda$. By the KKT condition in (\ref{eqn:KKT1_SGL}), we have $\mathbf{X}\beta^*(\lambda,\alpha)=0$.
We claim that $\beta^*(\lambda,\alpha)=0$. To see this, let $\beta'\neq0$ with $\mathbf{X}\beta'=0$ be another optimal solution of SGL. We denote by $h$ the objective function of SGL in (\ref{prob:SGL}). Then, we have
\begin{align*}
h(0)=\frac{1}{2}\|\mathbf{y}\|^2<h(\beta')=\frac{1}{2}\|\mathbf{y}\|^2+\lambda_1\sum_g\sqrt{n_g}\|\beta'_g\|+\lambda_2\|\beta'\|_1,
\end{align*}
which contradicts with the assumption $\beta'\neq0$ is also an optimal solution. This contradiction indicates that $\beta^*(\lambda,\alpha)$ must be $0$. The converse direction, i.e., (ii)$\Leftarrow$(iii), can be derived directly from the KKT condition in \eqref{eqn:KKT1_SGL}.
	
Finally, we show the equivalence (i)$\Leftrightarrow$(iv). Indeed, in view of the dual problem in (\ref{prob:dual_SGL_Shrinkage_Fenchel}), we can see that $\mathbf{y}/\lambda\in\mathcal{F}^{\alpha}$ if and only if
\begin{align}\label{ineqn:feasibility_y_SGL}
\|\mathcal{S}_{1}(\mathbf{X}_g^T\mathbf{y}/\lambda)\|\leq\alpha\sqrt{n_g},\,g=1,\ldots,G.
\end{align}
We note that $\|\mathcal{S}_{1}(\mathbf{X}_g^T\mathbf{y}/\lambda)\|$ is monotonically decreasing with respect to $\lambda$. Thus, the inequality in (\ref{ineqn:feasibility_y_SGL}) is equivalent to (iv), which completes the proof.
\end{proof}

We note that $\rho_g$ in the definition of $\lambda_{\rm max}^{\alpha}$ admits a closed form solution. For notational convenience, let $|\mathbf{w}|$ be the vector by taking absolute value of $\mathbf{w}$ component-wisely and $[\mathbf{w}]^{(k)}$ be the vector consisting of the first $k$ components of $\mathbf{w}$.
\begin{lemma}\label{lemma:rho}
We sort  $0\neq|\mathbf{X}_g^T\mathbf{y}|\in\mathbb{R}^{n_g}$ in descending order and denote it by $\mathbf{z}$.
\begin{enumerate}
\item If there exists $[\mathbf{z}]_k$ such that $\|\mathcal{S}_{1}(\mathbf{X}_g^T\mathbf{y}/[\mathbf{z}]_k)\|=\alpha\sqrt{n_g}$, then $\rho_g=[\mathbf{z}]_k$.
\item Otherwise, let $\tau_i=\|\mathcal{S}_{1}(\mathbf{X}_g^T\mathbf{y}/[\mathbf{z}]_i)\|$, $i=1,\ldots,n_g$, and $\tau_{n_g+1}=\infty$. There exists a $k$ such that $\alpha\sqrt{n_g}\in(\tau_{k},\tau_{k+1})$, and $\rho_g\in([\mathbf{z}]_{k+1},[\mathbf{z}]_k)$ is the root of
$$(k-\alpha^2n_g)\rho^2-2\rho\|[\mathbf{z}]^{(k)}\|_1+\|[\mathbf{z}]^{(k)}\|^2=0.$$
\end{enumerate}
\end{lemma}

We omit the proof of Lemma \ref{lemma:rho} because it is a direct consequence by noting that $$\|\mathcal{S}_{1}(\mathbf{X}_g^T\mathbf{y}/\lambda)\|^2=\alpha^2n_g$$ is piecewise quadratic.

\subsubsection{The Set of Parameter Values Leading to Zero Solutions of Problem (\ref{prob:SGL0})}\label{sssec:zero_sol_SGL0}
 
Theorem \ref{thm:lambda_alpha_SGL} implies that the optimal solution $\beta^*(\lambda,\alpha)$ is $0$ as long as $\mathbf{y}/\lambda\in\mathcal{F}^{\alpha}$. This geometric property also leads to an explicit characterization of the set of $(\lambda_1, \lambda_2)$ such that the corresponding solution of problem (\ref{prob:SGL0}) is $0$. We denote by $\bar{\beta}^*(\lambda_1,\lambda_2)$ the optimal solution of problem (\ref{prob:SGL0}).

\begin{corollary}\label{cor:lambdamx_SGL}
For the SGL problem in (\ref{prob:SGL0}), let $\lambda_1^{\rm max}(\lambda_2)=\max_{g}\,\tfrac{1}{\sqrt{n_g}}\|\mathcal{S}_{\lambda_2}(\mathbf{X}_g^T\mathbf{y})\|$. Then, 
\begin{enumerate}
\item $\bar{\beta}^*(\lambda_1,\lambda_2)=0\Leftrightarrow\lambda_1\geq\lambda_1^{\rm max}(\lambda_2)$.  
\item If $\lambda_1\geq\lambda_1^{\rm max}:=\max_g\tfrac{1}{\sqrt{n_g}}\|\mathbf{X}_g^T\mathbf{y}\|$ or $\lambda_2\geq\lambda_2^{\rm max}:=\|\mathbf{X}^T\mathbf{y}\|_{\infty}$, then $\bar{\beta}^*(\lambda_1,\lambda_2)=0$.
\end{enumerate}
\end{corollary}

Before we prove Corollary \ref{cor:lambdamx_SGL}, we first derive the Fenchel's dual of (\ref{prob:SGL0}).  
By letting $f(\mathbf{w})=\frac{1}{2}\|\mathbf{w}\|^2$ and $\Omega(\beta)=\lambda_1\sum_{g=1}^G\sqrt{n_g}\|\beta_g\|+\lambda_2\|\beta\|_1$, the SGL problem in (\ref{prob:SGL0}) can be written as:
\begin{align*}
\min_{\beta}\,f(\mathbf{y}-\mathbf{X}\beta)+\Omega(\beta).
\end{align*}
Then, by Fenchel's Duality Theorem, the Fenchel's dual problem of (\ref{prob:SGL0}) is
\begin{align}\label{prob:dual_SGL0_Fenchel}
\inf_{\theta}\,\left\{\frac{1}{2}\|\mathbf{y}-\theta\|^2-\frac{1}{2}\|\mathbf{y}\|^2: \left\|\mathcal{S}_{\lambda_2}(\mathbf{X}_g^T\theta)\right\|\leq\lambda_1\sqrt{n_g},\,g=1,\ldots,G\right\}.
\end{align}
Let $\bar{\beta}^*(\lambda_1,\lambda_2)$ and $\bar{\theta}^*(\lambda_1,\lambda_2)$ be the optimal solutions of problem (\ref{prob:SGL0}) and (\ref{prob:dual_SGL0_Fenchel}). The optimality conditions can be written as
\begin{align}\label{eqn:KKT1_SGL0}
\bar{\theta}^*(\lambda_1,\lambda_2)=&\mathbf{y}-\mathbf{X}\bar{\beta}^*(\lambda_1,\lambda_2),\\ \label{eqn:KKT2_SGL0}
\mathbf{X}_g^T\bar{\theta}^*(\lambda_1,\lambda_2)&\in
\lambda_1\sqrt{n_g}\partial \|\bar{\beta}_g^*(\lambda_1,\lambda_2)\|+\lambda_2\partial \|\bar{\beta}^*_g(\lambda_1,\lambda_2)\|_1,\,g=1,\ldots,G.
\end{align}
We denote by $\mathcal{F}(\lambda_1,\lambda_2)$ the feasible set of problem (\ref{prob:dual_SGL0_Fenchel}). It is easy to see that
\begin{align*}
\bar{\theta}^*(\lambda_1,\lambda_2)=\mathbf{P}_{\mathcal{F}(\lambda_1,\lambda_2)}(\mathbf{y}).
\end{align*}

We now present the proof of Corollary \ref{cor:lambdamx_SGL}.

\begin{proof}
	For notational convenience, let
	\begin{enumerate}
		\item[(i).] $\mathbf{y}\in\mathcal{F}(\lambda_1,\lambda_2)$,
		\item[(ii).] $\bar{\theta}^*(\lambda_1,\lambda_2)=\mathbf{y}$,
		\item[(iii).] $\bar{\beta}^*(\lambda_1,\lambda_2)=0$,
		\item[(iv).] $\lambda_1\geq\lambda_1^{\rm max}(\lambda_2)=\max_g\frac{1}{\sqrt{n_g}}\|\mathcal{S}_{\lambda_2}(\mathbf{X}_g^T\mathbf{y})\|$.
	\end{enumerate}
	The first half of the statement is (iii)$\Leftrightarrow$(iv). Indeed, 
	by a similar argument as in the proof of Theorem \ref{thm:lambda_alpha_SGL}, we can see that the above statements are all equivalent to each other.
	
	We now show the second half. We first show that
	\begin{align}\label{statement:lambda1mx}
	\lambda_1\geq\lambda_1^{\rm max}\Rightarrow \bar{\beta}^*(\lambda_1,\lambda_2)=0.
	\end{align}
	By the first half, we only need to show
	\begin{align*}
	\lambda_1\geq\lambda_1^{\rm max}\Rightarrow \mathbf{y}\in\mathcal{F}(\lambda_1,\lambda_2).
	\end{align*}
	Indeed, the definition of $\lambda_1$ implies that
	\begin{align*}
	\|\mathbf{X}_g^T\mathbf{y}\|\leq \lambda_1\sqrt{n_g},\,g=1,\ldots,G.
	\end{align*}
	We note that for any $\lambda_2\geq 0$, we have
	\begin{align*}
	\|\mathcal{S}_{\lambda_2}(\mathbf{X}_g^T\mathbf{y})\|\leq\|\mathbf{X}_g^T\mathbf{y}\|.
	\end{align*}
	Therefore, we can see that
	\begin{align*}
	\|\mathcal{S}_{\lambda_2}(\mathbf{X}_g^T\mathbf{y})\|\leq\|\mathbf{X}_g^T\mathbf{y}\|\leq \lambda_1\sqrt{n_g},\,g=1,\ldots,G\Rightarrow \mathbf{y}\in\mathcal{F}(\lambda_1,\lambda_2).
	\end{align*}
	The proof of (\ref{statement:lambda1mx}) is complete.
	
	Similarly, to show that $\lambda_2\geq\lambda_2^{\rm max}\Rightarrow\bar{\beta}^*(\lambda_1,\lambda_2)$, we only need to show
	\begin{align*}
	\lambda_2\geq\lambda_2^{\rm max}\Rightarrow \mathbf{y}\in\mathcal{F}(\lambda_1,\lambda_2).
	\end{align*}
	By the definition of $\lambda_2$, we can see that
	\begin{align*}
	\|\mathbf{X}_g^T\mathbf{y}\|_{\infty}\leq\lambda_2,\,g=1,\ldots,G\Rightarrow \|\mathcal{S}_{\lambda_2}(\mathbf{X}_g^T\mathbf{y})\|=0\leq \lambda_1\sqrt{n_g},\,g=1,\ldots,G.
	\end{align*}
	Thus, we have $\mathbf{y}\in\mathcal{F}(\lambda_1,\lambda_2)$, which completes the proof.
\end{proof}

\section{The Two-Layer Screening Rules for SGL}\label{section:screening_rules_SGL}
We follow the three steps in Section \ref{subsection:general_rule_SGL} to develop TLFre. In Section \ref{subsection:estimation_dual_SGL}, we give an accurate estimation of $\theta^*(\lambda,\alpha)$ via normal cones \cite{Ruszczynski2006}. Then, we compute the supreme values in (\ref{rrule1_SGL}) and (\ref{rrule2_SGL}) by solving nonconvex problems in Section \ref{subsection:optimization_SGL}. We present the TLFre rules in Section \ref{subsection:screening_rules_SGL}.

\subsection{Estimation of the Dual Optimal Solution}\label{subsection:estimation_dual_SGL}

Because of the geometric property of the dual problem in (\ref{prob:dual_SGL_Fenchel}), i.e., $\theta^*(\lambda,\alpha)=\mathbf{P}_{\mathcal{F}^{\alpha}}(\mathbf{y}/\lambda)$, we have a very useful characterization of the dual optimal solution via the so-called normal cones \cite{Ruszczynski2006}.
\begin{proposition}\label{prop:normal_cone}\textup{\cite{Ruszczynski2006,Bauschke2011}}
	For a closed convex set $\mathcal{C}\in\mathbb{R}^n$ and a point $\mathbf{w}\in\mathcal{C}$, the normal cone to $\mathcal{C}$ at $\mathbf{w}$ is defined by
	\begin{align}\label{def:projection1}
		N_{\mathcal{C}}(\mathbf{w})=\{\mathbf{v}:\langle\mathbf{v},\mathbf{w}'-\mathbf{w}\rangle\leq0,\,\forall \mathbf{w}'\in\mathcal{C}\}.
	\end{align}
	Then, the following hold:
		\begin{enumerate}
			\item
			$N_{\mathcal{C}}(\mathbf{w})=\{\mathbf{v}:\mathbf{P}_{\mathcal{C}}(\mathbf{w}+\mathbf{v})=\mathbf{w}\}$.
			\item $\mathbf{P}_{\mathcal{C}}(\mathbf{w}+\mathbf{v})=\mathbf{w},\,\forall\,\mathbf{v}\in N_{\mathcal{C}}(\mathbf{w})$.
			\item Let $\overline{\mathbf{w}}\notin\mathcal{C}$. Then, $\mathbf{w}=\mathbf{P}_{\mathcal{C}}(\overline{\mathbf{w}})\Leftrightarrow \overline{\mathbf{w}}-\mathbf{w}\in N_{\mathcal{C}}(\mathbf{w})$.
			\item  Let $\overline{\mathbf{w}}\notin\mathcal{C}$ and $\mathbf{w}=\mathbf{P}_{\mathcal{C}}(\overline{\mathbf{w}})$. Then, $\mathbf{P}_{\mathcal{C}}(\mathbf{w}+t(\overline{\mathbf{w}}-\mathbf{w}))=\mathbf{w}$ for all $t\geq0$.
		\end{enumerate}
\end{proposition}
By Theorem \ref{thm:lambda_alpha_SGL}, $\theta^*(\bar{\lambda},\alpha)$ is known if $\bar{\lambda}=\lambda_{\rm max}^{\alpha}$. Thus, we can estimate $\theta^*(\lambda,\alpha)$ in terms of $\theta^*(\bar{\lambda},\alpha)$. Due to the same reason, \emph{we only consider the cases with $\lambda<\lambda_{\rm max}^{\alpha}$ for $\theta^*(\lambda,\alpha)$} to be estimated.
\begin{remark}
	In many applications, the parameter values that perform the best are usually unknown. To determine appropriate parameter values, commonly used approaches such as cross validation and stability selection involve solving SGL many times over a grip of parameter values. Thus, given $\{\alpha^{(i)}\}_{i=1}^{\mathcal{I}}$ and $\lambda^{(1)}\geq\cdots\geq\lambda^{(\mathcal{J})}$, we can fix the value of $\alpha$ each time and solve SGL by varying the value of $\lambda$. We repeat the process until we solve SGL for all of the parameter values.
\end{remark}
\begin{theorem}\label{thm:estimation_SGL}
	For the SGL problem in (\ref{prob:SGL}), suppose that $\theta^*(\bar{\lambda},\alpha)$ is known with $\bar{\lambda}\leq\lambda_{\rm max}^{\alpha}$. Let $\rho_g$, $g=1,\ldots,G$, be defined by Theorem \ref{thm:lambda_alpha_SGL}. For any $\lambda\in(0,\bar{\lambda})$, we define
	\begin{align*}
		&\mathbf{n}_{\alpha}(\bar{\lambda})=
		\begin{cases}
			\dfrac{\mathbf{y}}{\bar{\lambda}}-\theta^*(\bar{\lambda},\alpha),\hspace{10.5mm}\textup{if}\hspace{1mm}\bar{\lambda}<\lambda_{\rm max}^{\alpha},\\
			\mathbf{X}_*\mathcal{S}_{1}\left(\mathbf{X}_*^T\frac{\mathbf{y}}{\lambda_{\rm max}^{\alpha}}\right),\hspace{3.5mm}\textup{if}\hspace{1mm}\bar{\lambda}=\lambda_{\rm max}^{\alpha},
		\end{cases}
		\hspace{-3mm}\textup{where}\hspace{1mm}\mathbf{X}_*=\textup{argmax}_{\mathbf{X}_g}\,\rho_g,\\
		&\mathbf{v}_{\alpha}(\lambda,\bar{\lambda})=\frac{\mathbf{y}}{\lambda}-\theta^*(\bar{\lambda},\alpha), \\
		&\mathbf{v}_{\alpha}(\lambda,\bar{\lambda})^{\perp} = \mathbf{v}_{\alpha}(\lambda,\bar{\lambda}) - \frac{\langle\mathbf{v}_{\alpha}(\lambda,\bar{\lambda}),\mathbf{n}_{\alpha}(\bar{\lambda})\rangle}{\vphantom{\widetilde{\bar{\lambda}}}\|\mathbf{n}_{\alpha}(\bar{\lambda})\|^2}\mathbf{n}_{\alpha}(\bar{\lambda}).
	\end{align*}
	Then, the following hold:
	\begin{enumerate}
		\item $\mathbf{n}_{\alpha}(\bar{\lambda})\in N_{\mathcal{F}^{\alpha}}(\theta^*(\bar{\lambda},\alpha))$,
		\item $\|\theta^*(\lambda,\alpha)-(\theta^*(\bar{\lambda},\alpha)+\tfrac{1}{2}\mathbf{v}^{\perp}_{\alpha}(\lambda,\bar{\lambda}))\|\leq \tfrac{1}{2}\|\mathbf{v}^{\perp}_{\alpha}(\lambda,\bar{\lambda})\|$.
	\end{enumerate}
\end{theorem}

\begin{proof}
	\begin{enumerate}
		\item Suppose that $\bar{\lambda}<\lambda_{\rm max}^{\alpha}$. Theorem \ref{thm:lambda_alpha_SGL} implies that $\mathbf{y}/\bar{\lambda}\notin\mathcal{F}^{\alpha}$ and thus
		\begin{align*}
		{\mathbf{y}}/{\bar{\lambda}}-\mathbf{P}_{\mathcal{F}^{\alpha}}\left({\mathbf{y}}/\bar{\lambda}\right)={\mathbf{y}}/{\bar{\lambda}}-\theta^*(\bar{\lambda},\alpha)\neq0.
		\end{align*}
		By the third part of Proposition \ref{prop:normal_cone}, we can see that
		\begin{align}\label{eqn:n_lambda}
		{\mathbf{y}}/\bar{\lambda}-\theta^*(\bar{\lambda},\alpha)\in N_{\mathcal{F}^{\alpha}}(\theta^*(\bar{\lambda},\alpha)).
		\end{align}
		Thus, the statement holds for all $\bar{\lambda}<\lambda_{\rm max}^{\alpha}$.

		Suppose that $\bar{\lambda}=\lambda_{\rm max}^{\alpha}$.
		By Theorem \ref{thm:lambda_alpha_SGL}, we have
		\begin{align*}
		\theta^*(\bar{\lambda},\alpha)=\mathbf{y}/\bar{\lambda}\in\mathcal{F}^{\alpha}.
		\end{align*}
		In view of the definition of $\mathbf{X}_*$, we have
		\begin{align*}
		\left\|\mathcal{S}_{1}\left(\mathbf{X}_*^T\tfrac{\mathbf{y}}{\lambda_{\rm max}^{\alpha}}\right)\right\|=\alpha\sqrt{n_*},
		\end{align*}
		where $n_*$ is the number of feature contained in $\mathbf{X}_*$. Moreover, it is easy to see that
		\begin{align*}
		\|\mathcal{S}_{1}(\mathbf{X}_*^T\theta)\|\leq\alpha\sqrt{n_*},\,\forall \theta\in\mathcal{F}^{\alpha}.
		\end{align*}
		Therefore, to prove the statement, we need to show that
		\begin{align}\label{ineqn:normal_vector_boundary_SGL}
		\left\langle\mathbf{X}_*\mathcal{S}_{1}\left(\mathbf{X}_*^T\tfrac{\mathbf{y}}{\lambda_{\rm max}^{\alpha}}\right),\theta-\tfrac{\mathbf{y}}{\lambda_{\rm max}^{\alpha}}\right\rangle\leq0,\,\forall \theta\in\mathcal{F}^{\alpha}.
		\end{align}
		Recall Remark \ref{remark:shrinkage_feasible_SGL}, we have the following identity [see \eqref{eqn:shrinkage}]
		\begin{align*}
		\mathcal{S}_{1}\left(\mathbf{X}^T_*\tfrac{\mathbf{y}}{\lambda_{\rm max}^{\alpha}}\right)=\mathbf{X}^T_*\tfrac{\mathbf{y}}{\lambda_{\rm max}^{\alpha}}-\mathbf{P}_{\mathcal{B}_{\infty}}\left(\mathbf{X}^T_*\tfrac{\mathbf{y}}{\lambda_{\rm max}^{\alpha}}\right).
		\end{align*}
		Thus, we have
		\begin{align}\label{eqn:normvec_bdy_SGL}
		&\left\langle\mathbf{X}_*\mathcal{S}_{1}\left(\mathbf{X}_*^T\tfrac{\mathbf{y}}{\lambda_{\rm max}^{\alpha}}\right),\theta-\tfrac{\mathbf{y}}{\lambda_{\rm max}^{\alpha}}\right\rangle\\ \nonumber
		=&\left\langle\mathcal{S}_{1}\left(\mathbf{X}_*^T\tfrac{\mathbf{y}}{\lambda_{\rm max}^{\alpha}}\right),\mathbf{X}_*^T\left(\theta-\tfrac{\mathbf{y}}{\lambda_{\rm max}^{\alpha}}\right)+\mathbf{P}_{\mathcal{B}_{\infty}}\left(\mathbf{X}_*^T\tfrac{\mathbf{y}}{\lambda_{\rm max}^{\alpha}}\right)-\mathbf{P}_{\mathcal{B}_{\infty}}\left(\mathbf{X}_*^T\tfrac{\mathbf{y}}{\lambda_{\rm max}^{\alpha}}\right)\right\rangle\\ \nonumber
		=&\left\langle\mathcal{S}_{1}\left(\mathbf{X}_*^T\tfrac{\mathbf{y}}{\lambda_{\rm max}^{\alpha}}\right),\mathbf{X}_*^T\theta-\mathbf{P}_{\mathcal{B}_{\infty}}\left(\mathbf{X}_*^T\tfrac{\mathbf{y}}{\lambda_{\rm max}^{\alpha}}\right)\right\rangle-\left\|\mathcal{S}_{1}\left(\mathbf{X}_*^T\tfrac{\mathbf{y}}{\lambda_{\rm max}^{\alpha}}\right)\right\|^2\\ \nonumber
		=&\left\langle\mathcal{S}_{1}\left(\mathbf{X}_*^T\tfrac{\mathbf{y}}{\lambda_{\rm max}^{\alpha}}\right),\mathbf{X}_*^T\theta-\mathbf{P}_{\mathcal{B}_{\infty}}\left(\mathbf{X}_*^T\tfrac{\mathbf{y}}{\lambda_{\rm max}^{\alpha}}\right)\right\rangle-\alpha^2n_*.
		\end{align}
		Consider the first term on the right hand side of \eqref{eqn:normvec_bdy_SGL}, we have
		\begin{align}\label{eqn:normvec_bdy_1_SGL}
		&\left\langle\mathcal{S}_{1}\left(\mathbf{X}_*^T\tfrac{\mathbf{y}}{\lambda_{\rm max}^{\alpha}}\right),\mathbf{X}_*^T\theta-\mathbf{P}_{\mathcal{B}_{\infty}}\left(\mathbf{X}_*^T\tfrac{\mathbf{y}}{\lambda_{\rm max}^{\alpha}}\right)\right\rangle\\ \nonumber
		=&\left\langle\mathcal{S}_{1}\left(\mathbf{X}_*^T\tfrac{\mathbf{y}}{\lambda_{\rm max}^{\alpha}}\right),\mathbf{X}_*^T\theta-\mathbf{P}_{\mathcal{B}_{\infty}}(\mathbf{X}_*^T\theta)+\mathbf{P}_{\mathcal{B}_{\infty}}(\mathbf{X}_*^T\theta)-\mathbf{P}_{\mathcal{B}_{\infty}}\left(\mathbf{X}_*^T\tfrac{\mathbf{y}}{\lambda_{\rm max}^{\alpha}}\right)\right\rangle\\ \nonumber
		=&\left\langle\mathcal{S}_{1}\left(\mathbf{X}_*^T\tfrac{\mathbf{y}}{\lambda_{\rm max}^{\alpha}}\right),\mathcal{S}_{1}(\mathbf{X}_*^T\theta)\right\rangle+\left\langle\mathcal{S}_{1}\left(\mathbf{X}_*^T\tfrac{\mathbf{y}}{\lambda_{\rm max}^{\alpha}}\right),\mathbf{P}_{\mathcal{B}_{\infty}}(\mathbf{X}_*^T\theta)-\mathbf{P}_{\mathcal{B}_{\infty}}\left(\mathbf{X}_*^T\tfrac{\mathbf{y}}{\lambda_{\rm max}^{\alpha}}\right)\right\rangle.
		\end{align}
		Let $\mathcal{P}=\{i:[\mathbf{X}_*^T\frac{\mathbf{y}}{\lambda_{\rm max}^{\alpha}}]_i>1\}$ and $\mathcal{N}=\{i:[\mathbf{X}_*^T\frac{\mathbf{y}}{\lambda_{\rm max}^{\alpha}}]_i<-1\}$. We note that the second term on the right hand side of \eqref{eqn:normvec_bdy_1_SGL} can be written as
		\begin{align}\label{eqn:normvec_bdy_2_SGL}
		&\left\langle\mathcal{S}_{1}\left(\mathbf{X}_*^T\tfrac{\mathbf{y}}{\lambda_{\rm max}^{\alpha}}\right),\mathbf{P}_{\mathcal{B}_{\infty}}(\mathbf{X}_*^T\theta)-\mathbf{P}_{\mathcal{B}_{\infty}}\left(\mathbf{X}_*^T\tfrac{\mathbf{y}}{\lambda_{\rm max}^{\alpha}}\right)\right\rangle\\ \nonumber
		=&\sum_{i\in\mathcal{P}}\left([\mathbf{X}_*^T\tfrac{\mathbf{y}}{\lambda_{\rm max}^{\alpha}}]_i-1\right)\left([\mathbf{P}_{\mathcal{B}_{\infty}}(\mathbf{X}_*^T\theta)]_i-1\right)+\sum_{j\in\mathcal{N}}\left([\mathbf{X}_*^T\tfrac{\mathbf{y}}{\lambda_{\rm max}^{\alpha}}]_j+1\right)\left([\mathbf{P}_{\mathcal{B}_{\infty}}(\mathbf{X}_*^T\theta)]_j+1\right).
		\end{align}
		Because $\|\mathbf{P}_{\mathcal{B}_{\infty}}(\mathbf{X}_*^T\theta)\|_{\infty}\leq1$, we can see that \eqref{eqn:normvec_bdy_2_SGL} is non-positive. Therefore, by \eqref{eqn:normvec_bdy_1_SGL}, we have
		\begin{align}\label{eqn:normvec_bdy_3_SGL}
		\left\langle\mathcal{S}_{1}\left(\mathbf{X}_*^T\tfrac{\mathbf{y}}{\lambda_{\rm max}^{\alpha}}\right),\mathbf{X}_*^T\theta-\mathbf{P}_{\mathcal{B}_{\infty}}\left(\mathbf{X}_*^T\tfrac{\mathbf{y}}{\lambda_{\rm max}^{\alpha}}\right)\right\rangle \leq& \left\langle\mathcal{S}_{1}\left(\mathbf{X}_*^T\tfrac{\mathbf{y}}{\lambda_{\rm max}^{\alpha}}\right),\mathcal{S}_{1}(\mathbf{X}_*^T\theta)\right\rangle\\ \nonumber
		\leq&\left\|\mathcal{S}_{1}\left(\mathbf{X}_*^T\tfrac{\mathbf{y}}{\lambda_{\rm max}^{\alpha}}\right)\right\|\left\|\mathcal{S}_{1}(\mathbf{X}_*^T\theta)\right\|\\ \nonumber
		\leq&\alpha^2n_*.
		\end{align}
		Combining \eqref{eqn:normvec_bdy_SGL} and the inequality in (\ref{eqn:normvec_bdy_3_SGL}), we can see that the inequality in (\ref{ineqn:normal_vector_boundary_SGL}) holds. Thus, the statement holds for $\bar{\lambda}=\lambda_{\rm max}^{\alpha}$. This completes the proof.
		
		\item We now show the second half. It is easy to see that the statement is equivalent to
		\begin{align}\label{ineqn:estimate1_SGL}
		\|\theta^*(\lambda,\alpha)-\theta^*(\bar{\lambda},\alpha)\|^2\leq \langle\theta^*(\lambda,\alpha)-\theta^*(\bar{\lambda},\alpha),\,\mathbf{v}_{\alpha}^{\perp}(\lambda,\bar{\lambda})\rangle.
		\end{align}
		Thus, we will show that the inequality in (\ref{ineqn:estimate1_SGL}) holds.
		
		Because of the first half, we have
		\begin{align}\label{ineqn:n1_o_SGL}
		\langle\mathbf{n}_{\alpha}(\bar{\lambda}),\,\theta-\theta^*(\bar{\lambda},\alpha)\rangle\leq0,\hspace{2mm}\forall\,\theta\in\mathcal{F}^{\alpha}.
		\end{align}
		By letting $\theta=\theta^*(\lambda,\alpha)$, the inequality in (\ref{ineqn:n1_o_SGL}) leads to
		\begin{align}\label{ineqn:n1_1_SGL}
		\langle\mathbf{n}_{\alpha}(\bar{\lambda}),\,\theta^*(\lambda,\alpha)-\theta^*(\bar{\lambda},\alpha)\rangle\leq0.
		\end{align}
		In view of the first half and by letting $\theta=0$, the inequality in (\ref{ineqn:n1_o_SGL}) leads to
		\begin{align}\label{ineqn:n1_r_SGL}
		\langle\mathbf{n}_{\alpha}(\bar{\lambda}),\,0-\theta^*(\bar{\lambda},\alpha)\rangle\leq0\Rightarrow
		\begin{cases}
		\langle\mathbf{n}_{\alpha}(\bar{\lambda}),\,\mathbf{y}\rangle\geq0,\hspace{9mm}\textup{if}\hspace{1mm}\bar{\lambda}=\lambda_{\rm max}^{\alpha},\\
		\|\mathbf{y}\|/{\bar{\lambda}}\geq\|\theta^*(\bar{\lambda},\alpha)\|,\hspace{2mm}\textup{if}\hspace{1mm}\bar{\lambda}<\lambda_{\rm max}^{\alpha}.
		\end{cases}
		\end{align}
		Moreover, the first half also leads to $\frac{\mathbf{y}}{\lambda}-\theta^*(\lambda,\alpha)\in N_{\mathcal{F}^{\alpha}}(\theta^*(\lambda,\alpha))$. Thus, we have
		\begin{align}\label{ineqn:n2_o_SGL}
		\langle\tfrac{\mathbf{y}}{\lambda}-\theta^*(\lambda,\alpha), \theta-\theta^*(\lambda,\alpha)\rangle\leq0,\hspace{2mm}\forall\,\theta\in\mathcal{F}^{\alpha}.
		\end{align}
		By letting $\theta=\theta^*(\bar{\lambda},\alpha)$, the inequality in (\ref{ineqn:n2_o_SGL}) results in
		\begin{align}\label{ineqn:n2_1_SGL}
		\langle\tfrac{\mathbf{y}}{\lambda}-\theta^*(\lambda,\alpha), \theta^*(\bar{\lambda},\alpha)-\theta^*(\lambda,\alpha)\rangle\leq0,\hspace{2mm}\forall\,\theta\in\mathcal{F}^{\alpha}.
		\end{align}
		We can see that the inequality in (\ref{ineqn:n2_1_SGL}) is equivalent to
		\begin{align}\label{ineqn:n2_2_SGL}
		\|\theta^*(\lambda,\alpha)-\theta^*(\bar{\lambda},\alpha)\|^2\leq& \langle\theta^*(\lambda,\alpha)-\theta^*(\bar{\lambda},\alpha),\,\mathbf{v}_{\alpha}(\lambda,\bar{\lambda})\rangle.
		\end{align}
		
		On the other hand, the right hand side of (\ref{ineqn:estimate1_SGL}) can be rewritten as
		\begin{align}\label{ineqn:n2_3_SGL}
		&\langle\theta^*(\lambda,\alpha)-\theta^*(\bar{\lambda},\alpha),\,\mathbf{v}_{\alpha}^{\perp}(\lambda,\bar{\lambda})\rangle\\ \nonumber
		=&\langle\theta^*(\lambda,\alpha)-\theta^*(\bar{\lambda},\alpha),\,\mathbf{v}_{\alpha}(\lambda,\bar{\lambda})\rangle-\langle\theta^*(\lambda,\alpha)-\theta^*(\bar{\lambda},\alpha),\,\mathbf{v}_{\alpha}(\lambda,\bar{\lambda})-\mathbf{v}^{\perp}_{\alpha}(\lambda,\bar{\lambda})\rangle\\ \nonumber
		=&\langle\theta^*(\lambda,\alpha)-\theta^*(\bar{\lambda},\alpha),\,\mathbf{v}_{\alpha}(\lambda,\bar{\lambda})\rangle-\left\langle\theta^*(\lambda,\alpha)-\theta^*(\bar{\lambda},\alpha),\,\tfrac{\langle\mathbf{v}_{\alpha}(\lambda,\bar{\lambda}),\mathbf{n}_{\alpha}(\bar{\lambda})\rangle}{\vphantom{\widetilde{\bar{\lambda}}}\|\mathbf{n}_{\alpha}(\bar{\lambda})\|^2}\mathbf{n}_{\alpha}(\bar{\lambda})\right\rangle.
		\end{align}
		
		In view of (\ref{ineqn:n1_1_SGL}), (\ref{ineqn:n2_2_SGL}) and (\ref{ineqn:n2_3_SGL}), we can see that (\ref{ineqn:estimate1_SGL}) holds if $\langle\mathbf{v}_{\alpha}(\lambda,\bar{\lambda}),\mathbf{n}_{\alpha}(\bar{\lambda})\rangle\geq0$. Indeed, 
		\begin{align}\label{eqn:ivn_SGL}
		\langle\mathbf{v}_{\alpha}(\lambda,\bar{\lambda}),\mathbf{n}_{\alpha}(\bar{\lambda})\rangle=&\left\langle{\mathbf{y}}/{\lambda}-\theta^*(\bar{\lambda},\alpha),\mathbf{n}_{\alpha}(\bar{\lambda})\right\rangle\\ \nonumber
		=&\left({1}/{\lambda}-{1}/{\bar{\lambda}}\right)\langle\mathbf{y},\mathbf{n}_{\alpha}(\bar{\lambda})\rangle+\langle{\mathbf{y}}/{\bar{\lambda}}-\theta^*(\bar{\lambda},\alpha),\mathbf{n}_{\alpha}(\bar{\lambda})\rangle
		\end{align}
		Consider the first term on the right hand side of \eqref{eqn:ivn_SGL}. 
		By the first half of (\ref{ineqn:n1_r_SGL}), we have
		\begin{align}\label{ineqn:iyn_1_SGL}
		\langle\mathbf{y},\mathbf{n}_{\alpha}(\bar{\lambda})\rangle\geq0,\hspace{2mm}\textup{if}\hspace{1mm}\bar{\lambda}=\lambda_{\rm max}^{\alpha}.
		\end{align}
		Suppose that $\bar{\lambda}<\lambda_{\rm max}^{\alpha}$. By the second half of (\ref{ineqn:n1_r_SGL}), we can see that
		\begin{align}\label{ineqn:iyn_2_SGL}
		\langle\mathbf{y},\mathbf{n}_{\alpha}(\bar{\lambda})\rangle=\langle\mathbf{y},{\mathbf{y}}/{\bar{\lambda}}-\theta^*(\bar{\lambda},\alpha)\rangle\geq 1/{\bar{\lambda}}\|\mathbf{y}\|^2-\|\mathbf{y}\|\|\theta^*(\bar{\lambda},\alpha)\|\geq0.
		\end{align}
		Consider the second term on the right hand side of \eqref{eqn:ivn_SGL}. It is easy to see that
		\begin{align}\label{eqn:ivn_2_SGL}
		\langle{\mathbf{y}}/{\bar{\lambda}}-\theta^*(\bar{\lambda},\alpha),\mathbf{n}_{\alpha}(\bar{\lambda})\rangle=
		\begin{cases}
		0,\hspace{14mm}\textup{if}\hspace{1mm}\bar{\lambda}=\lambda_{\rm max}^{\alpha},\\
		\|\mathbf{n}_{\alpha}(\bar{\lambda})\|^2,\hspace{2mm}\textup{if}\hspace{1mm}\bar{\lambda}<\lambda_{\rm max}^{\alpha}.
		\end{cases}
		\end{align}
		Combining (\ref{ineqn:iyn_1_SGL}), (\ref{ineqn:iyn_2_SGL}) and \eqref{eqn:ivn_2_SGL}, we have $\langle\mathbf{v}_{\alpha}(\lambda,\bar{\lambda}),\mathbf{n}_{\alpha}(\bar{\lambda})\rangle\geq0$, which completes the proof.
	\end{enumerate}
\end{proof}

For notational convenience, we denote
\begin{align}
\mathbf{o}_{\alpha}(\lambda,\bar{\lambda})=\theta^*(\bar{\lambda},\alpha)+\tfrac{1}{2}\mathbf{v}^{\perp}_{\alpha}(\lambda,\bar{\lambda}).
\end{align}
Theorem \ref{thm:estimation_SGL} shows that $\theta^*(\lambda,\alpha)$ lies inside the ball of radius $\tfrac{1}{2}\|\mathbf{v}^{\perp}_{\alpha}(\lambda,\bar{\lambda})\|$ centered at $\mathbf{o}_{\alpha}(\lambda,\bar{\lambda})$.

\subsection{Solving for the Supreme Values via Nonconvex Optimization}\label{subsection:optimization_SGL}

We solve the optimization problems in (\ref{rrule1_SGL}) and (\ref{rrule2_SGL}). To simplify notations, let
\begin{align}\label{eqn:ball1_SGL}
	\Theta &= \{\theta:\|\theta-\mathbf{o}_{\alpha}(\lambda,\bar{\lambda})\|\leq \tfrac{1}{2}\|\mathbf{v}^{\perp}_{\alpha}(\lambda,\bar{\lambda})\|\},\\ \label{eqn:ball2_SGL}
	\Xi_g &=\left\{\xi_g:\|\xi_g-\mathbf{X}_g^T\mathbf{o}_{\alpha}(\lambda,\bar{\lambda})\|\leq \tfrac{1}{2}\|\mathbf{v}^{\perp}_{\alpha}(\lambda,\bar{\lambda})\|\|\mathbf{X}_g\|_2\right\},\,g=1,\ldots,G.
\end{align}
Theorem \ref{thm:estimation_SGL} indicates that $\theta^*(\lambda,\alpha)\in\Theta$. Moreover, we can see that $\mathbf{X}_g^{T}\Theta\subseteq\Xi_g$, $g=1,\ldots,G$. To develop the TLFre rule by (\ref{rrule1_SGL}) and (\ref{rrule2_SGL}), we need to solve the following  optimization problems:
\begin{align}\label{prob:supreme1_SGL}
	s_g^*(\lambda,\bar{\lambda};\alpha)=&\sup\nolimits_{\xi_g}\,\{\|\mathcal{S}_{1}(\xi_g)\|:\xi_g\in\Xi_g\},\,g=1,\ldots,G,\\ \label{prob:supreme2_SGL}
	t_{g_i}^*(\lambda,\bar{\lambda};\alpha)=&\sup\nolimits_{\theta}\,\{|\mathbf{x}_{g_i}^T\theta|:\theta\in\Theta\},\,i=1,\ldots,n_g,\,g=1,\ldots,G.
\end{align}

\subsubsection{The Solution of Problem (\ref{prob:supreme1_SGL})} We consider the following equivalent problem of (\ref{prob:supreme1_SGL}):
\begin{align}\label{prob:supreme11_SGL}
	\tfrac{1}{2}\left(s_g^*(\lambda,\bar{\lambda};\alpha)\right)^2=\sup\nolimits_{\xi_g}\,\left\{\tfrac{1}{2}\|\mathcal{S}_{1}(\xi_g)\|^2:\xi_g\in\Xi_g\right\}.
\end{align}
We can see that the objective function of problem (\ref{prob:supreme11_SGL}) is \emph{continuously differentiable} and the feasible set is a ball. Thus,  problem (\ref{prob:supreme11_SGL}) is  \emph{nonconvex} because we need to \emph{maximize} a convex function subject to a convex set. We first derive the necessary optimality conditions in Lemma \ref{lemma:optimality_condition_supreme1_SGL} and then deduce the closed form solutions of problems (\ref{prob:supreme1_SGL}) and (\ref{prob:supreme11_SGL}) in Theorem \ref{thm:supreme1_SGL}. 

\begin{lemma}\label{lemma:optimality_condition_supreme1_SGL}
	Let $\Xi_g^*$ be the set of optimal solutions of  (\ref{prob:supreme11_SGL}) and $\xi_g^*\in\Xi_g^*$. Then, the following hold:\\
	\textup{(i)} Suppose that $\xi_g^*$ is an interior point of $\Xi_g$. Then, $\Xi_g$ is a subset of $\mathcal{B}_{\infty}$.  \\
	\textup{(ii)} Suppose that $\xi_g^*$ is a boundary point of $\Xi_g$. Then, there exists $\mu^*\geq0$ such that
	\begin{align}\label{eqn:opt_condition_supreme1_SGL}
	\mathcal{S}_{1}(\xi_g^*)=\mu^*\left(\xi_g^*-\mathbf{X}_g^T\mathbf{o}_{\alpha}(\lambda,\bar{\lambda})\right).
	\end{align}
	\textup{(iii)} Suppose that there exists $\xi_g^0\in\Xi_g$ and $\xi_g^0\notin\mathcal{B}_{\infty}$. Then, we have\\
	\indent\hspace{5mm} \textup{(iiia)} $\xi_g^*\notin\mathcal{B}_{\infty}$ and $\xi_g^*$ is a boundary point of $\Xi_g$, i.e., $$\|\xi_g^*-\mathbf{X}_g^T\mathbf{o}_{\alpha}(\lambda,\bar{\lambda})\|=\tfrac{1}{2}\|\mathbf{v}^{\perp}_{\alpha}(\lambda,\bar{\lambda})\|\|\mathbf{X}_g\|_2.$$
	
	\indent\hspace{5mm} \textup{(iiib)} The optimality condition in \eqref{eqn:opt_condition_supreme1_SGL} holds with $\mu^*>0$.
\end{lemma}

To show Lemma \ref{lemma:optimality_condition_supreme1_SGL}, we need the following proposition.

\begin{proposition}\label{prop:optimal_condition_nonconvex}
	\textup{\cite{Hiriart-Urruty1988}} Suppose that $h\in\Gamma_0$ and $\mathcal{C}$ is a nonempty closed convex set. If $\mathbf{w}^*\in\mathcal{C}$ is a local maximum of $h$ on $\mathcal{C}$, then $\partial h(\mathbf{w}^*)\subseteq N_{\mathcal{C}}(\mathbf{w}^*)$.
\end{proposition}

We now present the proof of Lemma \ref{lemma:optimality_condition_supreme1_SGL}.

\begin{proof}
	To simplify notations, let
	\begin{align}\label{eqn:c_r_supreme1_SGL}
	\mathbf{c}=\mathbf{X}_g^T\mathbf{o}_{\alpha}(\lambda,\bar{\lambda})\hspace{2mm}\textup{and}\hspace{1mm} r = \frac{1}{2}\|\mathbf{v}_{\alpha}^{\perp}(\lambda,\bar{\lambda})\|\|\mathbf{X}_g\|_2.
	\end{align}
	By \eqref{def:shrinkage_SGL}, we have
	\begin{align}\label{eqn:obj_fun_supreme1_SGL}
	h(\mathbf{w}):=\frac{1}{2}\|\mathcal{S}_{1}(\mathbf{w})\|^2=\frac{1}{2}\sum_i\,(|[\mathbf{w}]_i|-1)_+^2.
	\end{align}
	It is easy to see that $h(\cdot)$ is continuously differentiable. Indeed, we have
	\begin{align}\label{eqn:gradient_supreme1_SGL}
	\nabla h(\mathbf{w})=
	\mathcal{S}_{1}(\mathbf{w}).
	\end{align}
	Then, problem (\ref{prob:supreme11_SGL}) can be written as
	\begin{align}\label{prob:supreme12_SGL}
	\frac{1}{2}(s_g^*(\lambda,\bar{\lambda};\alpha))^2=\sup_{\xi_g}\,\left\{h(\xi_g)=\frac{1}{2}\sum_i\,([\xi_g]_i-1)_+^2:\,\xi_g\in\Xi_g\right\},
	\end{align}
	where $\Xi_g=\{\xi_g:\|\xi_g-\mathbf{c}\|\leq r\}$.
	Then, Proposition \ref{prop:optimal_condition_nonconvex} results in
	\begin{align}\label{eqn:general_opt_condition_supreme1_SGL}
	\mathcal{S}_{1}(\xi_g^*)=\nabla h(\xi_g^*)=\partial h(\xi_g^*)\subseteq N_{\Xi_g}(\xi^*_g).
	\end{align}
	\begin{enumerate}
		\item[(i)] Suppose that $\xi_g^*$ is an interior point of $\Xi_g$. Then, we have $N_{\Xi_g}(\xi^*_g)=0$. By \eqref{eqn:general_opt_condition_supreme1_SGL}, we can see that
		\begin{align*}
		0=\mathcal{S}_{1}(\xi_g^*)\Rightarrow 0=\frac{1}{2}\|\mathcal{S}_{1}(\xi_g^*)\|^2=\frac{1}{2}(s_g^*(\lambda,\bar{\lambda};\alpha))^2=\sup_{\xi_g}\,\left\{\frac{1}{2}\|\mathcal{S}_{1}(\xi_g)\|^2:\xi_g\in\Xi_g\right\}.
		\end{align*}
		Therefore, we have
		\begin{align}\label{eqn:Xi_subset_C_supreme1_SGL}
		\|\mathcal{S}_{1}(\xi_g)\|=0,\,\forall\, \xi_g\in\Xi_g.
		\end{align}
		Because $\mathcal{S}_{1}(\xi_g)=\xi_g-\mathbf{P}_{\mathcal{B}_{\infty}}(\xi_g)$ (see Remark \ref{remark:shrinkage_feasible_SGL}), \eqref{eqn:Xi_subset_C_supreme1_SGL} implies that
		\begin{align*}
		\xi_g=\mathbf{P}_{\mathcal{B}_{\infty}}(\xi_g),\,\forall\, \xi_g\in\Xi_g\Rightarrow \xi_g\in\mathcal{B}_{\infty},\,\forall\,\xi_g\in\Xi_g.
		\end{align*}
		This completes the proof.
		\item[(ii)] Suppose that $\xi_g^*$ is a boundary point of $\Xi_g$. We can see that
		\begin{align}\label{eqn:normal_cone_supreme1_SGL}
		N_{\Xi_g}(\xi_g^*)=\{\mu(\xi_g^*-\mathbf{c}),\,\mu\geq0\}.
		\end{align}
		Then, \eqref{eqn:opt_condition_supreme1_SGL} follows by combining \eqref{eqn:normal_cone_supreme1_SGL} and the optimality condition in (\ref{eqn:general_opt_condition_supreme1_SGL}).
		\item[(iii)]
		Suppose that there exists $\xi_g^0\in\Xi_g$ and $\xi_g^0\notin\mathcal{B}_{\infty}$.
		\begin{enumerate}
			\item[(iiia)] The definition of $\xi_g^0$ leads to
			\begin{align*}
			0<\|\mathcal{S}_{1}(\xi_g^0)\|\leq\|\mathcal{S}_{1}(\xi_g^*)\|\Rightarrow \xi_g^*\notin\mathcal{B}_{\infty}.
			\end{align*}
			Moreover, we can see that $\xi_g^*$ is a boundary point of $\Xi_g$. Because if $\xi_g^*$ is an interior point of $\Xi_g$, the first part implies that $\Xi_g\subset\mathcal{B}_{\infty}$. This contradicts with the existence of $\xi_g^0$. Thus, $\xi_g^*$ must be a boundary point of $\Xi_g$, i.e. $\|\xi_g^*-\mathbf{c}\|=r$.
			\item[(iiib)] Because  $\xi_g^*$ is a boundary point of $\Xi_g$, the second part implies that \eqref{eqn:opt_condition_supreme1_SGL} holds.
			Moreover, from (iiia), we know that $\xi_g^*\notin\mathcal{B}_{\infty}$. Therefore, both sides of \eqref{eqn:opt_condition_supreme1_SGL} are nonzero and thus $\mu^*>0$. This completes the proof.
		\end{enumerate}
	\end{enumerate}\vspace{-8mm}
\end{proof}

Based on the necessary optimality conditions in Lemma \ref{lemma:optimality_condition_supreme1_SGL}, we derive the closed form solutions of (\ref{prob:supreme1_SGL}) and (\ref{prob:supreme11_SGL}) in the following Theorem. The notations are the same as the ones in the proof of Lemma \ref{lemma:optimality_condition_supreme1_SGL} [see \eqref{eqn:c_r_supreme1_SGL} and \eqref{eqn:obj_fun_supreme1_SGL}].

\begin{theorem}\label{thm:supreme1_SGL}
	For problems (\ref{prob:supreme1_SGL}) and (\ref{prob:supreme11_SGL}), let $\mathbf{c}=\mathbf{X}_g^T\mathbf{o}_{\alpha}(\lambda,\bar{\lambda})$, $r=\tfrac{1}{2}\|\mathbf{v}^{\perp}_{\alpha}(\lambda,\bar{\lambda})\|\|\mathbf{X}_g\|_2$ and $\Xi_g^*$ be the set of the optimal solutions. 
	\begin{enumerate}
		\item Suppose that $\mathbf{c}\notin\mathcal{B}_{\infty}$, i.e., $\|\mathbf{c}\|_{\infty}>1$. Let $\mathbf{u}={r\mathcal{S}_{1}(\mathbf{c})}/{\|\mathcal{S}_{1}(\mathbf{c})\|}$. Then, 
		\begin{align}\label{eqn:sol1_supreme1_SGL}
		s_g^*(\lambda,\bar{\lambda};\alpha)={\|\mathcal{S}_{1}(\mathbf{c})\|}+r\hspace{2mm}\textup{and}\hspace{2mm}\Xi_g^*=\{\mathbf{c}+\mathbf{u}\}.
		\end{align}
		\item Suppose that $\mathbf{c}$ is a boundary point of $\mathcal{B}_{\infty}$, i.e., $\|\mathbf{c}\|_{\infty}=1$. Then, 
		\begin{align}\label{eqn:sol2_supreme1_SGL}
		s_g^*(\lambda,\bar{\lambda};\alpha)=r\hspace{2mm}\textup{and}\hspace{2mm}\Xi_g^*=\left\{\mathbf{c}+\mathbf{u}:\mathbf{u}\in N_{\mathcal{B}_{\infty}}(\mathbf{c}),\|\mathbf{u}\|=r\right\}.
		\end{align}
		\item Suppose that $\mathbf{c}\in\textup{int}\,\mathcal{B}_{\infty}$, i.e., $\|\mathbf{c}\|_{\infty}<1$. Let $i^*\in\mathcal{I}^*=\{i:|[\mathbf{c}]_i|=\|\mathbf{c}\|_{\infty}\}$. Then, 
		\begin{align}\label{eqn:sol3_supreme1_SGL}
		&s_g^*(\lambda,\bar{\lambda};\alpha)=\left(\|\mathbf{c}\|_{\infty}+r-1\right)_+,\\ \nonumber
		&\Xi_g^*=
		\begin{cases}
		\Xi_g,\hspace{44.5mm}\textup{if}\hspace{1mm}\Xi_g\subset\mathcal{B}_{\infty},\\
		\left\{\mathbf{c}+r\cdot\textup{sgn}([\mathbf{c}]_{i^*})\mathbf{e}_{i^*}:\,i^*\in\mathcal{I}^*\right\},\hspace{1.5mm}\textup{if}\hspace{1mm}\Xi_g\not\subset\mathcal{B}_{\infty}\,\textup{and}\,\mathbf{c}\neq0,\\
		\left\{r\cdot\mathbf{e}_{i^*}, -r\cdot\mathbf{e}_{i^*}:\,i^*\in\mathcal{I}^*\right\},\hspace{8.5mm}\textup{if}\hspace{1mm}\Xi_g\not\subset\mathcal{B}_{\infty}\,\textup{and}\,\mathbf{c}=0,\\
		\end{cases}
		\end{align}
		where $\mathbf{e}_i$ is the $i^{th}$ standard basis vector.
	\end{enumerate}
\end{theorem}

\begin{proof}
	\begin{enumerate}
		\item[(i)] Suppose that $\mathbf{c}\notin\mathcal{B}_{\infty}$. By the third part of Lemma \ref{lemma:optimality_condition_supreme1_SGL}, we have
		\begin{align} \label{eqn:opt_sol_bdy_SGL}
		&\xi_g^*\notin\mathcal{B}_{\infty},\hspace{2mm}\|\xi_g^*-\mathbf{c}\|=r,\\
		\label{eqn:opt_cond_case1_SGL}
		&\xi_g^*-\mathbf{P}_{\mathcal{B}_{\infty}}(\xi_g^*)=\mathcal{S}_{1}(\xi_g^*)=\mu^*(\xi_g^*-\mathbf{c}),\,\mu^*>0.
		\end{align}
		By \eqref{eqn:opt_cond_case1_SGL}, we can see that $\mu^*\neq1$ because otherwise we would have $\mathbf{c}=\mathbf{P}_{\mathcal{B}_{\infty}}(\xi_g^*)\in\mathcal{B}_{\infty}$. Moreover, we can only consider the cases with $\mu^*>1$ because $\|\mathcal{S}_{1}(\xi_g^*)\|=\mu^*r$ and we aim to maximize $\|\mathcal{S}_{1}(\xi_g^*)\|$. Therefore, if we can find a solution with $\mu^*>1$, there is no need to consider the cases with $\mu^*\in(0,1)$.
		
		Suppose that $\mu^*>1$. Then, \eqref{eqn:opt_cond_case1_SGL} leads to
		\begin{align}\label{eqn:c_supreme1_SGL}
		\mathbf{c}=&\mathbf{P}_{\mathcal{B}_{\infty}}(\xi_g^*)+\left(1-\frac{1}{\mu^*}\right)\left(\xi_g^*-\mathbf{P}_{\mathcal{B}_{\infty}}(\xi_g^*)\right),\\ \label{eqn:xi_supreme1_SGL}
		\xi_g^*=&\mathbf{P}_{\mathcal{B}_{\infty}}(\xi_g^*)+\frac{\mu^*}{\mu^*-1}\left(\mathbf{c}-\mathbf{P}_{\mathcal{B}_{\infty}}(\xi_g^*)\right).
		\end{align}
		In view of part (iv) of Proposition \ref{prop:normal_cone} and \eqref{eqn:c_supreme1_SGL}, we have
		\begin{align}\label{eqn:projc_projxi_supreme1_SGL}
		\mathbf{P}_{\mathcal{B}_{\infty}}(\mathbf{c})=\mathbf{P}_{\mathcal{B}_{\infty}}(\xi_g^*).
		\end{align}
		Therefore, \eqref{eqn:xi_supreme1_SGL} can be rewritten as
		\begin{align}\label{eqn:shrink_xi_c_supreme1_SGL}
		\mathcal{S}_{1}(\xi_g^*)=\xi_g^*-\mathbf{P}_{\mathcal{B}_{\infty}}(\xi_g^*)=\frac{\mu^*}{\mu^*-1}\left(\mathbf{c}-\mathbf{P}_{\mathcal{B}_{\infty}}(\mathbf{c})\right)=\frac{\mu^*}{\mu^*-1}\mathcal{S}_{1}(\mathbf{c}).
		\end{align}
		Combining \eqref{eqn:opt_cond_case1_SGL} and \eqref{eqn:shrink_xi_c_supreme1_SGL}, we have
		\begin{align}\label{eqn:mu_supreme1_SGL}
		\frac{\mu^*}{\mu^*-1}\|\mathcal{S}_{1}(\mathbf{c})\|=\mu^*\|\xi_g^*-\mathbf{c}\|=\mu^*r\Rightarrow\mu^*=1+\frac{\|\mathcal{S}_{1}(\mathbf{c})\|}{r}>1.
		\end{align}
		The statement holds by plugging \eqref{eqn:mu_supreme1_SGL} and \eqref{eqn:projc_projxi_supreme1_SGL} into \eqref{eqn:xi_supreme1_SGL} and \eqref{eqn:shrink_xi_c_supreme1_SGL}. Moreover, the above discussion implies that $\Xi_g^*$ only contains one element as shown in \eqref{eqn:sol1_supreme1_SGL}.
		
		\item[(ii)] Suppose that $\mathbf{c}$ is a boundary point of $\mathcal{B}_{\infty}$. Then, we can find a point $\xi_g^0\in\Xi_g$ and $\xi_g^0\notin\mathcal{B}_{\infty}$. By the third part of Lemma \ref{lemma:optimality_condition_supreme1_SGL}, we also have \eqref{eqn:opt_sol_bdy_SGL} and \eqref{eqn:opt_cond_case1_SGL} hold. We claim that $\mu^*\in(0,1]$. The argument is as follows.
		
		Suppose that $\mu^*>1$. By the same argument as in the proof of the first part, we can see that \eqref{eqn:shrink_xi_c_supreme1_SGL} holds. Because $\mathcal{S}_{1}(\xi_g^*)\neq0$ by \eqref{eqn:opt_sol_bdy_SGL}, we have $\mathcal{S}_{1}(\mathbf{c})\neq0$. This implies that $\mathbf{c}\notin\mathcal{B}_{\infty}$. Thus, we have a contradiction, which implies that $\mu^*\in(0,1]$.
		
		Let us consider the cases with $\mu^*=1$. Because $\|\mathcal{S}_{1}(\xi_g^*)\|=\mu^*r$ [see \eqref{eqn:opt_cond_case1_SGL}] and we want to maximize $\|\mathcal{S}_{1}(\xi_g^*)\|$, there is no need to consider the cases with $\mu^*\in(0,1)$ if we can find solutions of problem (\ref{prob:supreme1_SGL}) with $\mu^*=1$. Therefore, \eqref{eqn:opt_cond_case1_SGL} leads to
		\begin{align*}
		\mathbf{P}_{\mathcal{B}_{\infty}}(\xi_g^*)=\mathbf{c}.
		\end{align*}
		By part (iii) of Proposition \ref{prop:normal_cone}, we can see that
		\begin{align}\label{eqn:pxi_c_supreme1_SGL}
		\mathbf{P}_{\mathcal{B}_{\infty}}(\xi_g^*)=\mathbf{c}\Leftrightarrow \xi_g^*-\mathbf{c}\in N_{\mathcal{B}_{\infty}}(\mathbf{c}).
		\end{align}
		Combining \eqref{eqn:pxi_c_supreme1_SGL} and \eqref{eqn:opt_sol_bdy_SGL}, the statement holds immediately, which confirms that $\mu^*=1$.
		
		\item[(iii)] Suppose that $\mathbf{c}$ is an interior point of $\mathcal{B}_{\infty}$.
		
		\begin{enumerate}
			\item We first consider the cases with $\Xi_g\subset\mathcal{B}_{\infty}$. Then, we can see that
			\begin{align*}
			\mathcal{S}_{1}(\xi)=0,\,\forall \xi\in\Xi_g\Rightarrow\Xi_g^*=\Xi_g.
			\end{align*}
			In other words, an arbitrary point of $\Xi_g$ is an optimal solution of problem (\ref{prob:supreme1_SGL}).
			Thus, we have
			\begin{align*}
			&\mathbf{c}+r\cdot\textup{sgn}(\mathbf{e}_{i^*})\mathbf{e}_{i^*}\in\Xi_g^*,\\
			&s_g^*(\lambda,\bar{\lambda};\alpha)=0.
			\end{align*}
			On the other hand, we can see that
			\begin{align*}
			\mathbf{c}-r\mathbf{e}_i\in\Xi_g\subset\mathcal{B}_{\infty},\,\mathbf{c}+r\mathbf{e}_i\in\Xi_g\subset\mathcal{B}_{\infty},\,i=1,\ldots,n_g\Rightarrow \|\mathbf{c}\|_{\infty}+r\leq1.
			\end{align*}
			Therefore, we have
			\begin{align*}
			(\|\mathbf{c}\|_{\infty}+r-1)_+=0,
			\end{align*}
			and thus
			\begin{align*}
			s_g^*(\lambda,\bar{\lambda};\alpha)=(\|\mathbf{c}\|_{\infty}+r-1)_+.
			\end{align*}
			\item Suppose that $\Xi_g\not\subset\mathcal{B}_{\infty}$, i.e., there exists $\xi^0\in\Xi_g$ such that $\xi^0\notin \mathcal{B}_{\infty}$. By the third part of Lemma \ref{lemma:optimality_condition_supreme1_SGL}, we have \eqref{eqn:opt_sol_bdy_SGL} and \eqref{eqn:opt_cond_case1_SGL} hold. Moreover, in view of the proof of the first and second part, we can see that $\mu^*\in(0,1)$. Therefore, \eqref{eqn:opt_cond_case1_SGL} leads to
			\begin{align}\label{eqn:conv_comb_supreme1_SGL}
			(1-\mu^*)\xi_g^*+\mu^*\mathbf{c}=\mathbf{P}_{\mathcal{B}_{\infty}}(\xi_g^*).
			\end{align}
			By rearranging the terms of \eqref{eqn:conv_comb_supreme1_SGL}, we have
			\begin{align}\label{eqn:conv_comb2_supreme1_SGL}
			\mathbf{P}_{\mathcal{B}_{\infty}}(\xi_g^*)-\mathbf{c}=(1-\mu^*)(\xi_g^*-\mathbf{c}).
			\end{align}
			Because $\mu^*\in(0,1)$, \eqref{eqn:conv_comb_supreme1_SGL} implies that $\mathbf{P}_{\mathcal{B}_{\infty}}(\xi_g^*)$ lies on the line segment connecting $\xi_g^*$ and $\mathbf{c}$. Thus, we have
			\begin{align}
			\|\xi_g^*-\mathbf{P}_{\mathcal{B}_{\infty}}(\xi_g^*)\|+\|\mathbf{P}_{\mathcal{B}_{\infty}}(\xi_g^*)-\mathbf{c}\|=\|\xi_g^*-\mathbf{c}\|=r.
			\end{align}
			Therefore, to maximize $\|\mathcal{S}_{1}(\xi_g^*)\|=\|\xi_g^*-\mathbf{P}_{\mathcal{B}_{\infty}}(\xi_g^*)\|$, we need to minimize $\|\mathbf{P}_{\mathcal{B}_{\infty}}(\xi_g^*)-\mathbf{c}\|$. Because $\xi_g^*\notin\mathcal{B}_{\infty}$, we can see that $\mathbf{P}_{\mathcal{B}_{\infty}}(\xi_g^*)$ is a boundary point of $\mathcal{B}_{\infty}$. Therefore, we need to solve the following minimization problem:
			\begin{align}\label{prob:proj_faces_supreme1_SGL}
			\min_{\phi_g}\,\{\|\phi_g-\mathbf{c}\|:\|\phi_g\|_{\infty}=1\}.
			\end{align}
			
			Suppose that $\mathbf{c}=0$. We can see that the set of optimal solutions of problem ({\ref{prob:proj_faces_supreme1_SGL}}) is
			\begin{align*}
			\Phi_g^*=\{\mathbf{e}_i\}_{i=1}^{n_g}\cup\{-\mathbf{e}_i\}_{i=1}^{n_g}.
			\end{align*}
			For each $\phi_g^*\in\Phi_g^*$, we set it as $\mathbf{P}_{\mathcal{B}_{\infty}}(\xi_g^*)$. In view of \eqref{eqn:conv_comb2_supreme1_SGL} and \eqref{eqn:opt_sol_bdy_SGL}, the statement follows immediately.
			
			Suppose that $\mathbf{c}\neq0$. Recall that $\mathcal{I}^*=\{i^*:|[\mathbf{c}]_{i^*}|=\|\mathbf{c}\|_{\infty}\}$.
			It is easy to see that
			\begin{align*}
			\Phi_g^*=\left\{\phi_{i^*}:[\phi_{i^*}]_k=
			\begin{cases}
			\textup{sgn}([\mathbf{c}]_{i^*}),\hspace{6mm}\textup{if}\hspace{1mm}k=i^*,\\
			[\mathbf{c}]_{k},\hspace{14mm}\textup{otherwise},
			\end{cases}
			i^*\in\mathcal{I}^*
			\right\}.
			\end{align*}
			We can see that
			\begin{align*}
			\phi_{i^*}-\mathbf{c}=(1-|[\mathbf{c}]_{\infty}|)\textup{sgn}([\mathbf{c}]_{i^*})\mathbf{e}_{i^*},\,i^*\in\mathcal{I}^*.
			\end{align*}
			For each $\phi_{i^*}$, we set it to $\mathbf{P}_{\mathcal{B}_{\infty}}(\xi_g^*)$. Then, we can see that the statement holds by
			\eqref{eqn:conv_comb2_supreme1_SGL} and \eqref{eqn:opt_sol_bdy_SGL}. This completes the proof.
		\end{enumerate}
	\end{enumerate}
\end{proof}

\subsubsection{The Solution of Problem (\ref{prob:supreme2_SGL})} Problem (\ref{prob:supreme2_SGL}) can be solved directly via the Cauchy-Schwarz inequality.
\begin{theorem}\label{thm:supreme2_SGL}
	For problem (\ref{prob:supreme2_SGL}), we have
	$t^*_{g_i}(\lambda,\bar{\lambda};\alpha)=|\mathbf{x}_{g_i}^T\mathbf{o}_{\alpha}(\lambda,\bar{\lambda})|+\tfrac{1}{2}\|\mathbf{v}^{\perp}_{\alpha}(\lambda,\bar{\lambda})\|\|\mathbf{x}_{g_i}\|$.
\end{theorem}

\begin{proof}
	To simplify notations, let $\mathbf{o}=\mathbf{o}_{\alpha}(\lambda,\bar{\lambda})$, $r=\frac{1}{2}\|\mathbf{v}_{\alpha}^{\perp}(\lambda,\bar{\lambda})\|$ and $t^*_g=t^*_g(\lambda,\bar{\lambda};\alpha)$. Therefore, the set $\Theta$ in \eqref{eqn:ball1_SGL} can be written as
	\begin{align*}
	\Theta=\{\mathbf{o}+\mathbf{v}:\|\mathbf{v}\|\leq r\}.
	\end{align*}
	Then, problem (\ref{prob:supreme2_SGL}) becomes
	\begin{align*}
	t_{g_i}^*=\sup_{\mathbf{v}}\,\{|\mathbf{x}_{g_i}^T(\mathbf{o}+\mathbf{v})|:\|\mathbf{v}\|\leq r\}.
	\end{align*}
	We can see that
	\begin{align*}
	|\mathbf{x}_{g_i}^T(\mathbf{o}+\mathbf{v})|\leq |\mathbf{x}_{g_i}^T\mathbf{o}|+|\mathbf{x}_{g_i}^T\mathbf{v}|\leq |\mathbf{x}_{g_i}^T\mathbf{o}|+\|\mathbf{x}_{g_i}\|\|\mathbf{v}\|\leq |\mathbf{x}_{g_i}^T\mathbf{o}|+\|\mathbf{x}_{g_i}\|r.
	\end{align*}
	Thus, we have
	\begin{align*}
	t^*_{g_i}\leq |\mathbf{x}_{g_i}^T\mathbf{o}|+\|\mathbf{x}_{g_i}\|r.
	\end{align*}
	Consider $\mathbf{v}^*=r\mathbf{x}_{g_i}/\|\mathbf{x}_{g_i}\|$. It is easy to see that $\mathbf{o}+\mathbf{v}^*\in\Theta$ and
	\begin{align*}
	|\mathbf{x}_{g_i}^T(\mathbf{o}+\mathbf{v})|=|\mathbf{x}_{g_i}^T\mathbf{o}|+\|\mathbf{x}_{g_i}\|r.
	\end{align*}
	Therefore, we have
	\begin{align*}
	t^*_{g_i}=|\mathbf{x}_{g_i}^T\mathbf{o}|+\|\mathbf{x}_{g_i}\|r,
	\end{align*}
	which completes the proof.
\end{proof}

\subsection{The Proposed Two-Layer Screening Rules}\label{subsection:screening_rules_SGL}

To develop the two-layer screening rules for SGL, we only need to plug the supreme values $s_g^*(\lambda_2,\bar{\lambda}_2;\lambda_1)$ and $t_{g_i}^*(\lambda_2,\bar{\lambda}_2;\lambda_1)$ in (\ref{rrule1_SGL}) and (\ref{rrule2_SGL}). We present the TLFre rule as follows.
\begin{theorem}\label{thm:TLFre_SGL}
	For the SGL problem in (\ref{prob:SGL}), suppose that we are given $\alpha$ and a sequence of parameter values $\lambda_{\rm max}^{\alpha}=\lambda^{(0)}>\lambda^{(1)}>\ldots>\lambda^{(\mathcal{J})}$. Moreover, assume that $\beta^*(\lambda^{(j)},\alpha)$ is known for an integer $0\leq j<\mathcal{J}$. Let $\theta^*(\lambda^{(j)},\alpha)$, $\mathbf{v}_{\alpha}^{\perp}(\lambda^{(j+1)},\lambda^{(j)})$ and $s_g^*(\lambda^{(j+1)},\lambda^{(j)};\alpha)$ be given by \eqref{eqn:KKT1_SGL}, Theorems \ref{thm:estimation_SGL} and \ref{thm:supreme1_SGL}, respectively. Then, for $g=1,\ldots,G$, the following holds
	\begin{align}\tag{$\mathcal{L}_1$}\label{rule:L1}
		s_g^*(\lambda^{(j+1)},\lambda^{(j)};\alpha)<\alpha\sqrt{n_g}\Rightarrow \beta^*_g(\lambda^{(j+1)},\alpha)=0.
	\end{align}
	For the $\hat{g}^{th}$ group that does not pass the rule in (\ref{rule:L1}), we have $[\beta^*_{\hat{g}}(\lambda^{(j+1)},\alpha)]_i=0$ if
	\begin{align}\tag{$\mathcal{L}_2$}\label{rule:L2}
		\left|\mathbf{x}_{\hat{g}_i}^T\left(\frac{\mathbf{y}-\mathbf{X}\beta^*(\lambda^{(j)},\alpha)}{\lambda^{(j)}}+\frac{1}{2}\mathbf{v}_{\alpha}^{\perp}(\lambda^{(j+1)},\lambda^{(j)})\right)\right|+\frac{1}{2}\|\mathbf{v}_{\alpha}^{\perp}(\lambda^{(j+1)},\lambda^{(j)})\|\|\mathbf{x}_{\hat{g}_i}\|\leq1.
	\end{align}
\end{theorem}
(\ref{rule:L1}) and (\ref{rule:L2}) are the first layer and second layer screening rules of TLFre, respectively. 

\section{Extension to Nonnegative Lasso}\label{section:nnlasso}

The framework of TLFre is applicable to a large class of sparse models with multiple regularizers. As an example, we extend TLFre to nonnegative Lasso:
\begin{align}\label{prob:nnlassoo}
\min_{\beta\in\mathbb{R}^p}\,\left\{\frac{1}{2}\left\|\textbf{y}-\textbf{X}\beta\right\|^2+\lambda\|\beta\|_1:\,\beta\in\mathbb{R}_+^p\right\},
\end{align}
where $\lambda>0$ is the regularization parameter and $\mathbb{R}_+^p$ is the nonnegative orthant of $\mathbb{R}^p$. In Section \ref{ssec:fenchel_dual_nnlasso}, we transform the constraint $\beta\in\mathbb{R}_+^p$ to a regularizer and derive the Fenchel's dual of the nonnegative Lasso problem. We then motivate the screening method---called DPC since the key step is to \textbf{d}ecom\textbf{p}ose a \textbf{c}onvex set via Fenchel's Duality Theorem---via the KKT conditions in Section \ref{ssec:general_rule_nnlasso}. In Section \ref{ssec:lambdamx_nnlasso}, we analyze the geometric properties of the dual problem and derive the set of parameter values leading to zero solutions. We then develop the screening method for nonnegative Lasso in Section \ref{ssec:DPC_nnlasso}.

\subsection{The Fenchel's Dual of Nonnegative Lasso}\label{ssec:fenchel_dual_nnlasso}
Let $\textbf{I}_{\mathbb{R}_+^p}$ be the indicator function of $\mathbb{R}_+^p$. By noting that $\textbf{I}_{\mathbb{R}_+^p}=\lambda\textbf{I}_{\mathbb{R}_+^p}$ for any $\lambda>0$, we can rewrite the nonnegative Lasso problem in (\ref{prob:nnlassoo}) as
\begin{align}\label{prob:nnlasso}
\min_{\beta\in\mathbb{R}^p}\,\frac{1}{2}\left\|\textbf{y}-\textbf{X}\beta\right\|^2+\lambda\|\beta\|_1+\lambda\textbf{I}_{\mathbb{R}_+^p}(\beta).
\end{align}
In other words, we incorporate the constraint $\beta\in\mathbb{R}_+^p$ to the objective function as an additional regularizer. As a result, the nonnegative lasso problem in (\ref{prob:nnlasso}) has two regularizers. Thus, similar to SGL, we can derive the Fenchel's dual of nonnegative Lasso via Theorem \ref{thm:fenchel_duality}. 

We now proceed by following a similar procedure as the one in Section \ref{subsection:Fenchel_Dual_SGL}. We note that the nonnegative Lasso problem in (\ref{prob:nnlasso}) can also be formulated as the one in (\ref{prob:general}) with $f(\cdot)=\frac{1}{2}\|\cdot\|^2$ and $\Omega(\beta)=\|\beta\|_1+\textbf{I}_{\mathbb{R}_+^p}(\beta)$. To derive the Fenchel's dual of nonnegative Lasso, we need to find $f^*$ and $\Omega^*$ by Theorem \ref{thm:fenchel_duality}. Since we have already seen that $f^*(\cdot)=\frac{1}{2}\|\cdot\|^2$ in Section \ref{subsection:Fenchel_Dual_SGL}, we only need to find $\Omega^*(\cdot)$.
The following result is indeed a counterpart of Lemma \ref{lemma:conjugate_Omega_SGL}.
\begin{lemma}\label{lemma:conjugate_Omega_nnlasso}
	Let $\Omega_2(\beta)=\|\beta\|_1$, $\Omega_3=\textup{\textbf{I}}_{\mathbb{R}_+^p}(\beta)$, and $\Omega(\beta)=\Omega_2(\beta)+\Omega_3(\beta)$. Then, 
	\begin{enumerate}
		\item $(\Omega_2)^*(\xi)=\textup{\textbf{I}}_{\mathcal{B}_{\infty}}(\xi)$ and $(\Omega_3)^*(\xi)=\textup{\textbf{I}}_{\mathbb{R}_-^p}(\xi)$, where $\mathbb{R}_-^p$ is the nonpositive orthant of $\mathbb{R}^p$.
		\item $\Omega^*(\xi)=((\Omega_2)^*\boxdot(\Omega_3)^*)(\xi)=\textup{\textbf{I}}_{\mathbb{R}_-^p}(\xi-\textup{\textbf{1}})$, where $\mathbb{R}^p\ni\textup{\textbf{1}}=(1,1,\ldots,1)^T$.
	\end{enumerate}
\end{lemma}
We omit the proof of Lemma \ref{lemma:conjugate_Omega_nnlasso} since it is very similar to that of Lemma \ref{lemma:conjugate_Omega_SGL}.

\begin{remark}
Consider the second part of Lemma \ref{lemma:conjugate_Omega_nnlasso}. Let $\mathcal{C}_{\textup{\textbf{1}}}=\{\xi:\xi\leq\textup{\textbf{1}}\}$, where ``$\leq$" is defined component-wisely. We can see that 
\begin{align*}
\textup{\textbf{I}}_{\mathbb{R}_-^p}(\xi-\textup{\textbf{1}})=\textup{\textbf{I}}_{\mathcal{C}_{\textup{\textbf{1}}}}(\xi).
\end{align*}
On the other hand, Lemma \ref{lemma:infconv_sets} implies that
\begin{align*}
\Omega^*(\xi)=((\Omega_2)^*\boxdot(\Omega_3)^*)(\xi)=\textup{\textbf{I}}_{\mathcal{B}_{\infty}+\mathbb{R}_-^p}(\xi).
\end{align*}
Thus, we have $\mathcal{B}_{\infty}+\mathbb{R}_-^p=\mathcal{C}_{\textup{\textbf{1}}}$. The second part of Lemma \ref{lemma:conjugate_Omega_nnlasso} decomposes each $\xi\in\mathcal{B}_{\infty}+\mathbb{R}_-^p$ into two components: $\textup{\textbf{1}}$ and $\xi-\textup{\textbf{1}}$ that belong to $\mathcal{B}_{\infty}$ and $\mathbb{R}_-^p$, respectively.
\end{remark}

By Theorem \ref{thm:fenchel_duality} and Lemma \ref{lemma:conjugate_Omega_nnlasso}, we can derive the Fenchel's dual of nonnegative Lasso in the following theorem (which is indeed the counterpart of Theorem \ref{thm:dual_SGL}).

\begin{theorem}\label{thm:dual_nnlasso}
	For the nonnegative Lasso problem, the following hold:
	\begin{enumerate}
		\item The Fenchel's dual of nonnegative Lasso is given by:
		\begin{align}\label{prob:nnlasso_dual}
		\inf_{\theta}\,\left\{\frac{1}{2}\left\|\frac{\textup{\textbf{y}}}{\lambda}-\theta\right\|^2-\frac{1}{2}\|\textup{\textbf{y}}\|^2:\langle\textup{\textbf{x}}_i,\theta\rangle\leq1,\,i=1,\ldots,p\right\}.
		\end{align}
		\item Let $\beta^*(\lambda)$ and $\theta^*(\lambda)$ be the optimal solutions of problems (\ref{prob:nnlasso}) and (\ref{prob:nnlasso_dual}), respectively.Then, 
		\begin{align}\label{eqn:KKT1_nnlasso}
		\lambda\theta^*(\lambda)&=\textup{\textbf{y}}-\textup{\textbf{X}}\beta^*(\lambda),\\\label{eqn:KKT2_nnlasso}
		\textup{\textbf{X}}^T\theta^*(\lambda)&\in\partial\|\beta^*(\lambda)\|_1+\partial \textup{\textbf{I}}_{\mathbb{R}_+^p}(\beta^*(\lambda)).
		\end{align}
	\end{enumerate}
\end{theorem}
We omit the proof of Theorem \ref{thm:dual_nnlasso} since it is very similar to that of Theorem \ref{thm:dual_SGL}.

\subsection{Motivation of the Screening Method via KKT Conditions}\label{ssec:general_rule_nnlasso}

The key to develop the DPC rule for nonnegative lasso is the KKT condition in (\ref{eqn:KKT2_nnlasso}). We can see that $\partial\|\mathbf{w}\|_1=\textup{SGN}(\mathbf{w})$ and
\begin{align*}
\partial \mathbf{I}_{\mathbb{R}_+^p}(\mathbf{w})=\left\{\xi\in\mathbb{R}^p:
[\xi]_i=
\begin{cases}
0,\hspace{13mm}\mbox{if }[\mathbf{w}]_i>0,\\
\rho,\,\rho\leq0,\hspace{2mm}\mbox{if }[\mathbf{w}]_i=0,\\
\end{cases}
\right\}.
\end{align*}
Therefore, the KKT condition in (\ref{eqn:KKT2_nnlasso}) implies that
\begin{align}\label{eqn:KKT3_nnlasso}
\langle\mathbf{x}_i,\theta^*(\lambda)\rangle\in
\begin{cases}
1,\hspace{13mm}\mbox{if }[\beta^*(\lambda)]_i>0,\\
\varrho,\,\varrho\leq1,\hspace{2mm}\mbox{if }[\beta^*(\lambda)]_i=0.
\end{cases}
\end{align}
By \eqref{eqn:KKT3_nnlasso}, we have the following rule:
\begin{align}\tag{R3}\label{R3}
\langle\mathbf{x}_i,\theta^*(\lambda)\rangle<1\Rightarrow[\beta^*(\lambda)]_i=0.
\end{align}
Because $\theta^*(\lambda)$ is unknown, we can apply (\ref{R3}) to identify the inactive features---which have $0$ coefficients in $\beta^*(\lambda)$. Similar to TLFre, we can first find a region $\Theta$ that contains $\theta^*(\lambda)$. Then, we can relax (\ref{R3}) as follows:
\begin{align}\tag{R3$^*$}\label{rrule3_nnlasso}
\sup_{\theta\in\Theta}\,\langle\mathbf{x}_i,\theta\rangle<1\Rightarrow[\beta^*(\lambda)]_i=0.
\end{align}

Inspired by (\ref{rrule3_nnlasso}), we develop DPC via the following three steps:
\begin{enumerate}
	\item[\textbf{Step 1}.] Given $\lambda$, we estimate a region $\Theta$ that contains $\theta^*(\lambda)$.
	\item[\textbf{Step 2}.] We solve the optimization problem $\omega_i=\sup_{\theta\in\Theta}\,\langle\mathbf{x}_i,\theta\rangle$.
	\item[\textbf{Step 3}.] By plugging in $\omega_i$ computed from \textbf{Step 2}, (\ref{rrule3_nnlasso}) leads to the desired screening method DPC for nonnegative Lasso.
\end{enumerate}

\subsection{Geometric Properties of the Fenchel's Dual of Nonnegative Lasso}\label{ssec:lambdamx_nnlasso}

In view of the Fenchel's dual of nonnegative Lasso in (\ref{prob:nnlasso_dual}), we can see that the optimal solution is indeed the projection of $\textbf{y}/\lambda$ onto the feasible set $\mathcal{F}=\{\theta:\langle\textbf{x}_i,\theta\rangle\leq1,\,i=1,\ldots,p\}$, i.e.,
\begin{align}\label{eqn:dual_proj_nnlasso}
\theta^*(\lambda)=\textup{\textbf{P}}_{\mathcal{F}}\left(\frac{\textbf{y}}{\lambda}\right).
\end{align}
Therefore, if $\mathbf{y}/\lambda\in\mathcal{F}$, \eqref{eqn:dual_proj_nnlasso} implies that $\theta^*(\lambda)=\mathbf{y}/\lambda$. If further $\mathbf{y}/\lambda$ is an interior point of $\mathcal{F}$, \ref{rrule3_nnlasso} implies that $\beta^*(\lambda)=0$. The next theorem gives the set of parameter values leading to $0$ solutions of nonnegative Lasso.
\begin{theorem}\label{thm:lambdamx_nnlasso}
	For the nonnegative Lasso problem (\ref{prob:nnlasso}), Let $\lambda_{\textup{max}}=\max_{i}\langle\mathbf{x}_i,\mathbf{y}\rangle$. Then, the following statements are equivalent:
	
	\vspace{2mm}
	\begin{enumerate*}[itemjoin=\hfill]
		\item $\dfrac{\mathbf{y}}{\lambda}\in\mathcal{F}$,
		\item $\theta^*(\lambda)=\dfrac{\mathbf{y}}{\lambda}$,
		\item $\beta^*(\lambda)=0$,
		\item $\lambda\geq\lambda_{\textup{max}}$.
	\end{enumerate*}
\end{theorem}
We omit the proof of Theorem \ref{thm:lambdamx_nnlasso} since it is very similar to that of Theorem \ref{thm:lambda_alpha_SGL}.

\subsection{The Proposed Screening Rule for Nonnegative Lasso}\label{ssec:DPC_nnlasso}

We follow the three steps in Section \ref{ssec:general_rule_nnlasso} to develop the screening rule for nonnegative Lasso. We first estimate a region that contains $\theta^*(\lambda)$. Because $\theta^*(\lambda)$ admits a closed form solution with $\lambda\geq\lambda_{\textup{max}}$ by Theorem \ref{thm:lambdamx_nnlasso}, we focus on the cases with $\lambda<\lambda_{\textup{max}}$. 
\begin{theorem}\label{thm:estimation_nnlasso}
	For the nonnegative Lasso problem, suppose that $\theta^*({\bar{\lambda}})$ is known with $\bar{\lambda}\leq\lambda_{\rm max}$. For any $\lambda\in(0,\bar{\lambda})$, we define
	\begin{align*}
	&\mathbf{n}(\bar{\lambda})=
	\begin{cases}
	\dfrac{\mathbf{y}}{\bar{\lambda}}-\theta^*(\bar{\lambda}),\hspace{5.5mm}\textup{if}\hspace{1mm}\bar{\lambda}<\lambda_{\rm max}^{\alpha},\\
	\mathbf{x}_*,\hspace{18.5mm}\textup{if}\hspace{1mm}\bar{\lambda}=\lambda_{\rm max},
	\end{cases}
	\hspace{-3mm}\textup{where}\hspace{1mm}\mathbf{x}_*=\textup{argmax}_{\mathbf{x}_i}\,\langle\mathbf{x}_i,\mathbf{y}\rangle,\\
	&\mathbf{v}(\lambda,\bar{\lambda})=\frac{\mathbf{y}}{\lambda}-\theta^*(\bar{\lambda}), \\
	&\mathbf{v}(\lambda,\bar{\lambda})^{\perp} = \mathbf{v}(\lambda,\bar{\lambda}) - \frac{\langle\mathbf{v}(\lambda,\bar{\lambda}),\mathbf{n}(\bar{\lambda})\rangle}{\vphantom{\widetilde{\bar{\lambda}}}\|\mathbf{n}(\bar{\lambda})\|^2}\mathbf{n}(\bar{\lambda}).
	\end{align*}
	Then, the following hold:
	\begin{enumerate}
		\item $\mathbf{n}(\bar{\lambda})\in N_{\mathcal{F}}(\theta^*(\bar{\lambda}))$,
		\item $\left\|\theta^*(\lambda)-\left(\theta^*(\bar{\lambda})+\dfrac{1}{2}\mathbf{v}^{\perp}(\lambda,\bar{\lambda})\right)\right\|\leq \dfrac{1}{2}\|\mathbf{v}^{\perp}(\lambda,\bar{\lambda})\|$.
	\end{enumerate}
\end{theorem}

\begin{proof}
	We only show that $\mathbf{n}(\lambda_{\textup{max}})\in N_{\mathcal{F}}(\theta^*(\lambda_{\textup{max}}))$ since the proof of the other statement is very similar to that of Theorem \ref{thm:estimation_SGL}. 
	
	By Proposition \ref{prop:normal_cone} and Theorem \ref{thm:lambdamx_nnlasso}, it suffices to show that
	\begin{align}\label{ineqn:xmx_nnlasso}
	\langle\mathbf{x}_*,\theta-\mathbf{y}/\lambda_{\textup{max}}\rangle\leq0,\,\forall\,\theta\in\mathcal{F}.
	\end{align}
	Because $\theta\in\mathcal{F}$, we have $\langle\mathbf{x}_*,\theta\rangle\leq1$. The definition of $\mathbf{x}_*$ implies that $\langle\mathbf{x}_*,\mathbf{y}/\lambda_{\textup{max}}\rangle=1$. Thus, the inequality in (\ref{ineqn:xmx_nnlasso}) holds, which completes the proof. 
\end{proof}

Theorem \ref{thm:estimation_nnlasso} implies that $\theta^*(\lambda)$ is in a ball---denoted by $\mathcal{B}(\lambda,\bar{\lambda})$---of radius $\frac{1}{2}\|\mathbf{v}^{\perp}(\lambda,\bar{\lambda})\|$ centered at $\theta^*(\bar{\lambda})+\frac{1}{2}\mathbf{v}^{\perp}(\lambda,\bar{\lambda})$. Simple calculations lead to
\begin{align}
\omega_i=\sup_{\theta\in\mathcal{B}(\lambda,\bar{\lambda})}\,\langle\mathbf{x}_i,\theta\rangle=\left\langle\mathbf{x}_i,\theta^*(\bar{\lambda})+\frac{1}{2}\mathbf{v}^{\perp}(\lambda,\bar{\lambda})\right\rangle+\frac{1}{2}\|\mathbf{v}^{\perp}(\lambda,\bar{\lambda})\|\|\mathbf{x}_i\|.
\end{align}
By plugging $\omega_i$ into (\ref{rrule3_nnlasso}), we have the DPC screening rule for nonnegative Lasso as follows.
\begin{theorem}
For the nonnegative Lasso problem, suppose that we are given a sequence of parameter values $\lambda_{\textup{max}}=\lambda^{(0)}>\lambda^{(1)}>\ldots>\lambda^{(\mathcal{J})}$. Then, $[\beta^*(\lambda^{(j+1)})]_i=0$ if $\beta^*(\lambda^{(j)})$ is known and the following holds:
\begin{align}
\left\langle\mathbf{x}_i,\frac{\mathbf{y}-\mathbf{X}\beta^*(\lambda^{(j)})}{\lambda^{(j)}}+\frac{1}{2}\mathbf{v}^{\perp}(\lambda^{(j+1)},\lambda^{(j)})\right\rangle+\frac{1}{2}\|\mathbf{v}^{\perp}(\lambda^{(j+1)},\lambda^{(j)})\|\|\mathbf{x}_i\|<1.
\end{align} 
\end{theorem}

\section{Experiments}\label{section:experiments}

We evaluate TLFre for SGL and DPC for nonnegative Lasso in Sections \ref{ssec:exp_SGL} and \ref{ssec:exp_nnlasso}, respectively, on both synthetic and real data sets. To the best of knowledge, the TLFre and DPC are the first screening methods for SGL and nonnegative Lasso, respectively.

\subsection{TLFre for SGL}\label{ssec:exp_SGL}

We perform experiments to evaluate TLFre on synthetic and real data sets in Sections \ref{sssec:exp_syn_SGL} and \ref{sssec:exp_ADNI_SGL}, respectively. To measure the performance of TLFre, we compute the \emph{rejection ratios} of (\ref{rule:L1}) and (\ref{rule:L2}), respectively. Specifically, let $m$ be the number of features that have $0$ coefficients in the solution, $\overline{\mathcal{G}}$ be the index set of groups that are discarded by (\ref{rule:L1}) and $\overline{p}$ be the number of inactive features that are detected by (\ref{rule:L2}). The rejection ratios of (\ref{rule:L1}) and (\ref{rule:L2}) are defined by 
$r_1=\frac{\sum_{g\in\overline{\mathcal{G}}}n_g}{m}$ and $r_2=\frac{|\overline{p}|}{m}$, respectively. Moreover, we report the \emph{speedup} gained by TLFre, i.e., the ratio of the running time of solver without screening to the running time of solver with TLFre. The solver used in this paper is from SLEP \cite{SLEP}. 

To determine appropriate values of $\alpha$ and $\lambda$ by cross validation or stability selection, we can run TLFre with as many parameter values as we need. Given a data set, for illustrative purposes only,  we select seven values of $\alpha$ from $\{\tan(\psi):\psi=5^{\circ},15^{\circ},30^{\circ},45^{\circ},60^{\circ},75^{\circ},85^{\circ}\}$. Then, for each value of $\alpha$, we run TLFre along a sequence of $100$ values of $\lambda$ equally spaced on the logarithmic scale of $\lambda/\lambda_{\rm max}^{\alpha}$ from $1$ to $0.01$. Thus, $700$ pairs of parameter values of $(\lambda,\alpha)$ are sampled in total. 

\subsubsection{Simulation Studies}\label{sssec:exp_syn_SGL}

We perform experiments on two synthetic data sets that are commonly used in the literature \cite{Tibshirani11,Zou2005}. The true model is $\mathbf{y}=\mathbf{X}\beta^*+0.01\epsilon$, $\epsilon\sim N(0,1)$. We generate two data sets with $250\times 10000$ entries: Synthetic 1 and Synthetic 2. We randomly break the $10000$ features into $1000$ groups. For Synthetic 1,  the entries of the data matrix ${\bf X}$ are i.i.d. standard Gaussian with
pairwise correlation zero, i.e., ${\rm corr}({\bf x}_i,{\bf x}_i)=0$. For Synthetic 2, the entries of the data matrix ${\bf X}$ are drawn from i.i.d. standard Gaussian with pairwise correlation $0.5^{|i-j|}$, i.e., ${\rm corr}({\bf x}_i,{\bf x}_j)=0.5^{|i-j|}$. To construct $\beta^*$, we first randomly select $\gamma_1$ percent of groups. Then, for each selected group, we randomly select $\gamma_2$ percent of features. The selected components of $\beta^*$ are populated from a standard Gaussian and the remaining ones are set to $0$. We set $\gamma_1=\gamma_2=10$ for Synthetic 1 and $\gamma_1=\gamma_2=20$ for Synthetic 2.

\begin{figure*}[t]
	\centering{
		\subfigure[] { \label{fig:syn1_eff_region}
			\includegraphics[width=0.22\columnwidth]{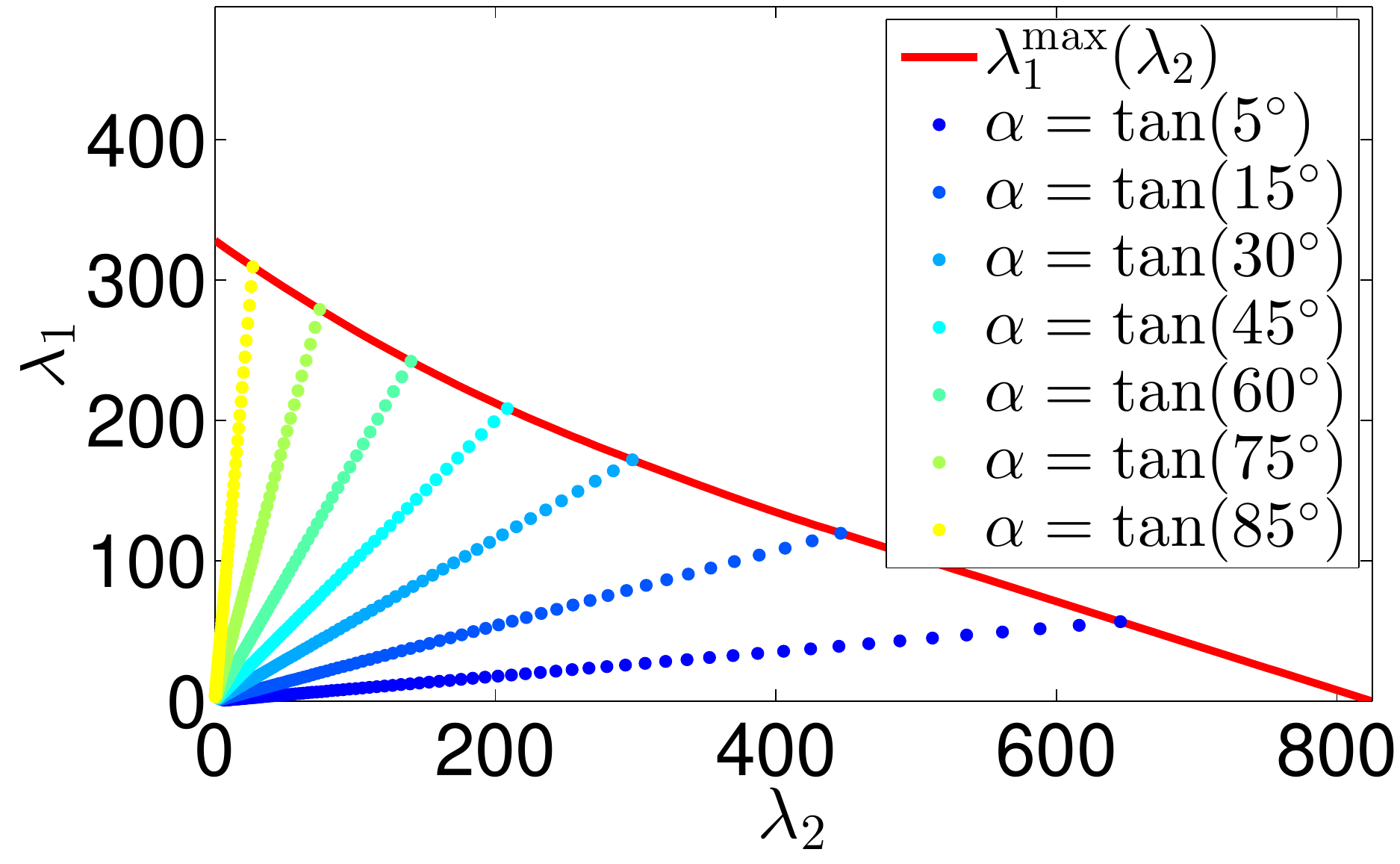}
		}
		\subfigure[$\alpha=\tan(5^{\circ})$] { \label{fig:syn1_5}
			\includegraphics[width=0.22\columnwidth]{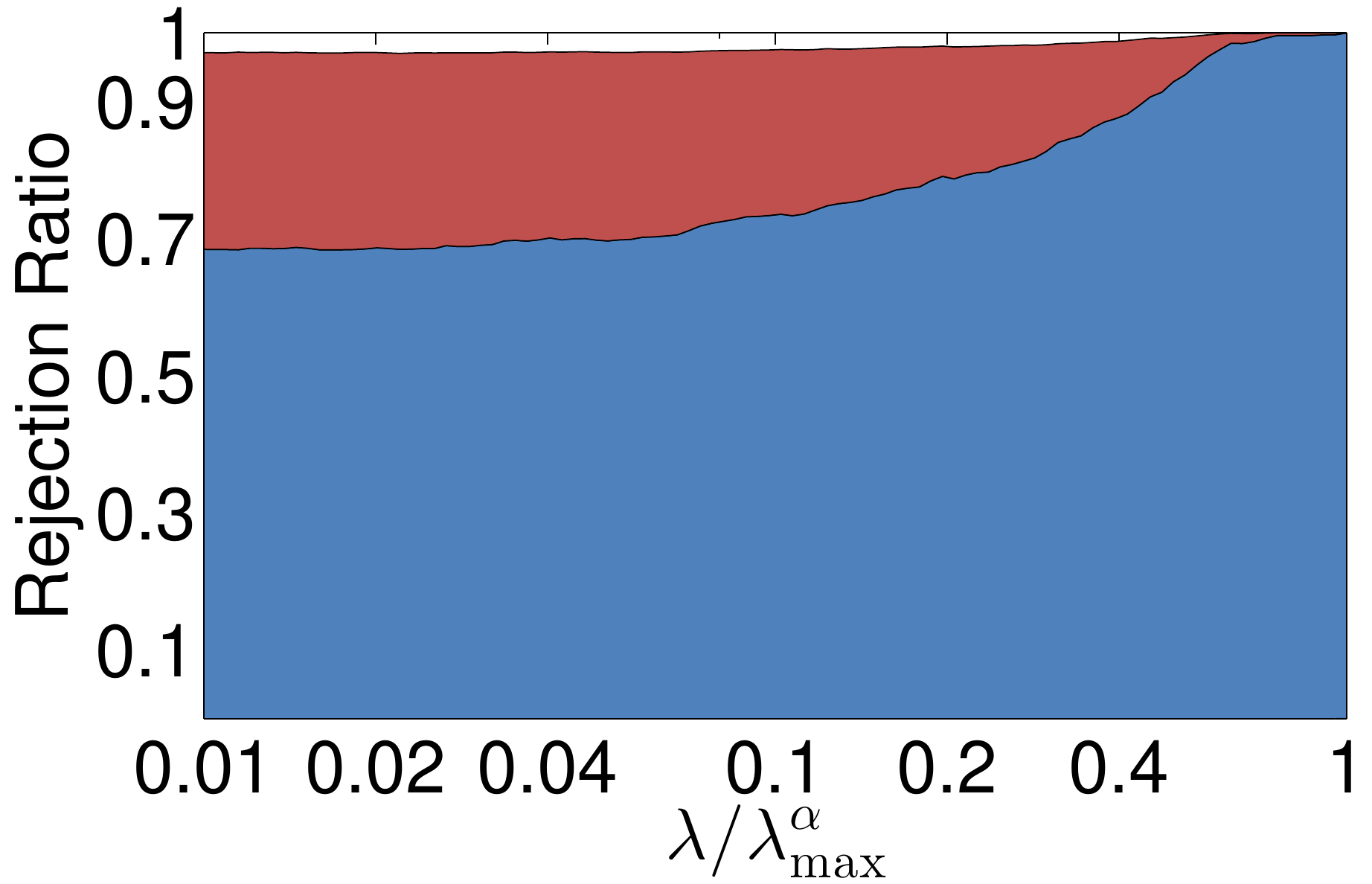}
		}
		\subfigure[$\alpha=\tan(15^{\circ})$] { \label{fig:syn1_15}
			\includegraphics[width=0.22\columnwidth]{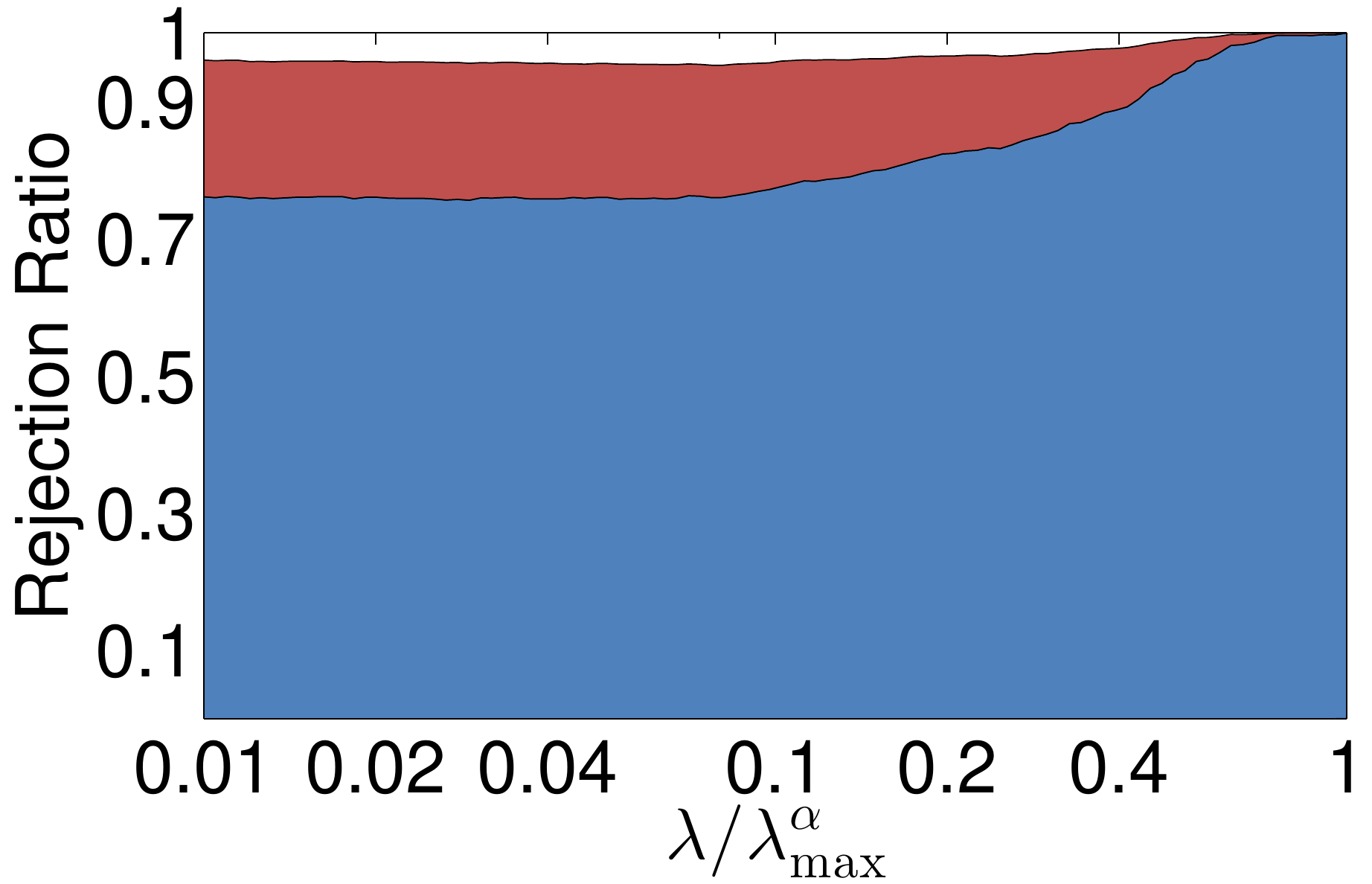}
		}
		\subfigure[$\alpha=\tan(30^{\circ})$] { \label{fig:syn1_30}
			\includegraphics[width=0.22\columnwidth]{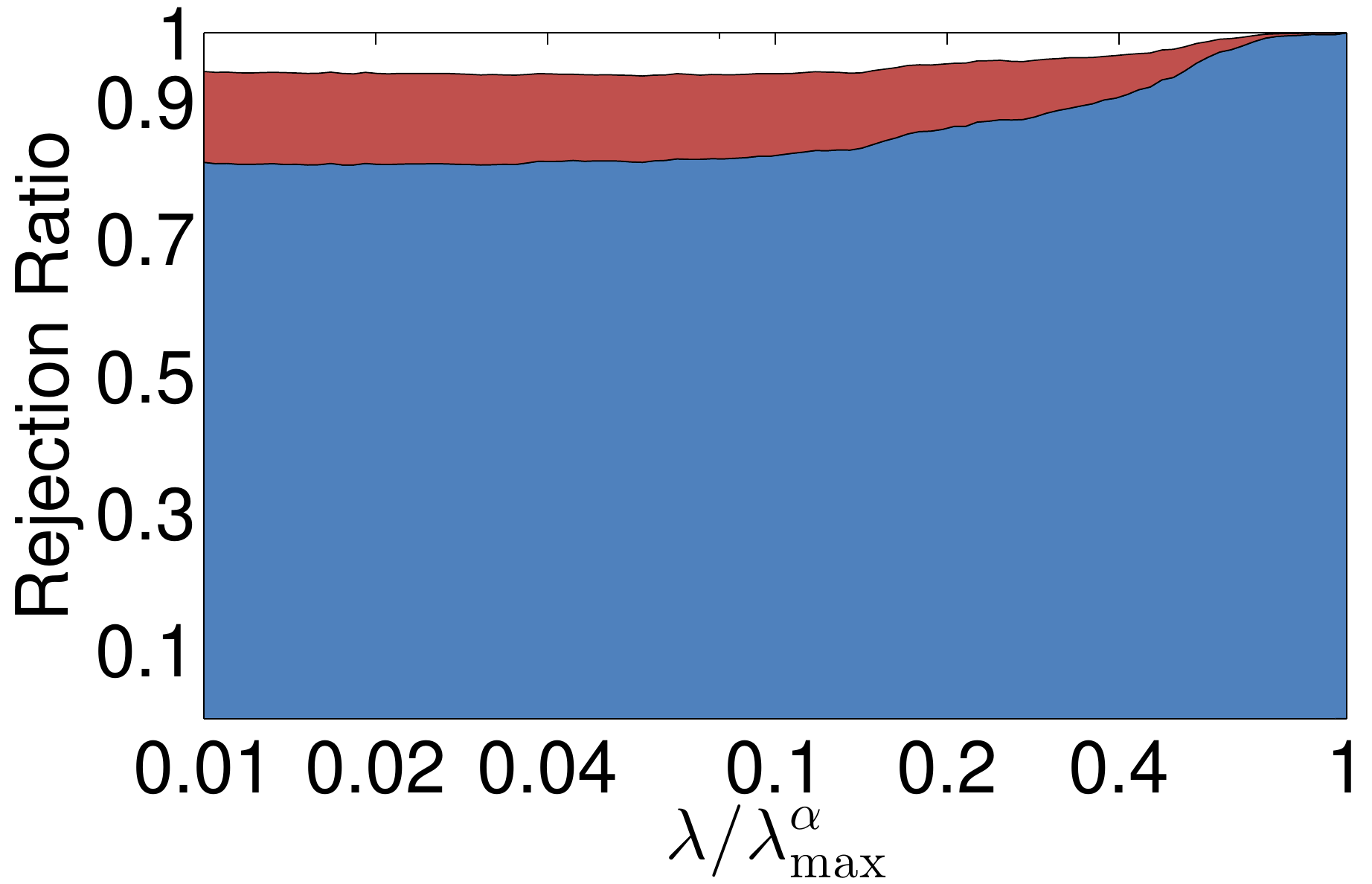}
		}\\[-0.35cm]
		\subfigure[$\alpha=\tan(45^{\circ})$] { \label{fig:syn1_45}
			\includegraphics[width=0.22\columnwidth]{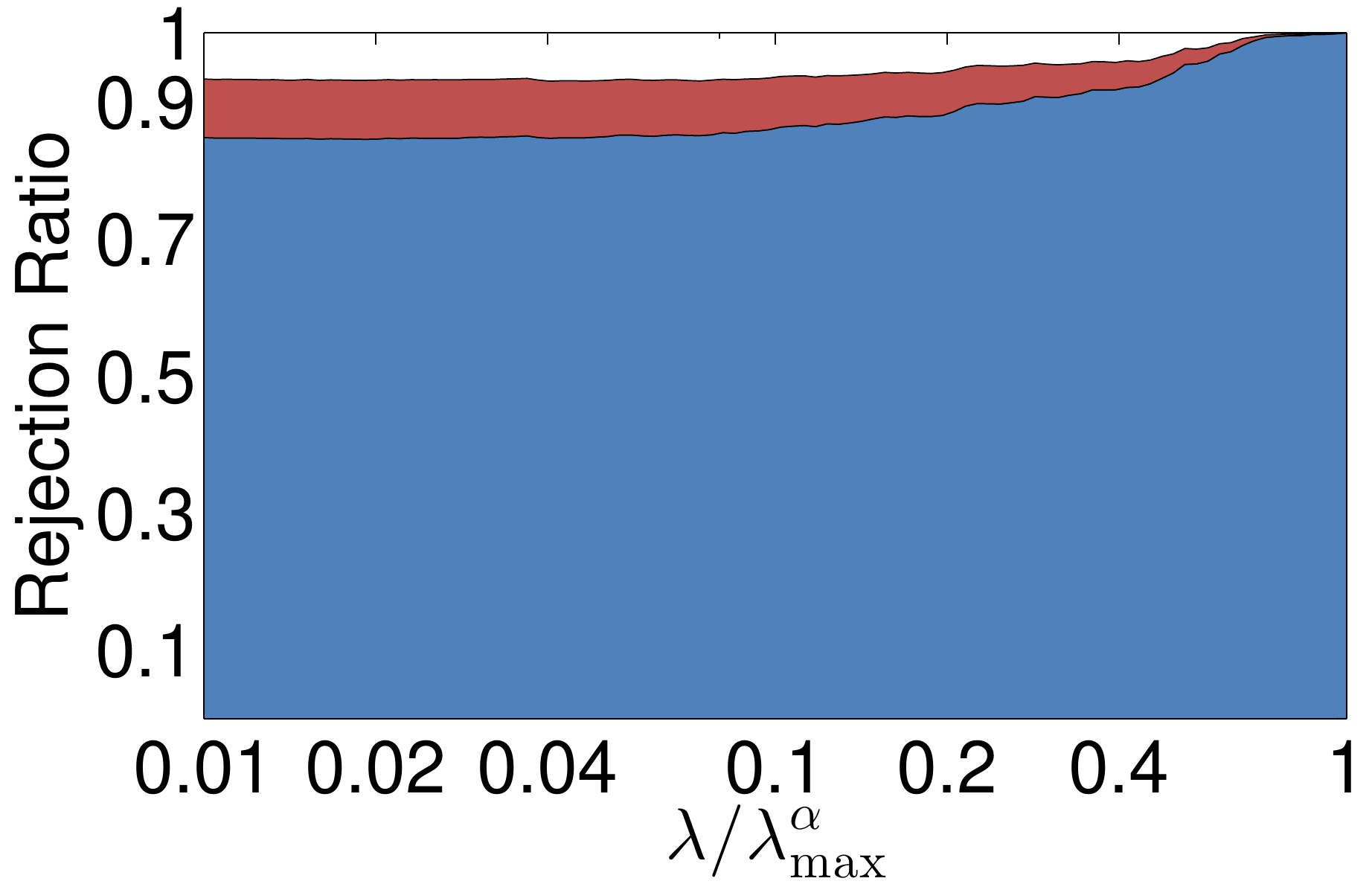}
		}
		\subfigure[$\alpha=\tan(60^{\circ})$] { \label{fig:syn1_60}
			\includegraphics[width=0.22\columnwidth]{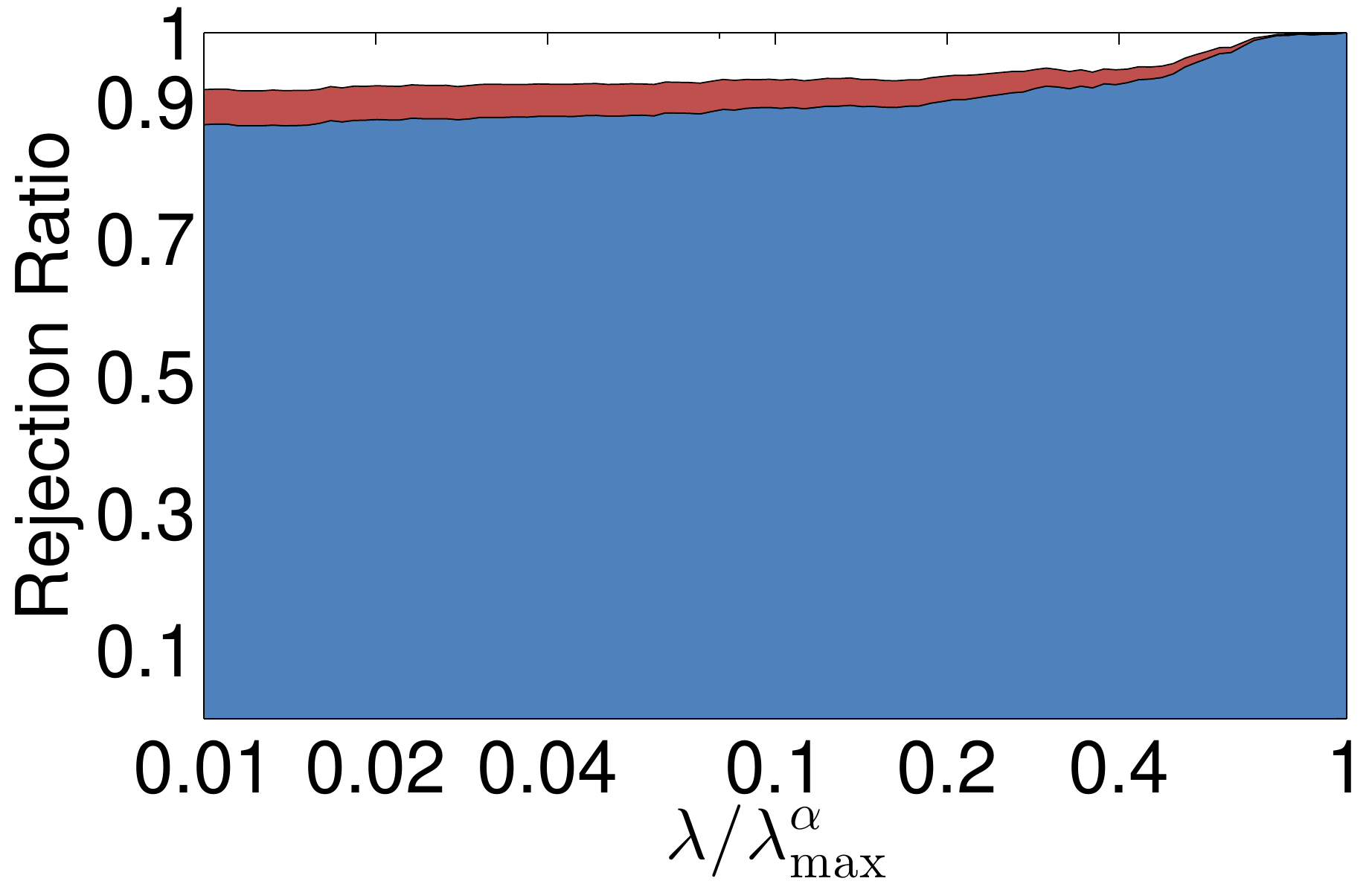}
		}
		\subfigure[$\alpha=\tan(75^{\circ})$] { \label{fig:syn1_75}
			\includegraphics[width=0.22\columnwidth]{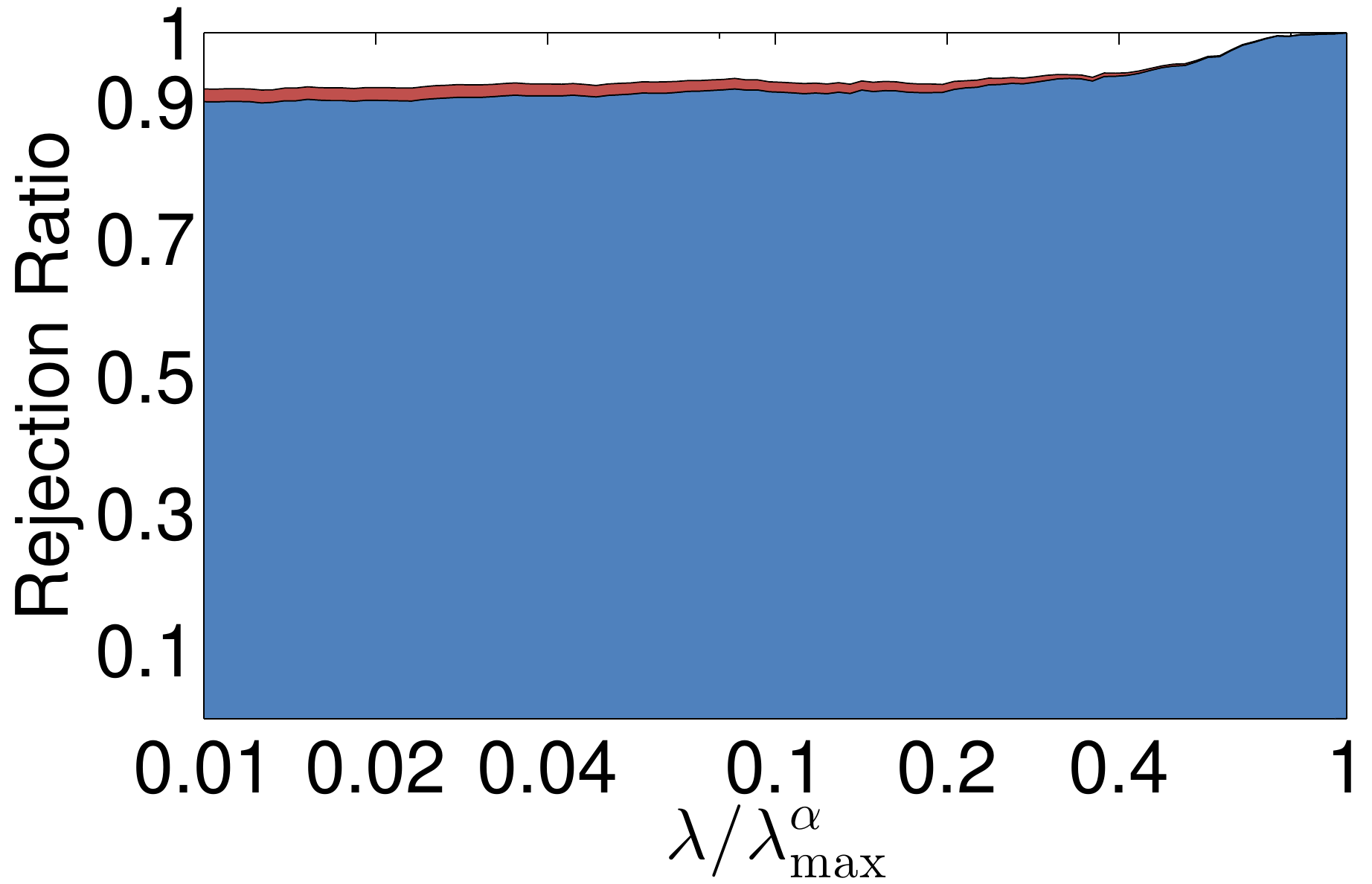}
		}
		\subfigure[$\alpha=\tan(85^{\circ})$] { \label{fig:syn1_85}
			\includegraphics[width=0.22\columnwidth]{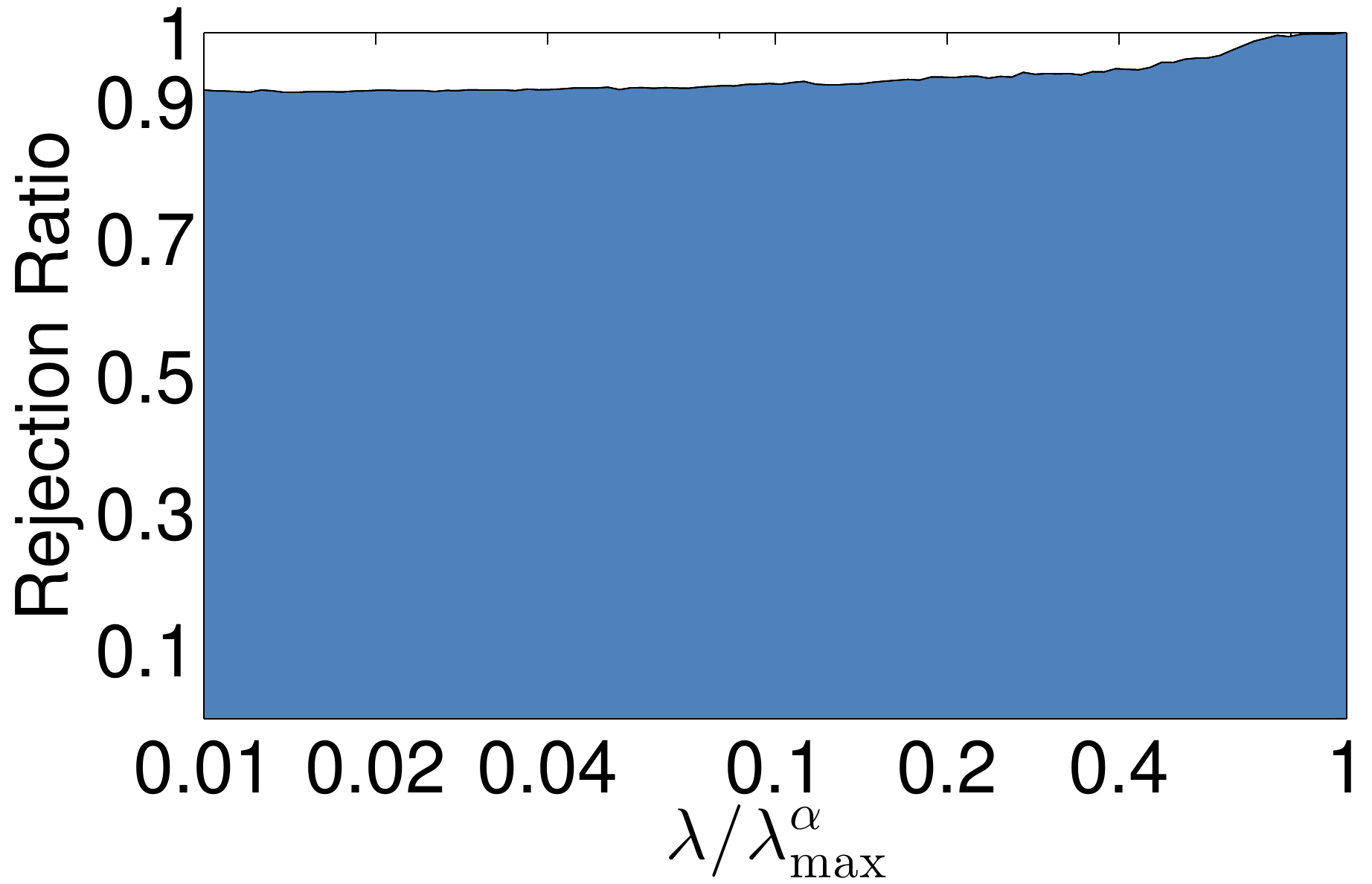}
		}
	}\\[-0.4cm]
	\caption{Rejection ratios of TLFre on the Synthetic 1 data set. }
	\vspace{-0.1in}
	\label{fig:synthetic1}
\end{figure*}

The figures in the upper left corner of \figref{fig:synthetic1} and \figref{fig:synthetic2} show the plots of $\lambda_1^{\rm max}(\lambda_2)$ (see Corollary \ref{cor:lambdamx_SGL}) and the sampled parameter values of $\lambda$ and $\alpha$ (recall that $\lambda_1=\alpha\lambda$ and $\lambda_2=\lambda$). For the other figures, the blue and red regions represent the rejection ratios of (\ref{rule:L1}) and (\ref{rule:L2}), respectively. We can see that TLFre is very effective in discarding inactive groups/features; that is, more than $90\%$ of inactive features  
can be detected. Moreover, we can observe that the first layer screening (\ref{rule:L1}) becomes more effective with a larger $\alpha$. Intuitively, this is because the group Lasso penalty plays a more important role in enforcing the sparsity with a larger value of $\alpha$ (recall that $\lambda_1=\alpha\lambda$). The top and middle parts of Table \ref{table:TLFre_runtime_sync} indicate that the speedup gained by TLFre is very significant (up to $30$ times) and TLFre is very efficient. Compared to the running time of the solver without screening, the running time of TLFre is negligible. The running time of TLFre includes that of computing $\|\mathbf{X}_g\|_2$, $g=1,\ldots,G$, which can be efficiently computed by the power method \cite{Halko2011}. Indeed, this can be shared for TLFre with different parameter values.

\begin{figure*}[t]
	\centering{
		\subfigure[] { \label{fig:syn2_eff_region}
			\includegraphics[width=0.22\columnwidth]{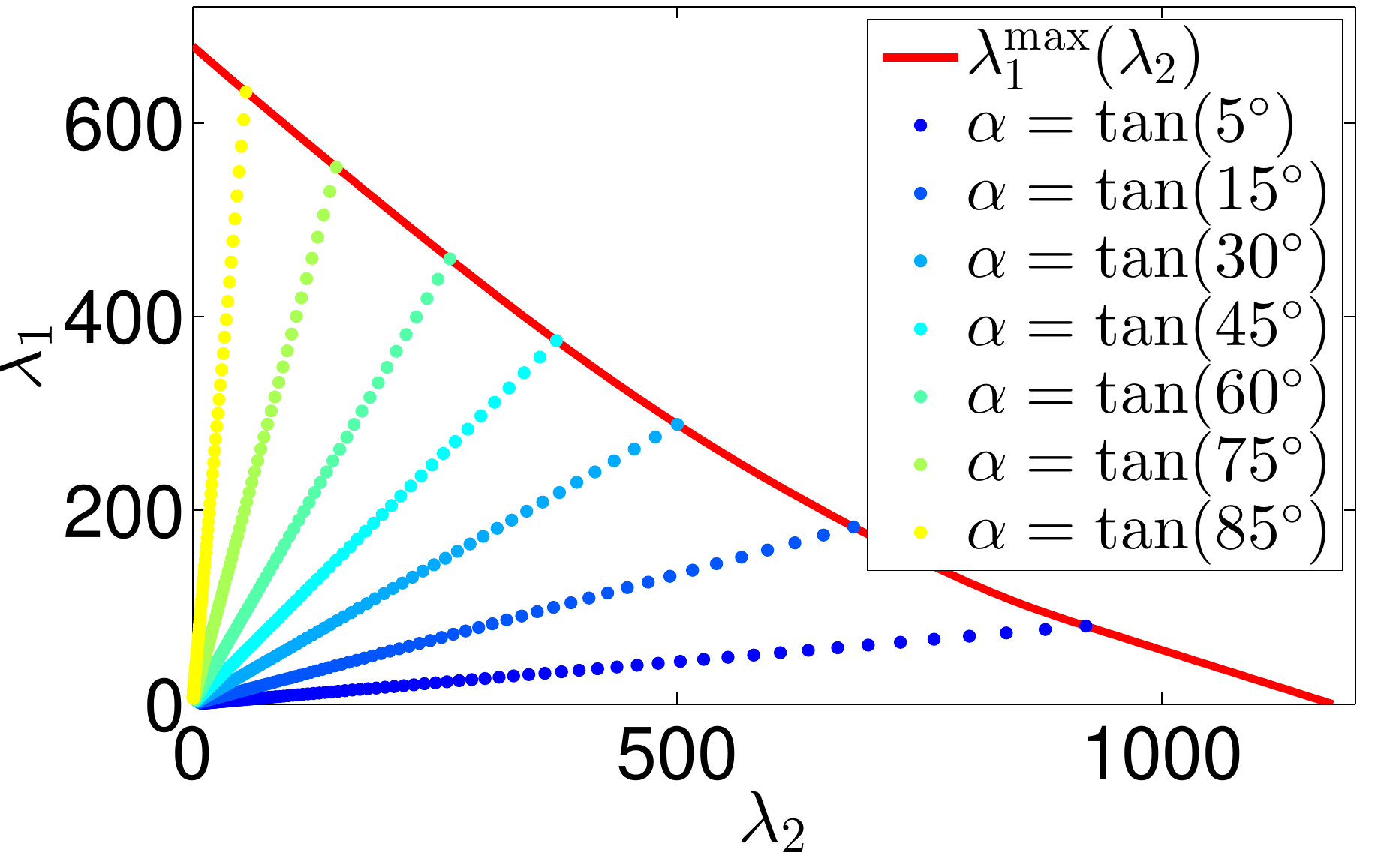}
		}
		\subfigure[$\alpha=\tan(5^{\circ})$] { \label{fig:syn2_5}
			\includegraphics[width=0.23\columnwidth]{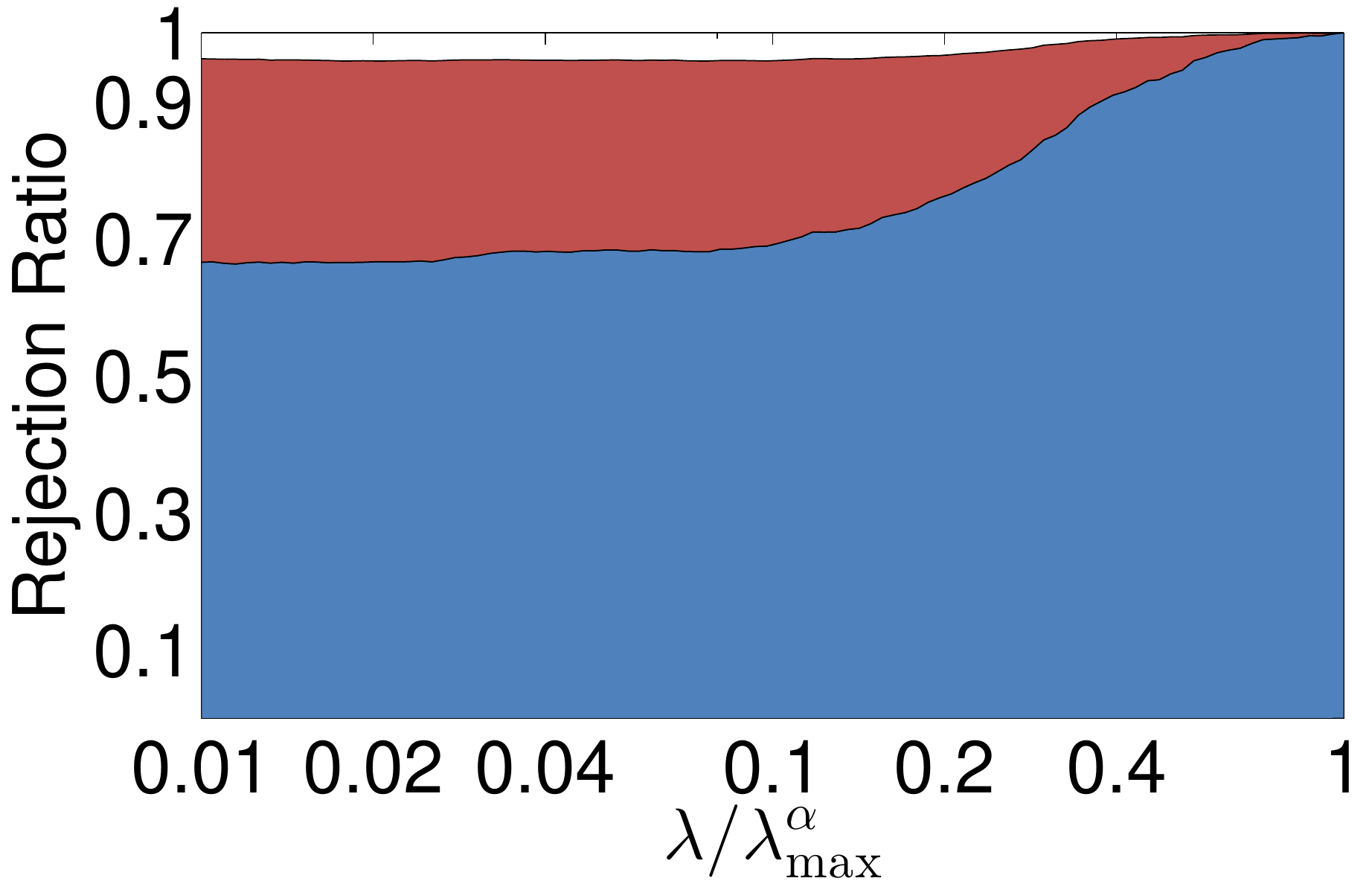}
		}
		\subfigure[$\alpha=\tan(15^{\circ})$] { \label{fig:syn2_15}
			\includegraphics[width=0.22\columnwidth]{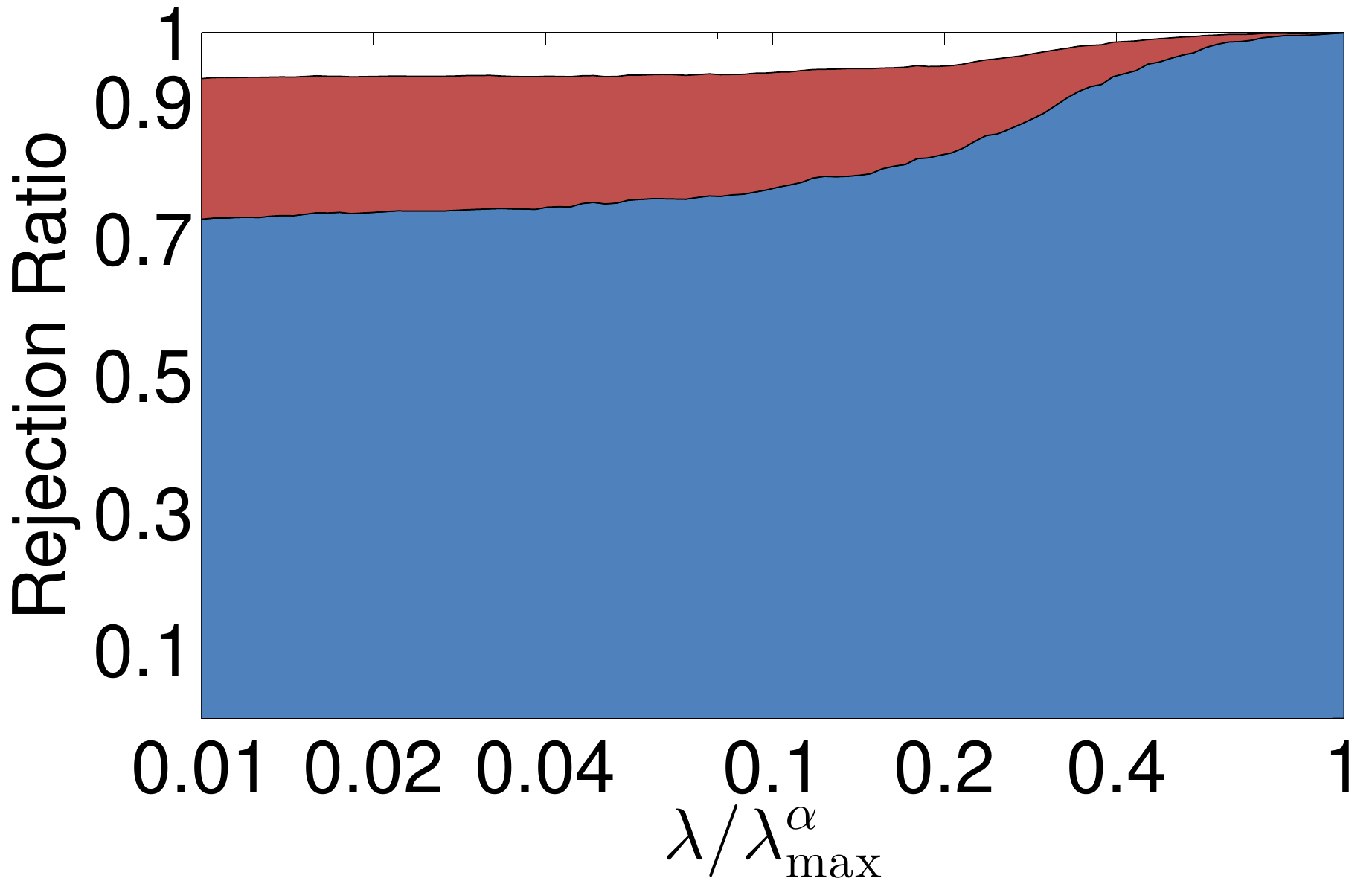}
		}
		\subfigure[$\alpha=\tan(30^{\circ})$] { \label{fig:syn2_30}
			\includegraphics[width=0.22\columnwidth]{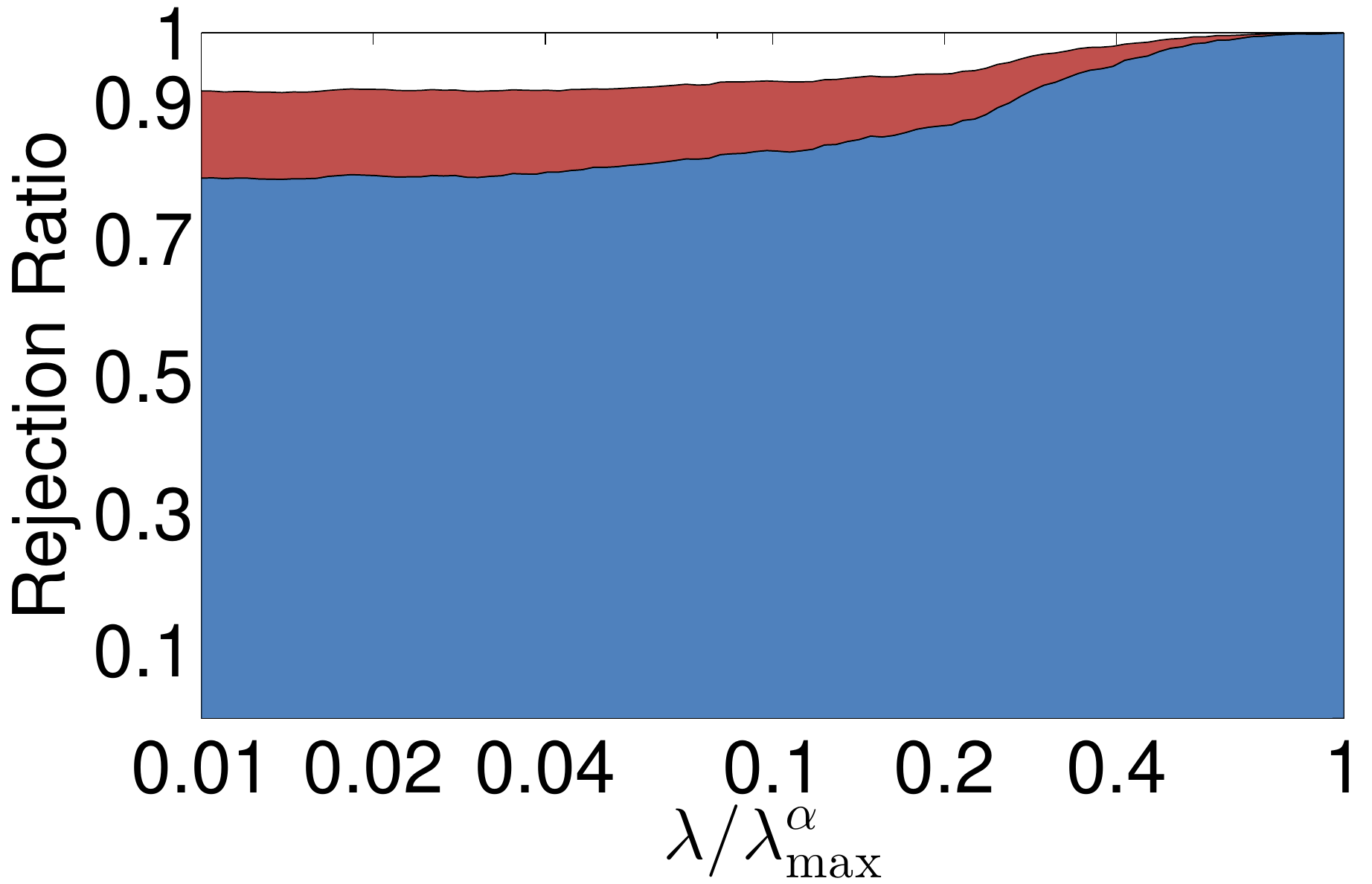}
		}\\[-.3cm]
		\subfigure[$\alpha=\tan(45^{\circ})$] { \label{fig:syn2_45}
			\includegraphics[width=0.22\columnwidth]{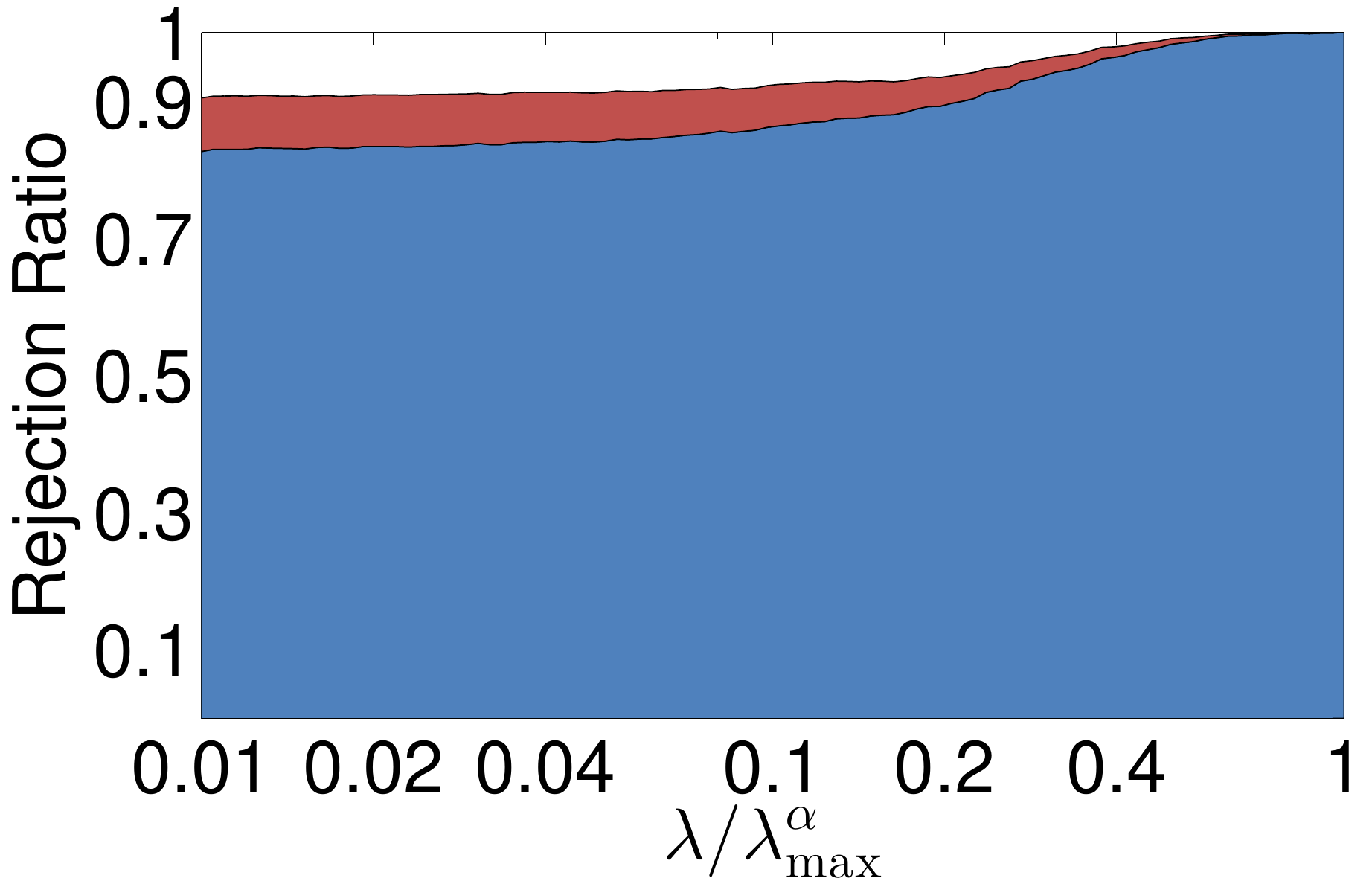}
		}
		\subfigure[$\alpha=\tan(60^{\circ})$] { \label{fig:syn2_60}
			\includegraphics[width=0.22\columnwidth]{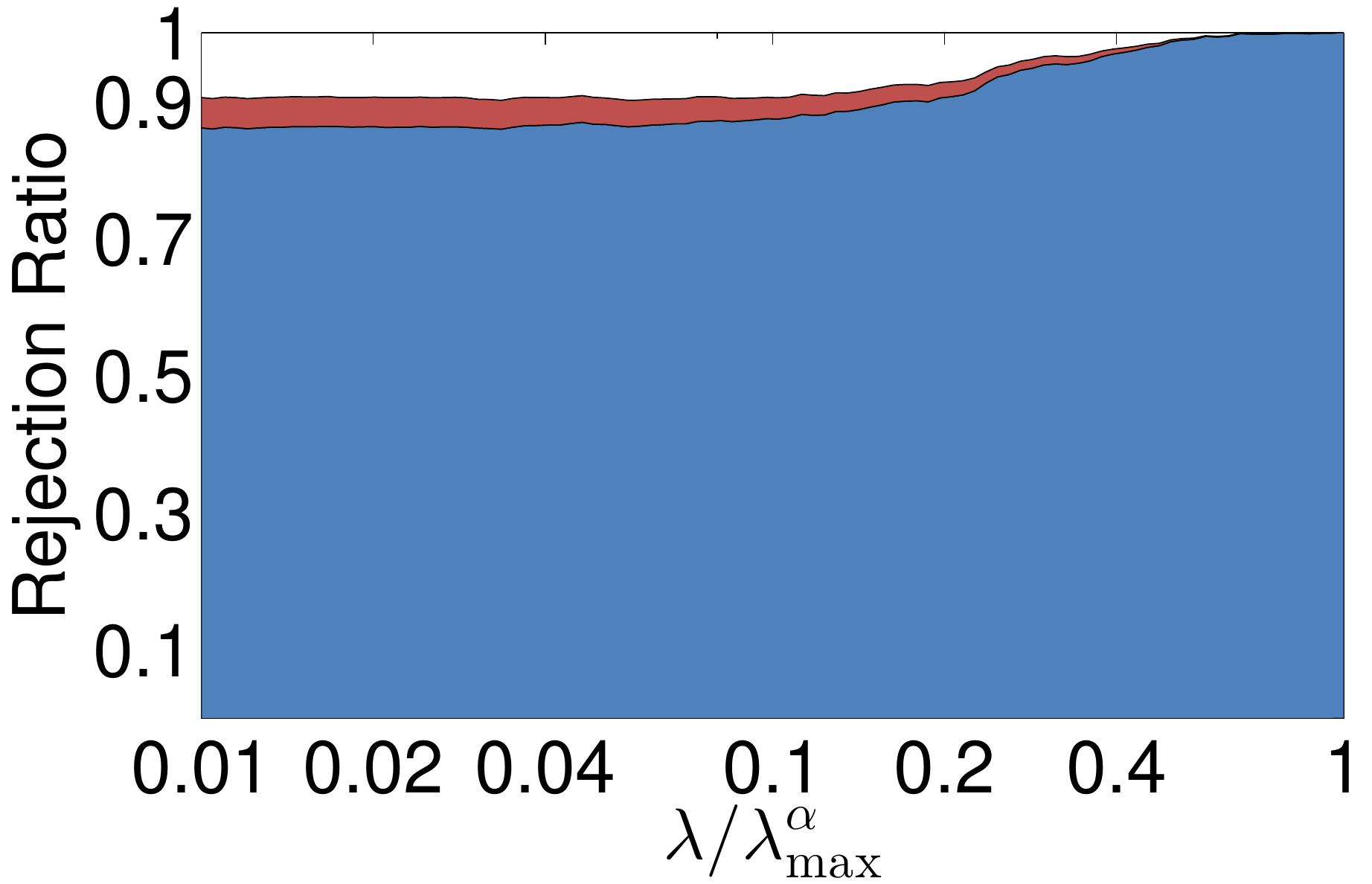}
		}
		\subfigure[$\alpha=\tan(75^{\circ})$] { \label{fig:syn2_75}
			\includegraphics[width=0.22\columnwidth]{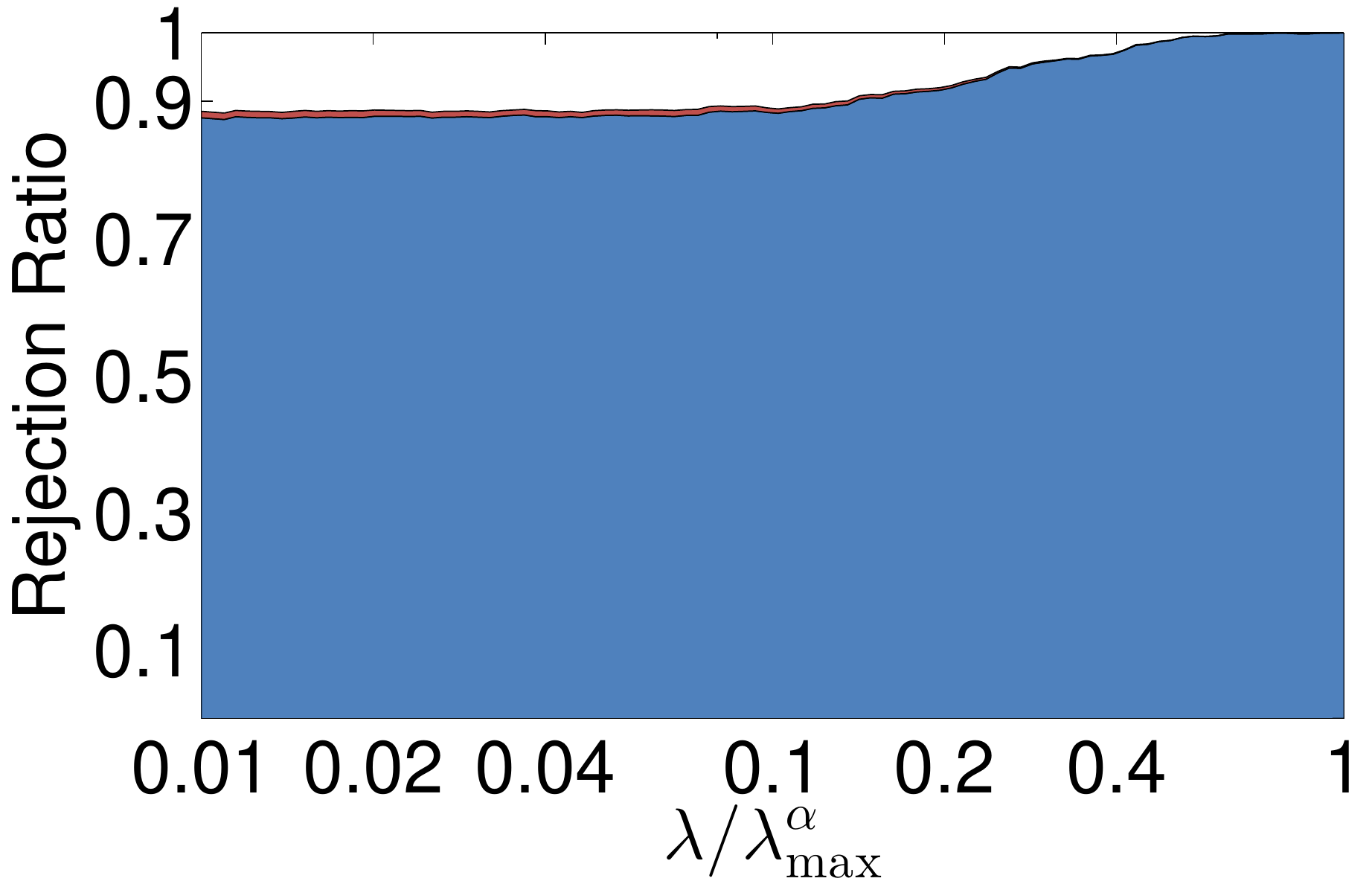}
		}
		\subfigure[$\alpha=\tan(85^{\circ})$] { \label{fig:syn2_85}
			\includegraphics[width=0.22\columnwidth]{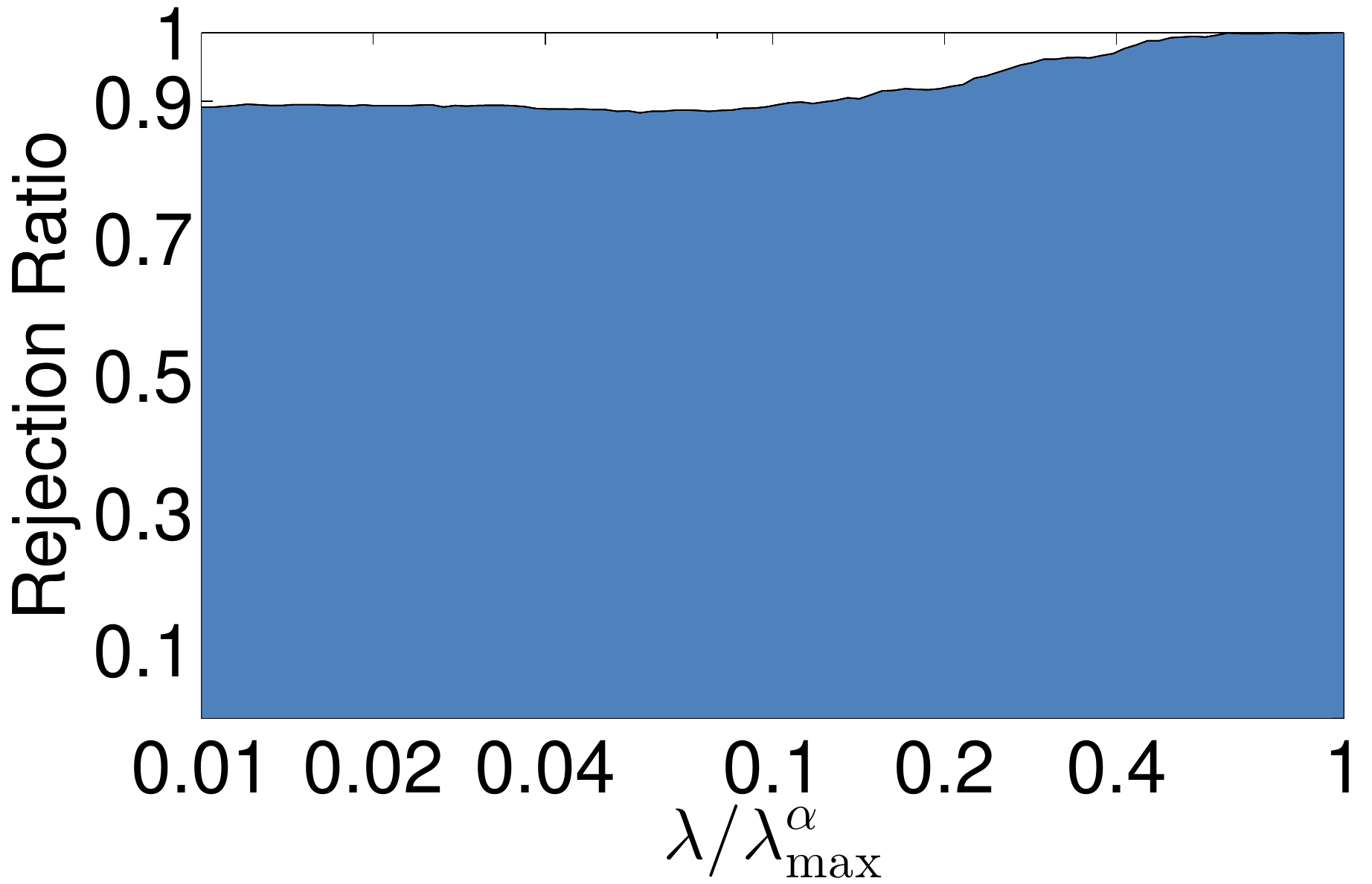}
		}
	}\\[-0.3cm]
	\caption{Rejection ratios of TLFre on the Synthetic 2 data set. }
	\vspace{-0.1in}
	\label{fig:synthetic2}
\end{figure*}

\setlength{\tabcolsep}{.18em}
\begin{table}
	\begin{center}
		\caption{Running time (in seconds) for solving SGL along a sequence of $100$ tuning parameter values of $\lambda$ equally spaced on the logarithmic scale of ${\lambda}/{\lambda_{\rm max}^{\alpha}}$ from $1.0$ to $0.01$ by (a): the solver \cite{SLEP} without screening; (b): the solver combined with TLFre. The data sets are Synthetic 1 and Synthetic 2.
		}\label{table:TLFre_runtime_sync}\vspace{1mm}
		\begin{scriptsize}
			\def\arraystretch{1.25}
			\begin{tabular}{ l c|c|c|c|c|c|c|c| }
				\cline{2-9}
				& \multicolumn{1}{|c|}{$\alpha$} &  $\tan(5^{\circ})$ & $\tan(15^{\circ})$ &  $\tan(30^{\circ})$ & $\tan(45^{\circ})$ & $\tan(60^{\circ})$ & $\tan(75^{\circ})$ & $\tan(85^{\circ})$ \\
				\cline{2-9}\\ [-2.5ex]\hline
				\multicolumn{1}{|r|}{\multirow{4}{*}{Synthetic 1}}  & \multicolumn{1}{|c|}{solver} & 298.36 & 301.74 & 308.69 & 307.71 & 311.33 & 307.53 & 291.24 \\ \cline{2-9}
				\multicolumn{1}{|r|}{}  & \multicolumn{1}{|c|}{TLFre} & 0.77 & 0.78 & 0.79 & 0.79 & 0.81 & 0.79 & 0.77 \\\cline{2-9}
				\multicolumn{1}{|r|}{} & \multicolumn{1}{|c|}{TLFre+solver} & 10.26 & 12.47 & 15.73 & 17.69 & 19.71 & 21.95 & 22.53 \\\cline{2-9}
				\multicolumn{1}{|r|}{}  & \multicolumn{1}{|c|}{\textbf{speedup}} & \textbf{29.09} & \textbf{24.19} & \textbf{19.63} & \textbf{17.40} & \text{15.79} & \textbf{14.01} & \textbf{12.93}\\\hline\hline
				\multicolumn{1}{|r|}{\multirow{4}{*}{Synthetic 2}}  & \multicolumn{1}{|c|}{solver} & 294.64 & 294.92 & 297.29 & 297.50 & 297.59 & 295.51 & 292.24\\\cline{2-9}
				\multicolumn{1}{|r|}{}  & \multicolumn{1}{|c|}{TLFre} & 0.79 & 0.80 & 0.80 & 0.81 & 0.81 & 0.81 & 0.82 \\\cline{2-9}
				\multicolumn{1}{|r|}{}  & \multicolumn{1}{|c|}{TLFre+solver} & 11.05 & 12.89 & 16.08 & 18.90 & 20.45 & 21.58 & 22.80 \\\cline{2-9}
				\multicolumn{1}{|r|}{} & \multicolumn{1}{|c|}{\textbf{speedup}} & \textbf{26.66} & \textbf{22.88} & \textbf{18.49} & \textbf{15.74} & \textbf{14.55} & \textbf{13.69} & \textbf{12.82} \\\hline 
			\end{tabular}
		\end{scriptsize}
	\end{center}
	\vspace{-0.1in}
	
\end{table}

\subsubsection{Experiments on Real Data Set}\label{sssec:exp_ADNI_SGL}

\begin{figure*}[t]
	\centering{
		\subfigure[] { \label{fig:ADNI_GMV_eff_region}
			\includegraphics[width=0.22\columnwidth]{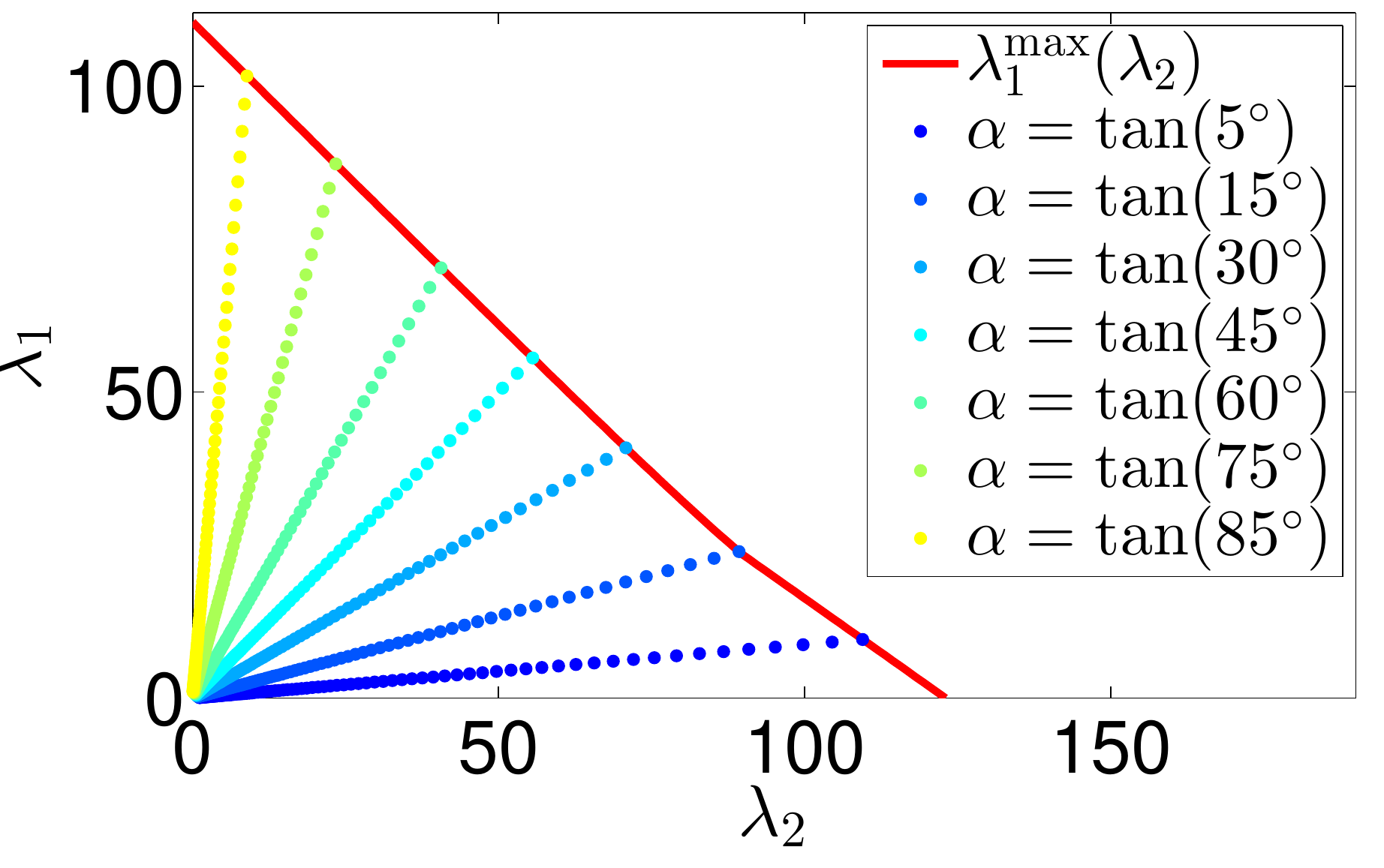}
		}
		\subfigure[$\alpha=\tan(5^{\circ})$] { \label{fig:ADNI_GMV_5}
			\includegraphics[width=0.22\columnwidth]{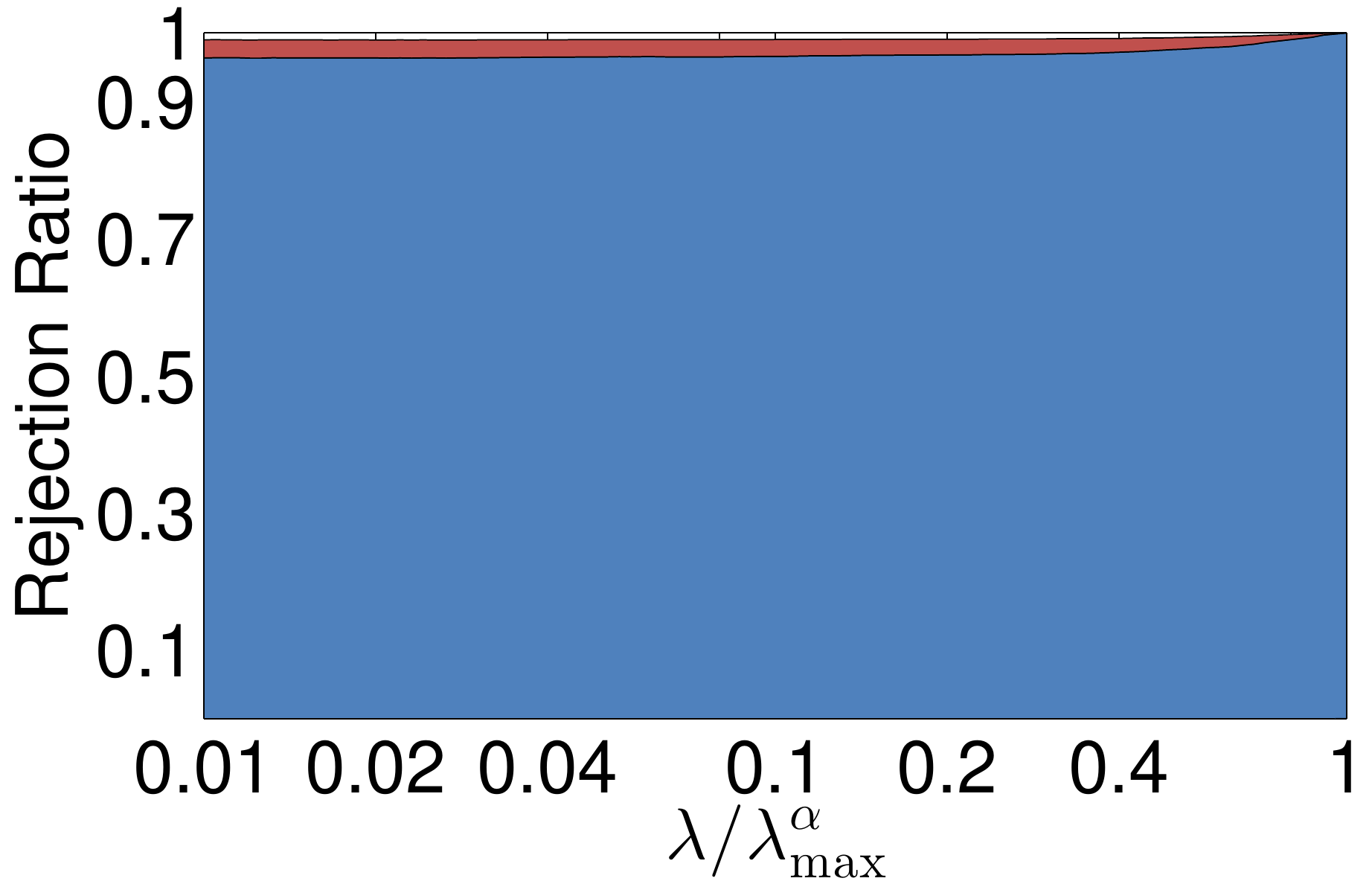}
		}
		\subfigure[$\alpha=\tan(15^{\circ})$] { \label{fig:ADNI_GMV_15}
			\includegraphics[width=0.22\columnwidth]{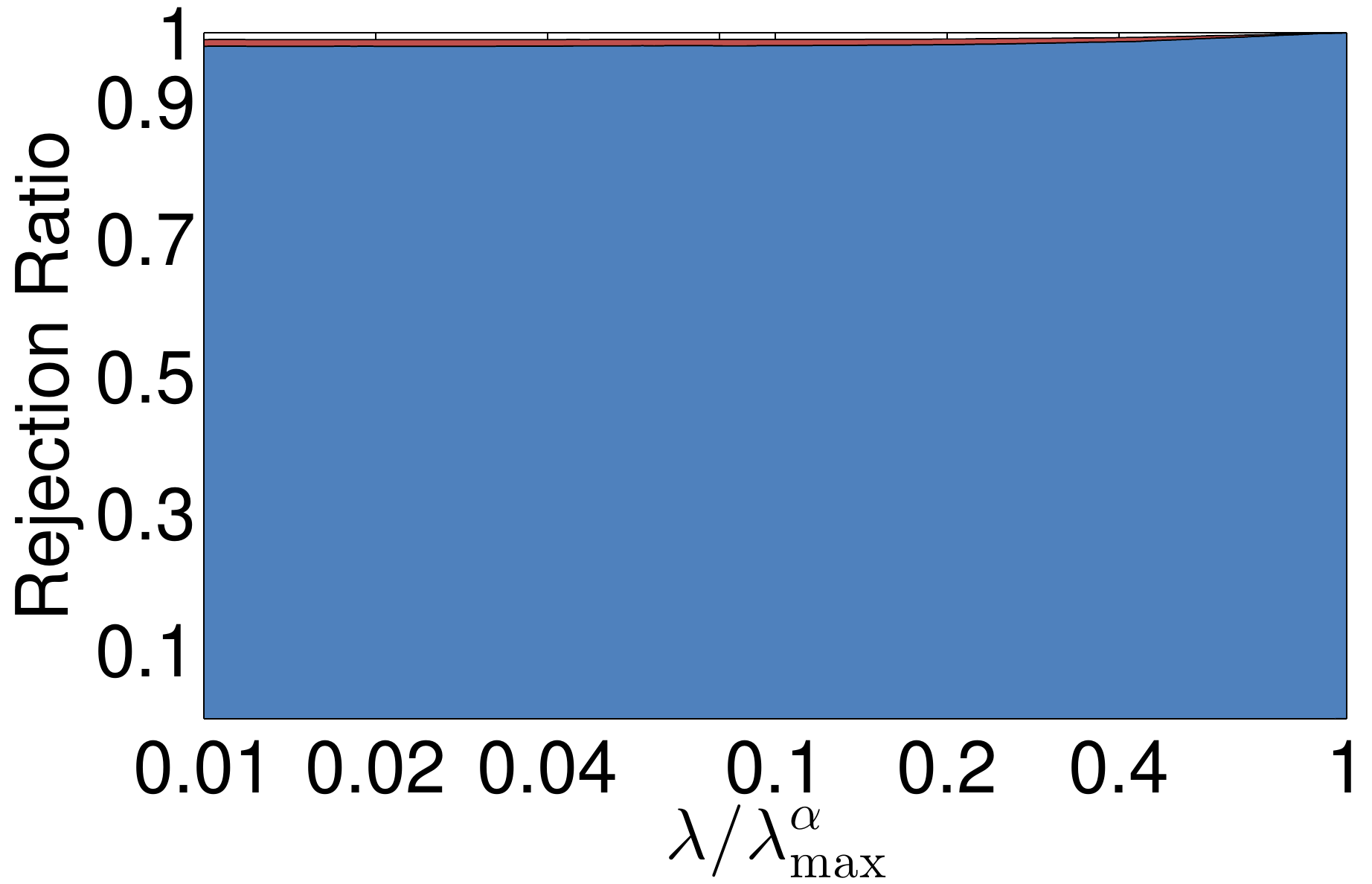}
		}
		\subfigure[$\alpha=\tan(30^{\circ})$] { \label{fig:ADNI_GMV_30}
			\includegraphics[width=0.22\columnwidth]{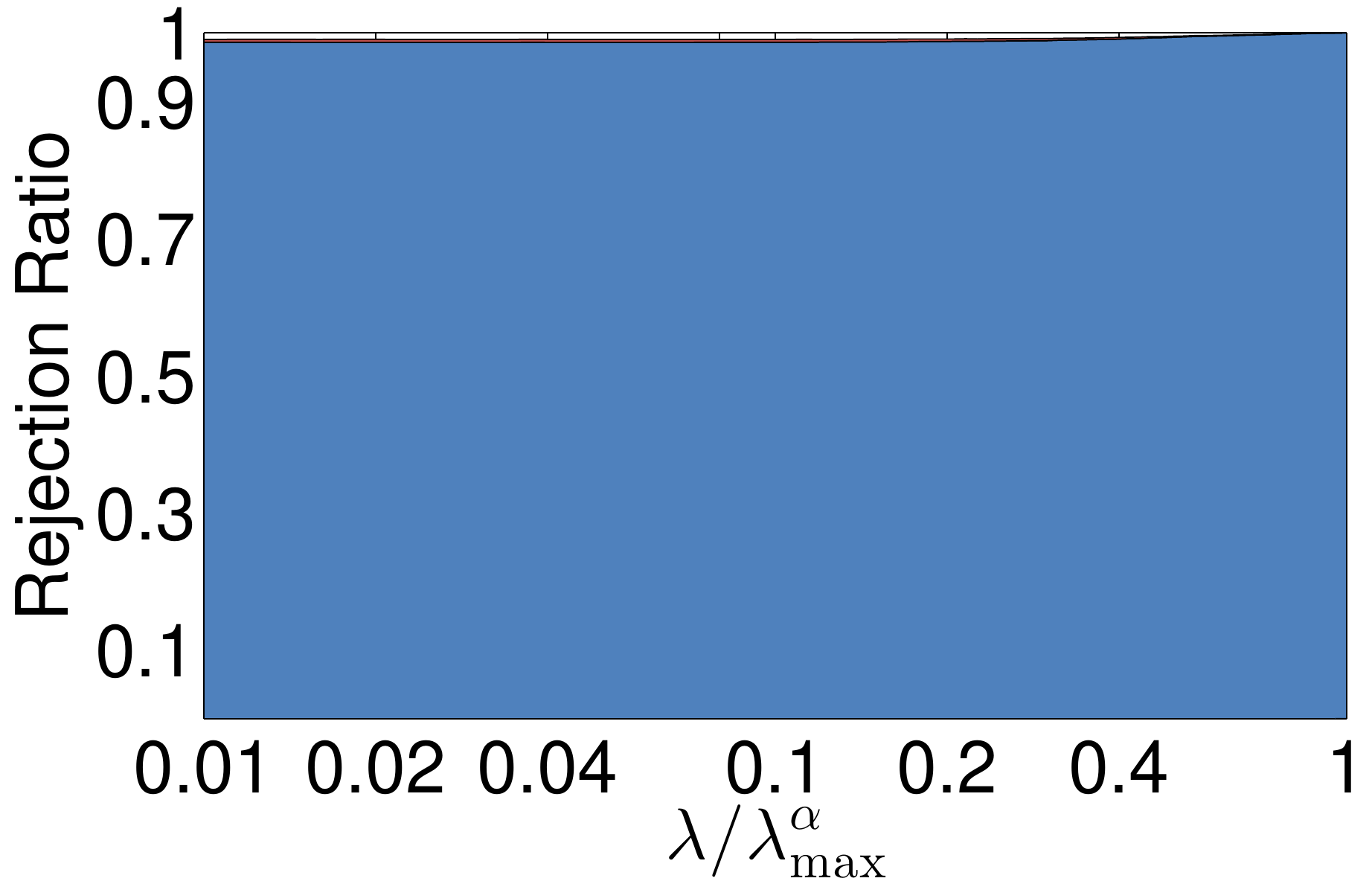}
		}\\[-0.3cm]
		\subfigure[$\alpha=\tan(45^{\circ})$] { \label{fig:ADNI_GMV_45}
			\includegraphics[width=0.22\columnwidth]{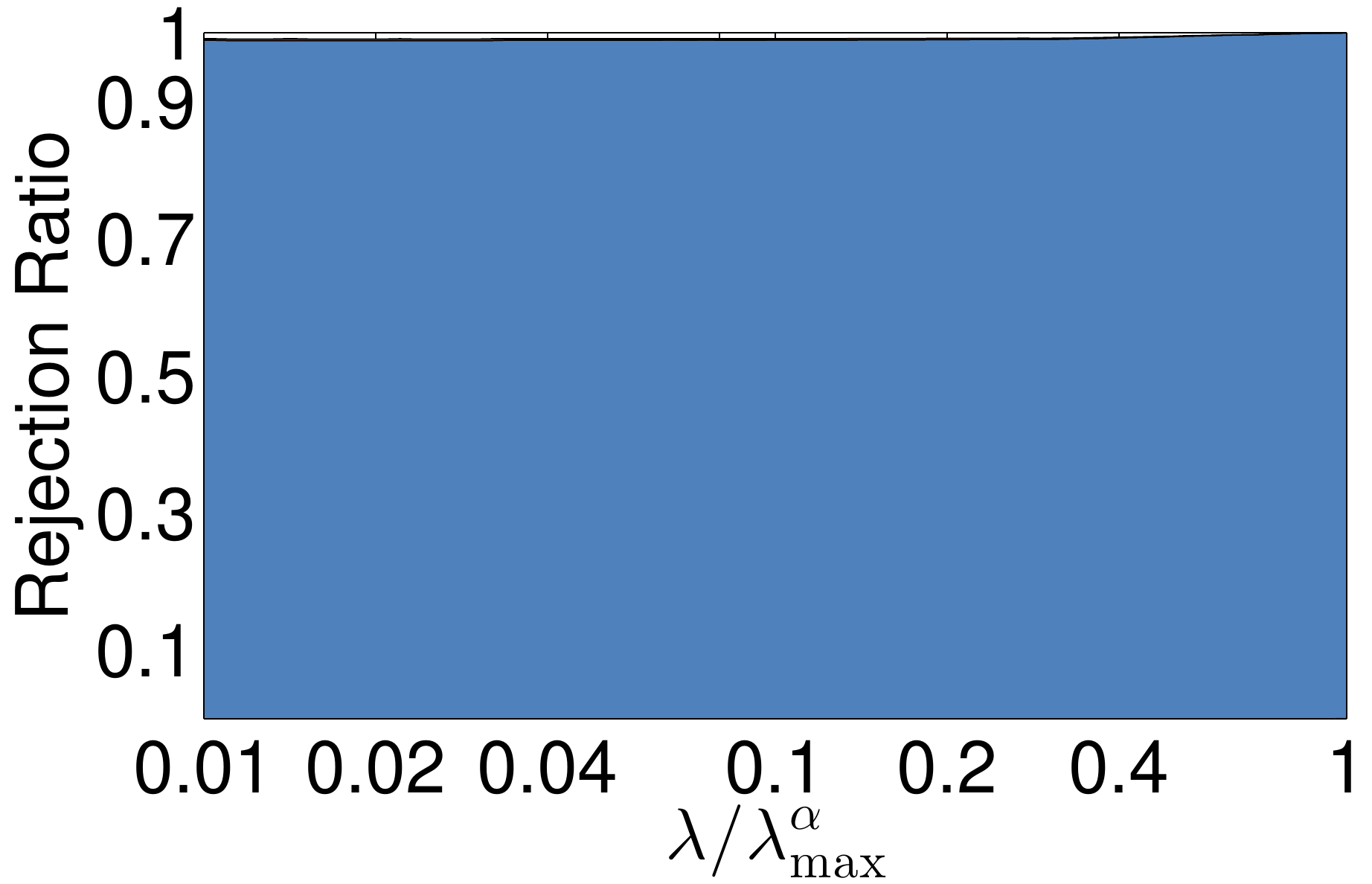}
		}
		\subfigure[$\alpha=\tan(60^{\circ})$] { \label{fig:ADNI_GMV_60}
			\includegraphics[width=0.22\columnwidth]{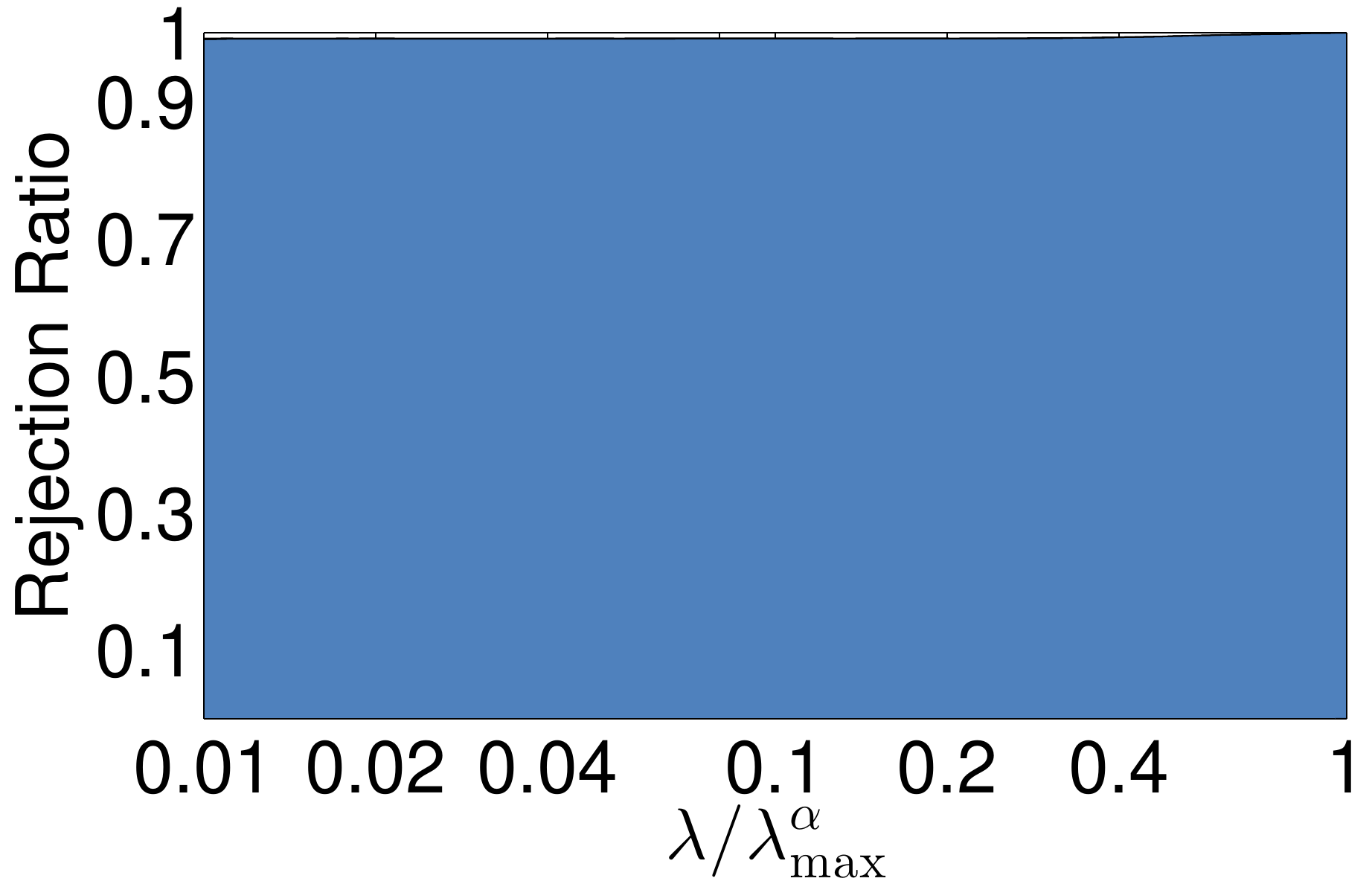}
		}
		\subfigure[$\alpha=\tan(75^{\circ})$] { \label{fig:ADNI_GMV_75}
			\includegraphics[width=0.22\columnwidth]{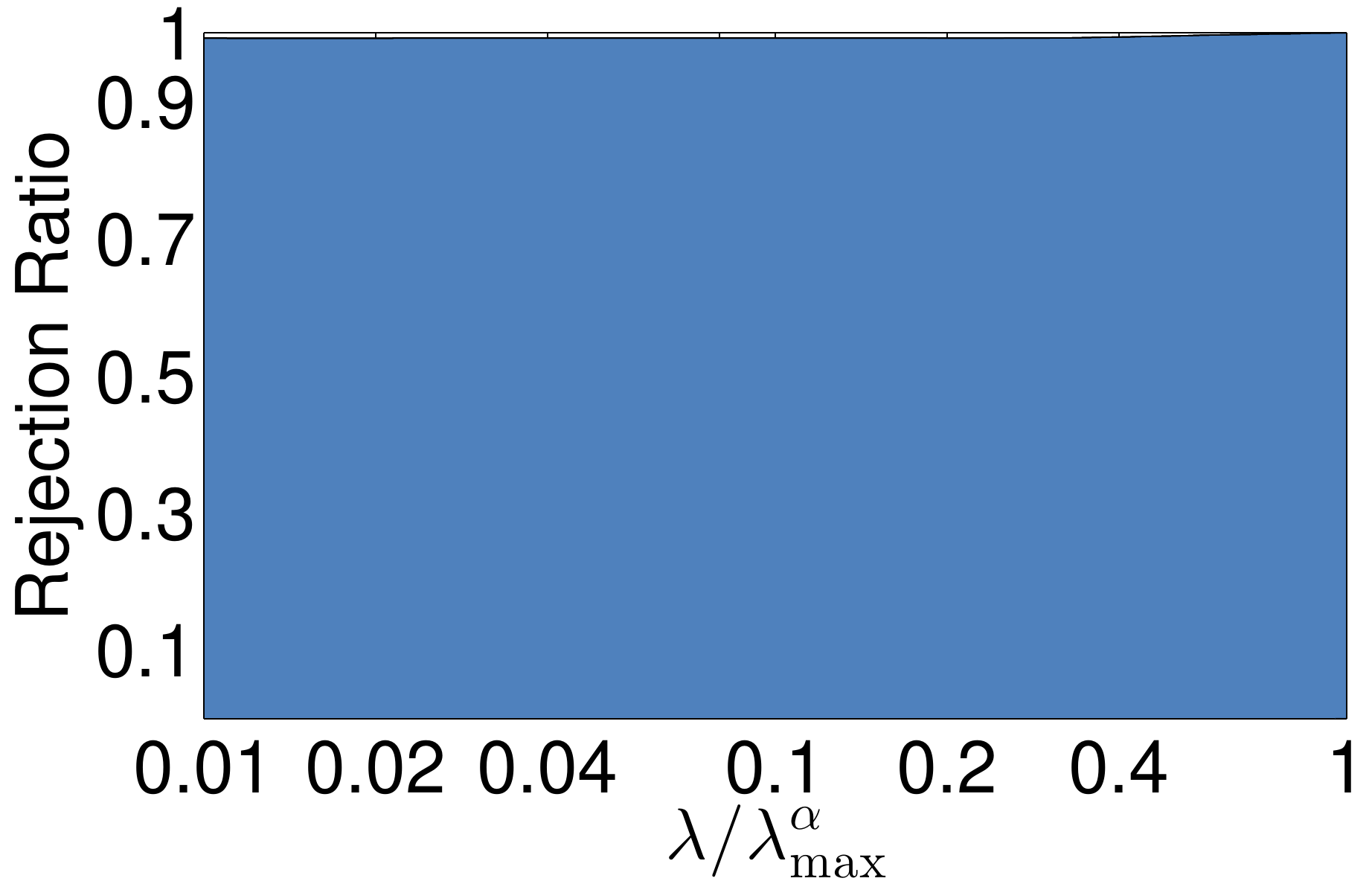}
		}
		\subfigure[$\alpha=\tan(85^{\circ})$] { \label{fig:ADNI_GMV_85}
			\includegraphics[width=0.22\columnwidth]{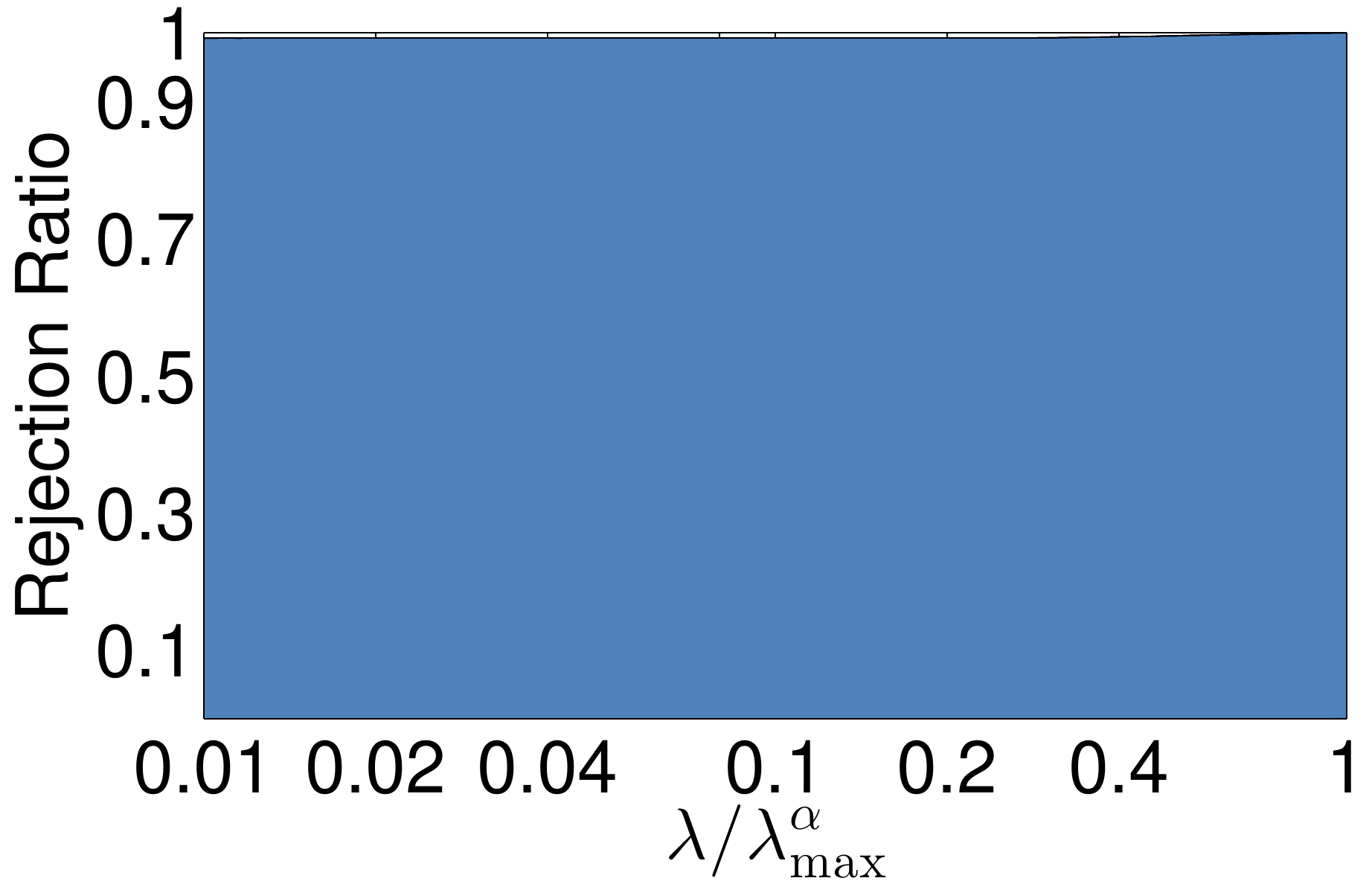}
		}
	}\\[-0.3cm]
	\caption{Rejection ratios of TLFre on the ADNI data set with grey matter volume as response. }
	\vspace{-0.15in}
	\label{fig:ADNI_GMV}
\end{figure*}

\begin{figure*}[th]
	\centering{
		\subfigure[] { \label{fig:ADNI_WMV_eff_region}
			\includegraphics[width=0.22\columnwidth]{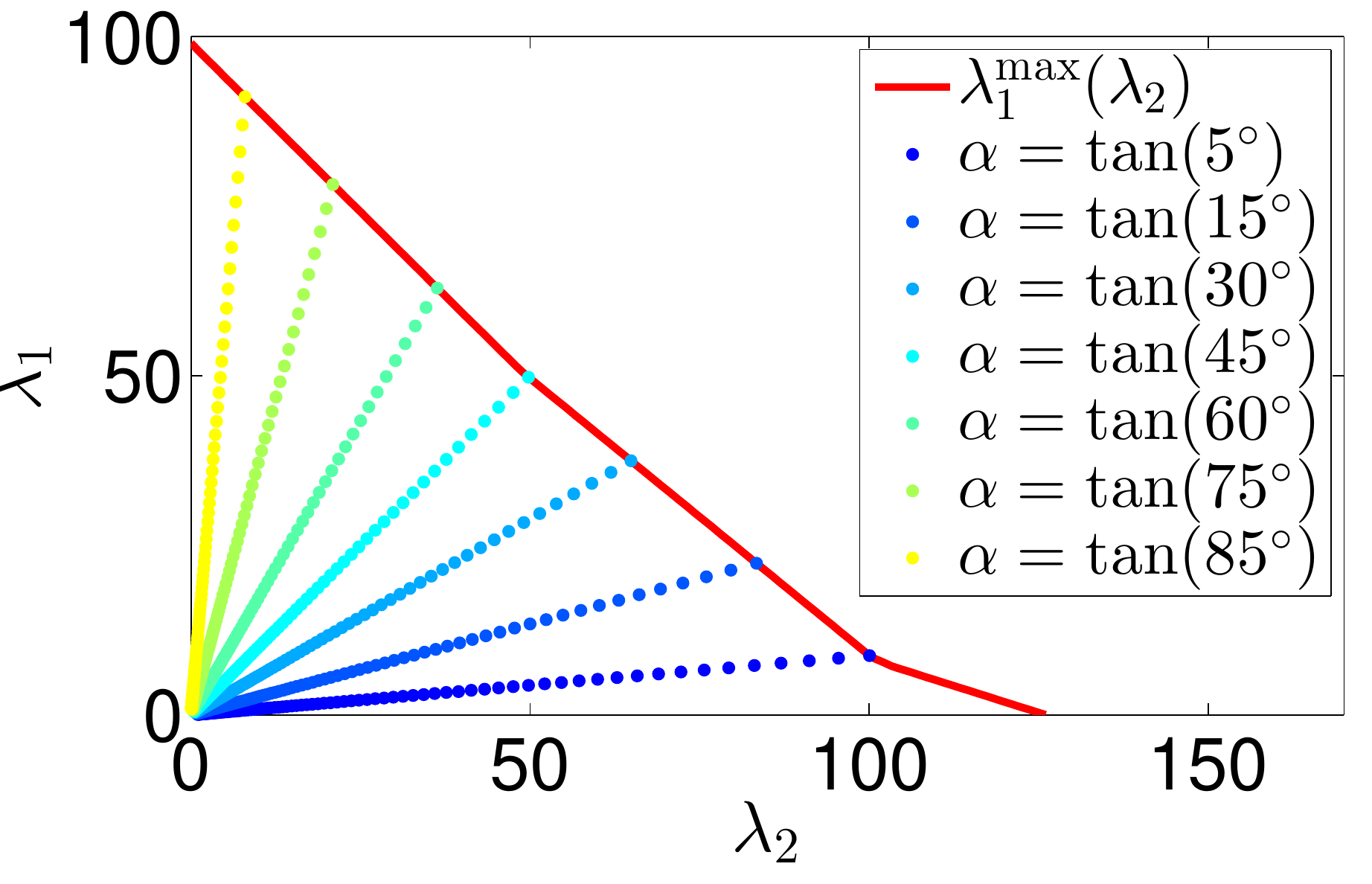}
		}
		\subfigure[$\alpha=\tan(5^{\circ})$] { \label{fig:ADNI_WMV_5}
			\includegraphics[width=0.22\columnwidth]{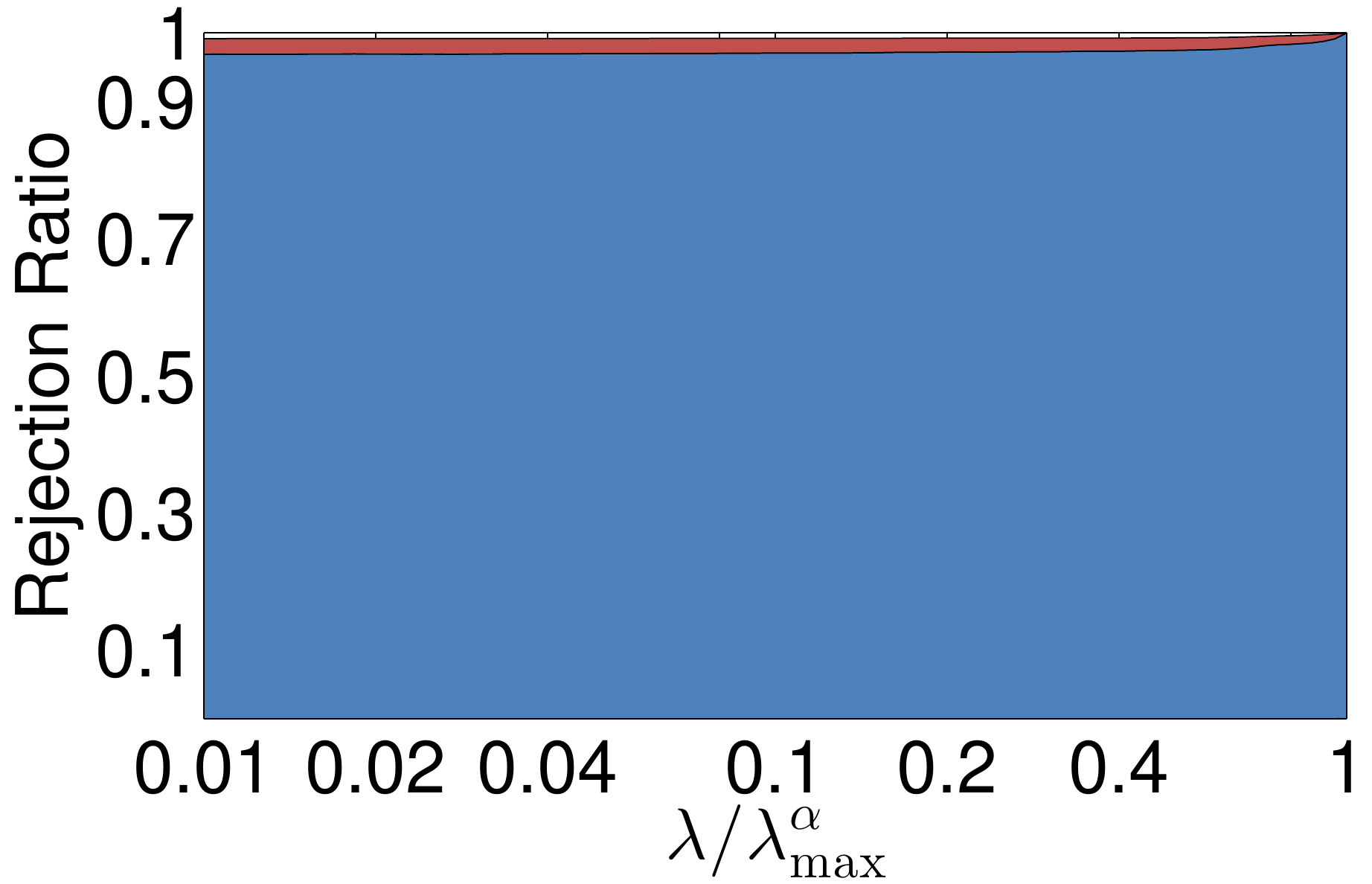}
		}
		\subfigure[$\alpha=\tan(15^{\circ})$] { \label{fig:ADNI_WMV_15}
			\includegraphics[width=0.22\columnwidth]{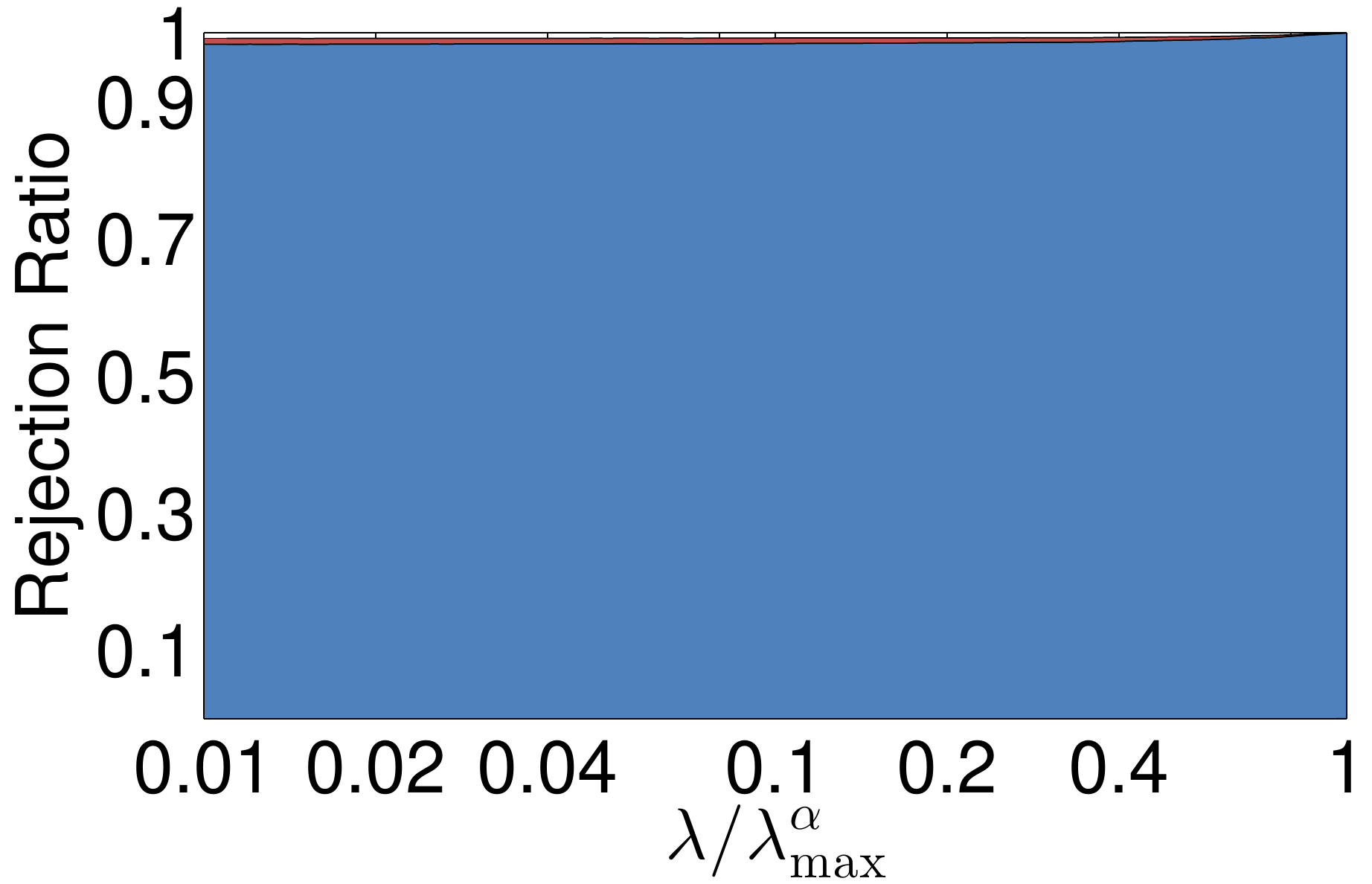}
		}
		\subfigure[$\alpha=\tan(30^{\circ})$] { \label{fig:ADNI_WMV_30}
			\includegraphics[width=0.22\columnwidth]{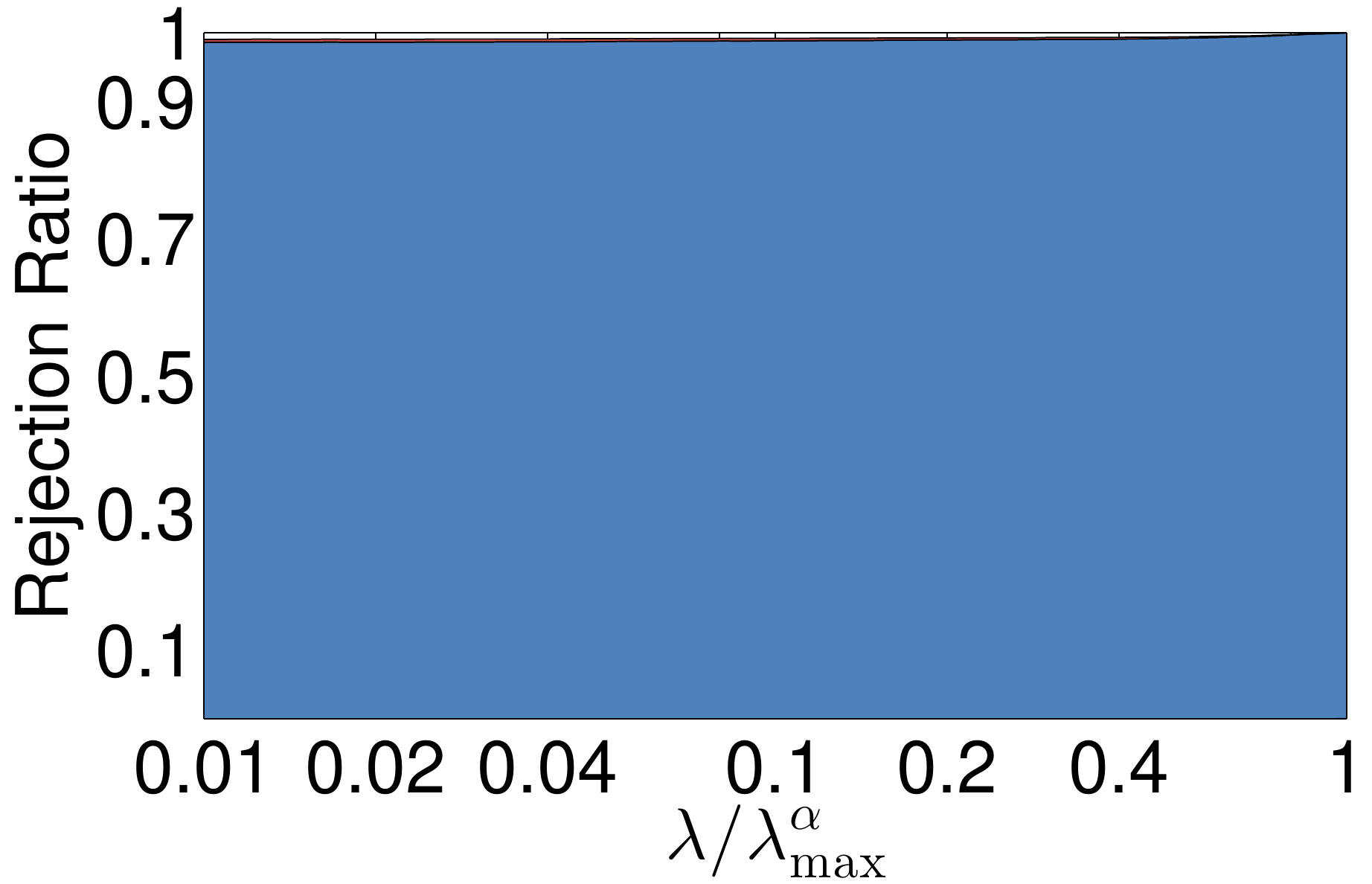}
		}\\[-0.35cm]
		\subfigure[$\alpha=\tan(45^{\circ})$] { \label{fig:ADNI_WMV_45}
			\includegraphics[width=0.22\columnwidth]{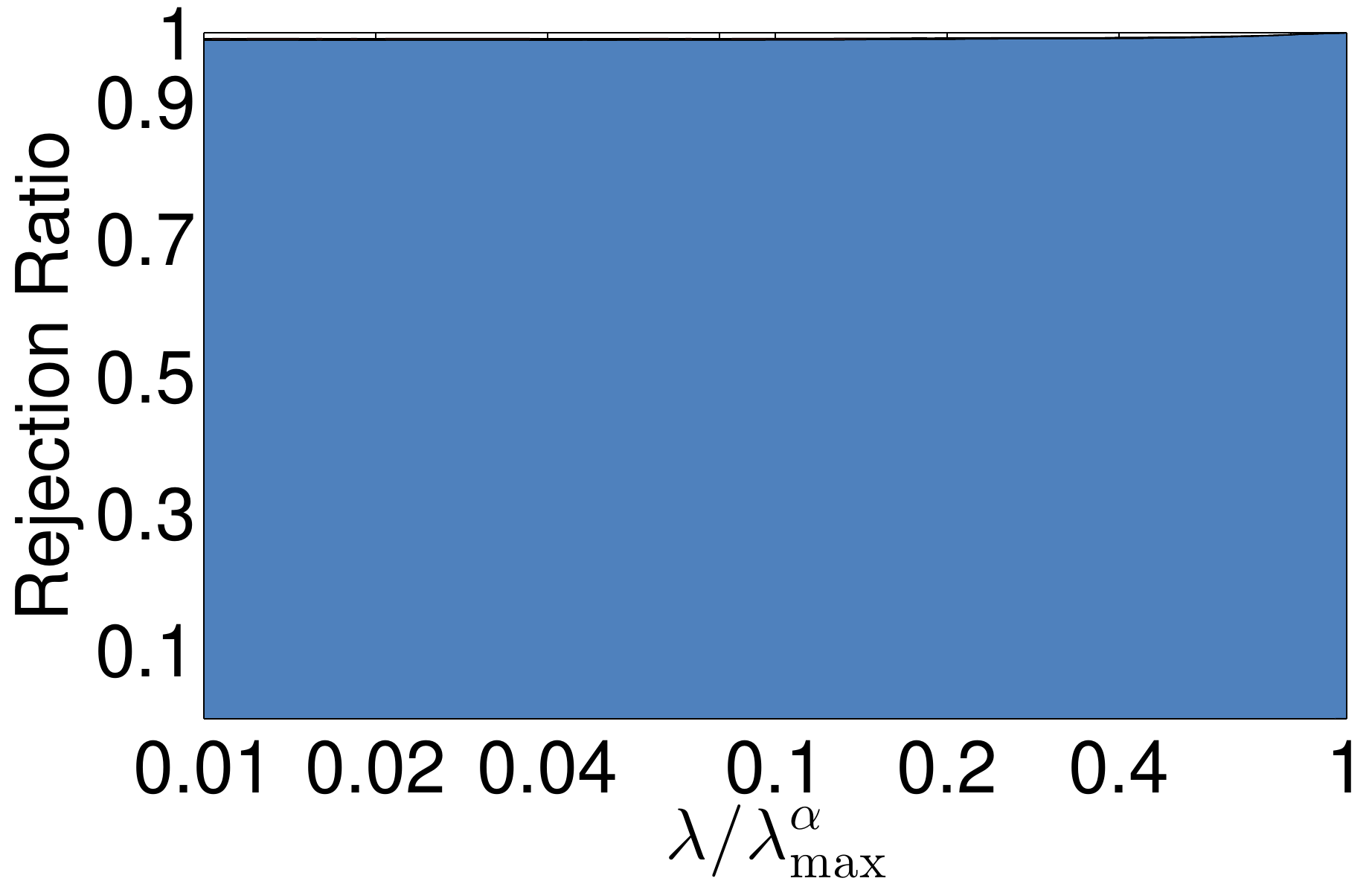}
		}
		\subfigure[$\alpha=\tan(60^{\circ})$] { \label{fig:ADNI_WMV_60}
			\includegraphics[width=0.22\columnwidth]{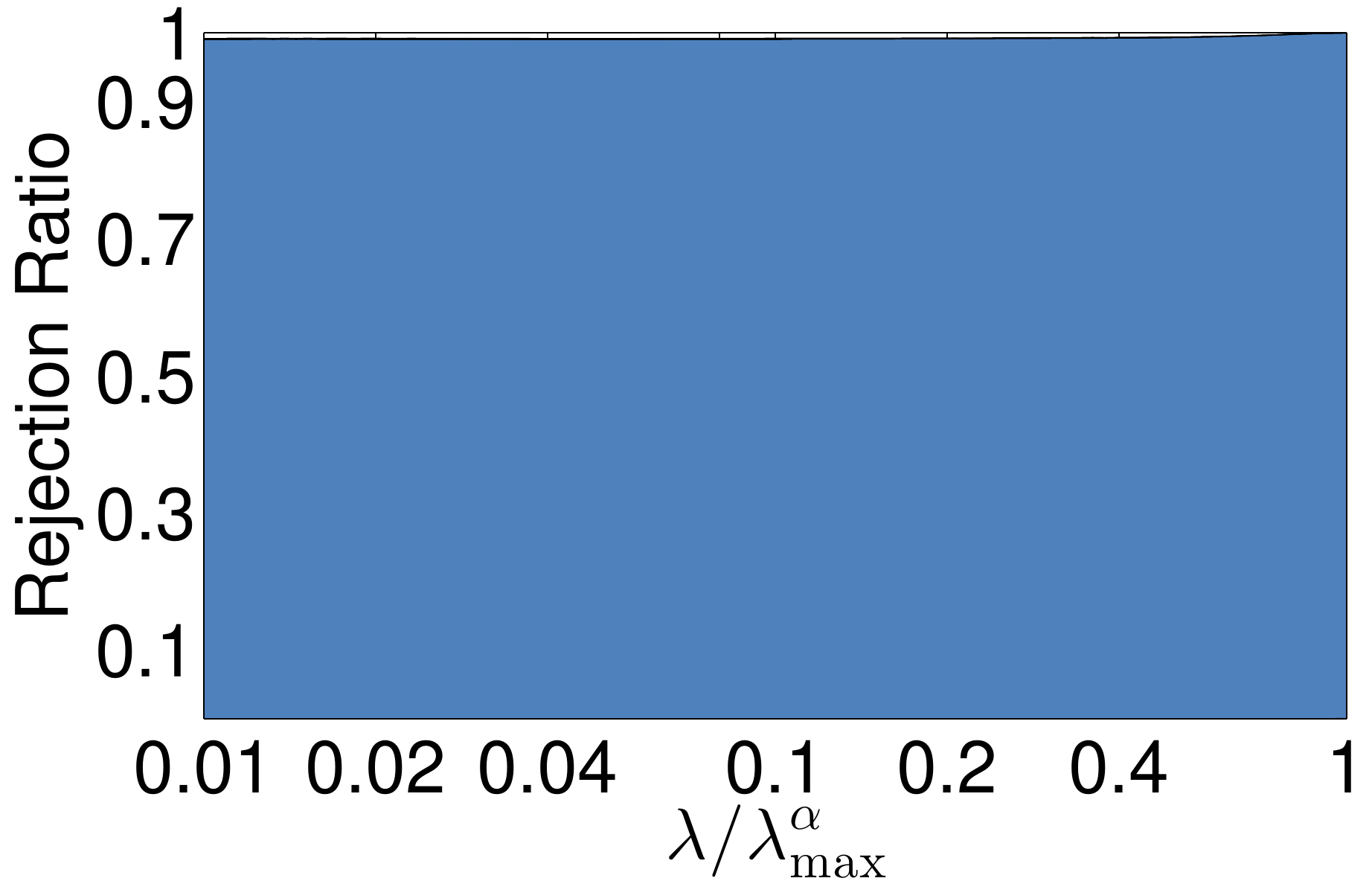}
		}
		\subfigure[$\alpha=\tan(75^{\circ})$] { \label{fig:ADNI_WMV_75}
			\includegraphics[width=0.22\columnwidth]{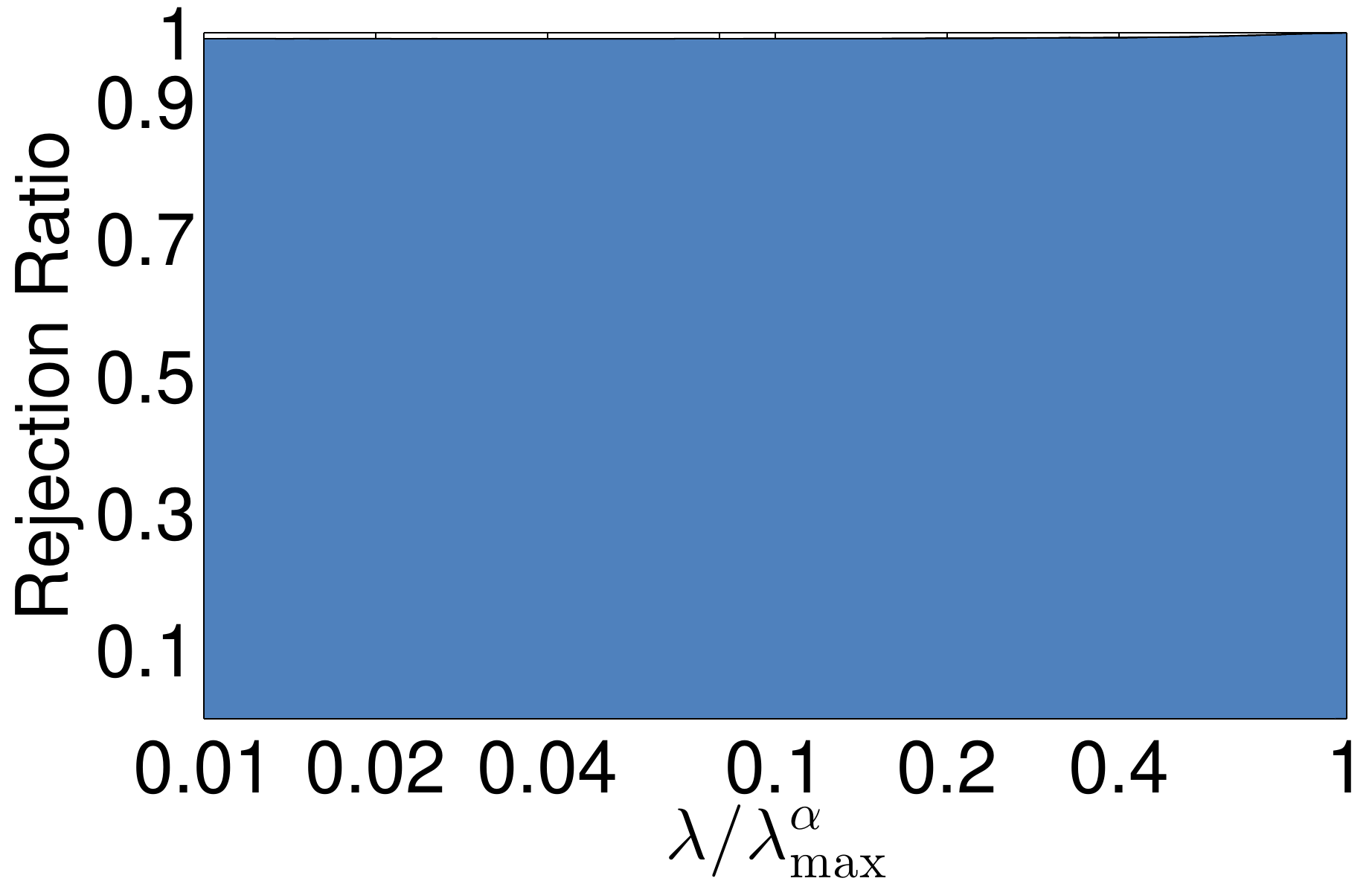}
		}
		\subfigure[$\alpha=\tan(85^{\circ})$] { \label{fig:ADNI_WMV_85}
			\includegraphics[width=0.22\columnwidth]{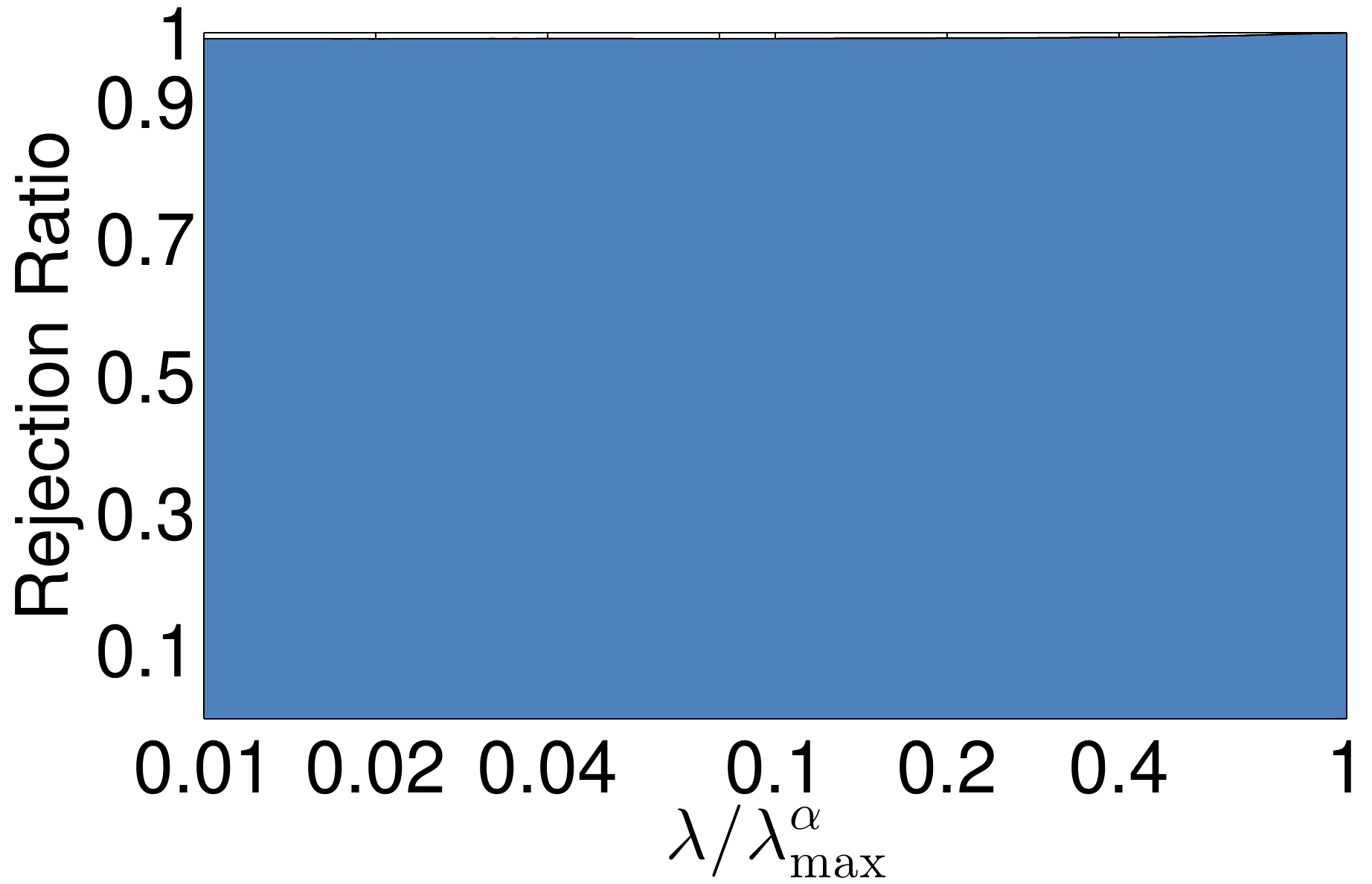}
		}
	}\\[-0.4cm]
	\caption{Rejection ratios of TLFre on the ADNI data set with white matter volume as response. }
	\label{fig:ADNI_WMV}
\end{figure*}

\setlength{\tabcolsep}{.18em}
\begin{table}
	\begin{center}
		\caption{Running time (in seconds) for solving SGL along a sequence of $100$ tuning parameter values of $\lambda$ equally spaced on the logarithmic scale of ${\lambda}/{\lambda_{\rm max}^{\alpha}}$ from $1.0$ to $0.01$ by (a): the solver \cite{SLEP} without screening; (b): the solver combined with TLFre. We perform experiments on the ADNI data sets. The response vectors are GMV and WMV, respectively.
		}\label{table:TLFre_runtime_real}\vspace{3mm}
		\begin{scriptsize}
			\def\arraystretch{1.25}
			\begin{tabular}{ l c|c|c|c|c|c|c|c| }
				\cline{2-9}
				& \multicolumn{1}{|c|}{$\alpha$} &  $\tan(5^{\circ})$ & $\tan(15^{\circ})$ &  $\tan(30^{\circ})$ & $\tan(45^{\circ})$ & $\tan(60^{\circ})$ & $\tan(75^{\circ})$ & $\tan(85^{\circ})$ \\
				\cline{2-9}\\ [-2.5ex]\hline
				\multicolumn{1}{|r|}{\multirow{4}{*}{ADNI+GMV}}  & \multicolumn{1}{|c|}{solver} & 30652.56 & 30755.63 & 30838.29  & 31096.10 & 30850.78 & 30728.27 & 30572.35 \\ \cline{2-9}
				\multicolumn{1}{|r|}{}  & \multicolumn{1}{|c|}{TLFre} & 64.08 & 64.56 & 64.96 & 65.00 & 64.89 & 65.17 & 65.05\\\cline{2-9}
				\multicolumn{1}{|r|}{} & \multicolumn{1}{|c|}{TLFre+solver} & 372.04 & 383.17 & 386.80 & 402.72 & 391.63 & 385.98 & 382.62\\\cline{2-9}
				\multicolumn{1}{|r|}{}  & \multicolumn{1}{|c|}{\textbf{speedup}} & \textbf{82.39} & \textbf{80.27} & \textbf{79.73} & \textbf{77.22} & \textbf{78.78} & \textbf{79.61} & \textbf{79.90} \\\hline\hline
				\multicolumn{1}{|r|}{\multirow{4}{*}{ADNI+WMV}}  & \multicolumn{1}{|c|}{solver} & 29751.27 & 29823.15 & 29927.52 & 30078.62 & 30115.89 & 29927.58 & 29896.77 \\\cline{2-9}
				\multicolumn{1}{|r|}{}  & \multicolumn{1}{|c|}{TLFre} & 62.91 & 63.33 & 63.39 & 63.99 & 64.13 & 64.31 & 64.36\\\cline{2-9}
				\multicolumn{1}{|r|}{}  & \multicolumn{1}{|c|}{TLFre+solver} & 363.43 & 364.78 & 386.15 & 393.03 & 395.87 & 400.11 & 399.48\\\cline{2-9}
				\multicolumn{1}{|r|}{} & \multicolumn{1}{|c|}{\textbf{speedup}} & \textbf{81.86} & \textbf{81.76} & \textbf{77.50} & \textbf{76.53} & \textbf{76.08} & \textbf{74.80} & \textbf{74.84} \\\hline
			\end{tabular}
		\end{scriptsize}
	\end{center}
	\vspace{-0.1in}
	
\end{table}

We perform experiments on the Alzheimer's Disease Neuroimaging Initiative (ADNI) data set (\url{http://adni.loni.usc.edu/}). The data matrix consists of $747$ samples with $426040$  \mbox{single} nucleotide polymorphisms (SNPs), which are divided into $94765$ groups. The response vectors are the grey matter volume (GMV) and white matter volume (WMV), respectively. 

The figures in the upper left corner of \figref{fig:ADNI_GMV} and \figref{fig:ADNI_WMV} show the plots of $\lambda_1^{\rm max}(\lambda_2)$ (see Corollary \ref{cor:lambdamx_SGL}) and the sampled parameter values of $\alpha$ and $\lambda$. The other figures present the rejection ratios of (\ref{rule:L1}) and (\ref{rule:L2}) by blue and red regions, respectively. We can see that almost all of the inactive groups/features are discarded by TLFre. The rejection ratios of $r_1+r_2$ are very close to $1$ in all cases. Table \ref{table:TLFre_runtime_real} shows that TLFre leads to a very significant speedup (about $80$ times). In other words, the solver without screening needs about eight and a half hours to solve the $100$ SGL problems for each value of $\alpha$. However, combined with TLFre, the solver needs only six to eight minutes. Moreover, we can observe that the computational cost of TLFre is negligible compared to that of the solver without screening. This demonstrates the efficiency of TLFre.

\subsection{DPC for Nonnegative Lasso}\label{ssec:exp_nnlasso}

In this experiment, we evaluate the performance of DPC on two synthetic data sets and six real data sets. We integrate DPC with the solver \cite{SLEP} to solve the nonnegative Lasso problem along a sequence of $100$ parameter values of $\lambda$ equally spaced on the logarithmic scale of $\lambda/\lambda_{\textup{max}}$ from $1.0$ to $0.01$. The two synthetic data sets are the same as the ones we used in Section \ref{sssec:exp_syn_SGL}. To construct $\beta^*$, we first randomly select $10$ percent of features. The corresponding components of $\beta^*$ are populated from a standard Gaussian and the remaining ones are set to 0. We list the six real data sets and the corresponding experimental settings as follows.
\begin{enumerate}
	\item[a)] \textbf{Breast Cancer data set} \citep{West2001,Shevade2003}: this data set contains $7129$ gene expression values of $44$ tumor samples (thus the data matrix $\mathbf{X}$ is of $44\times 7129$). The response vector $\mathbf{y}\in\{1,-1\}^{44}$
	contains the binary label of each sample.
	\item[b)] \textbf{Leukemia data set} \citep{Armstrong2002}: this data set contains $11225$ gene expression values of $52$ samples ($\mathbf{X}\in\mathbb{R}^{52\times 11225}$). The response vector $\mathbf{y}$ contains the binary label of each sample. 
	\item[c)] \textbf{Prostate Cancer data set} \citep{Petricoin2002}: this data set contains $15154$ measurements of 132 patients ($\mathbf{X}\in\mathbb{R}^{132\times 15154}$). By protein mass spectrometry, the features are indexed by time-of-flight values, which are related to the mass over charge ratios of the constituent proteins in the blood. The response vector $\mathbf{y}$ contains the binary label of each sample. 
	\item[d)] \textbf{PIE face image data set} \citep{Sim2003,Cai2007}: this data set contains $11554$ gray face images (each has $32\times 32$ pixels) of $68$ people, taken under different poses, illumination conditions and expressions. In each trial, we first randomly pick an image as the response $\mathbf{y}\in\mathbb{R}^{1024}$, and then use the	remaining images to form the data matrix $\mathbf{X}\in\mathbb{R}^{1024\times 11553}$. We run $100$ trials and report the average performance of DPC.
	\item[e)] \textbf{MNIST handwritten digit data set} \citep{Lecun1998}: this data set contains grey images of scanned handwritten digits (each has $28\times 28$ pixels). The training and test sets contain $60,000$ and $10,000$ images, respectively. We first randomly select $5000$ images for each digit from the training set and get a data matrix ${\bf X}\in\mathbb{R}^{784\times 50000}$. Then, in each trial, we randomly select an image from the testing set as the response ${\bf y}\in\mathbb{R}^{784}$. We run $100$ trials and report the average performance of the screening rules.
	\item[f)] \textbf{Street View House Number (SVHN) data set} \citep{Netzer2001}: this data set contains color images of street view house numbers (each has $32\times 32$ pixels), including $73257$ images for training and $26032$ for testing. In each trial, we first randomly select an image as the response ${\bf y}\in\mathbb{R}^{3072}$, and then use the remaining ones to form the data matrix ${\bf X}\in\mathbb{R}^{3072\times 99288}$. We run $20$ trials and report the average performance.
\end{enumerate}

We present the \emph{rejection ratios}---the ratio of the number of inactive features identified by DPC to the actual number of inactive features---in \figref{fig:DPC_rej_ratio}. We also report the running time of the solver with and without DPC, the time for running DPC, and the corresponding \emph{speedup} in Table \ref{table:DPC_runtime}.

\figref{fig:DPC_rej_ratio} shows that DPC is very effective in identifying the inactive features even for small parameter values: the rejection ratios are very close to $100\%$ for the entire sequence of parameter values on the eight data sets. Table \ref{table:DPC_runtime} shows that DPC leads to a very significant speedup on all the data sets. 
Take MNIST as an example. The solver without DPC takes $50$ minutes to solve the $100$ nonnegative Lasso problems. However, combined with DPC, the solver only needs $10$ seconds. The speedup gained by DPC on the MNIST data set is thus more than $300$ times. Similarly, on the SVHN data set, the running time for solving the $100$ nonnegative Lasso problems by the solver without DPC is close to seven hours. However, combined with DPC, the solver takes less than two minutes to solve all the $100$ nonnegative Lasso problems, leading to a speedup about $230$ times. Moreover, we can also observe that the computational cost of DPC is very low---which is negligible compared to that of the solver without DPC.

\begin{figure*}[t]
	\centering{
		\subfigure[Synthetic 1] { \label{fig:syn_no_cor_nnlasso}
			\includegraphics[width=0.22\columnwidth]{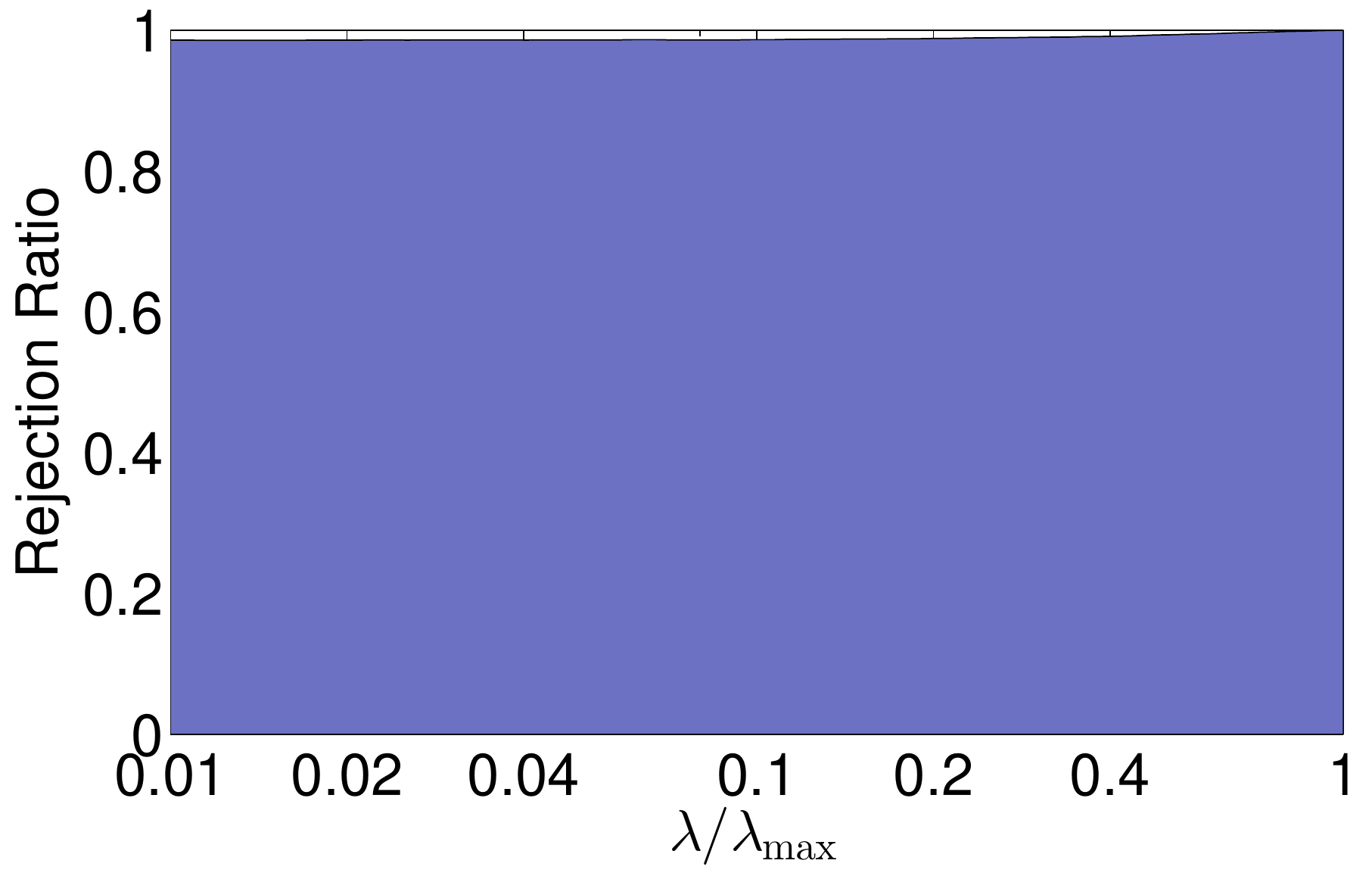}
		}
		\subfigure[Synthetic 2] { \label{fig:syn_pos_vcor_nnlasso}
			\includegraphics[width=0.22\columnwidth]{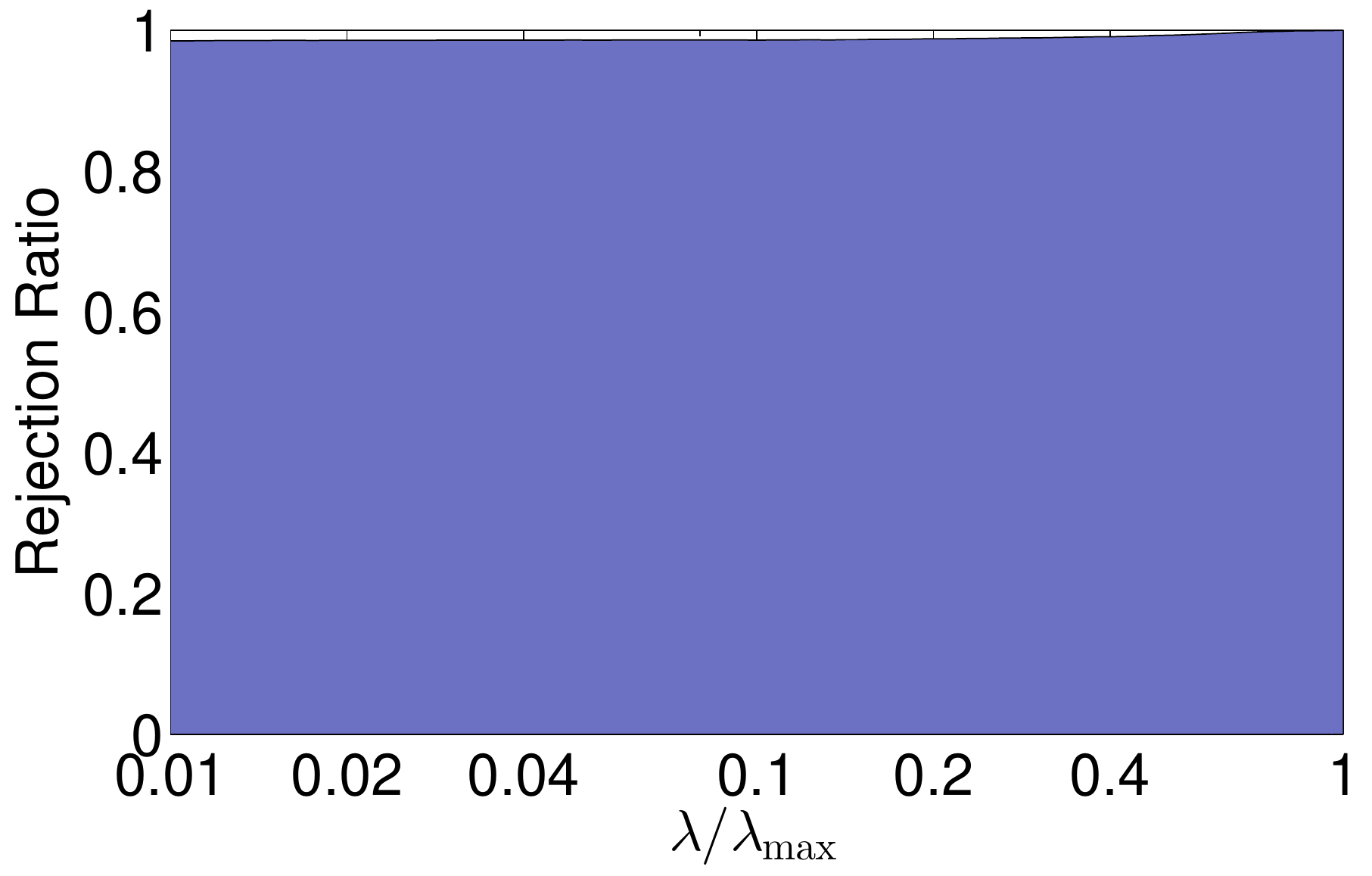}
		}
		\subfigure[Breast Cancer] { \label{fig:breastcancer}
			\includegraphics[width=0.22\columnwidth]{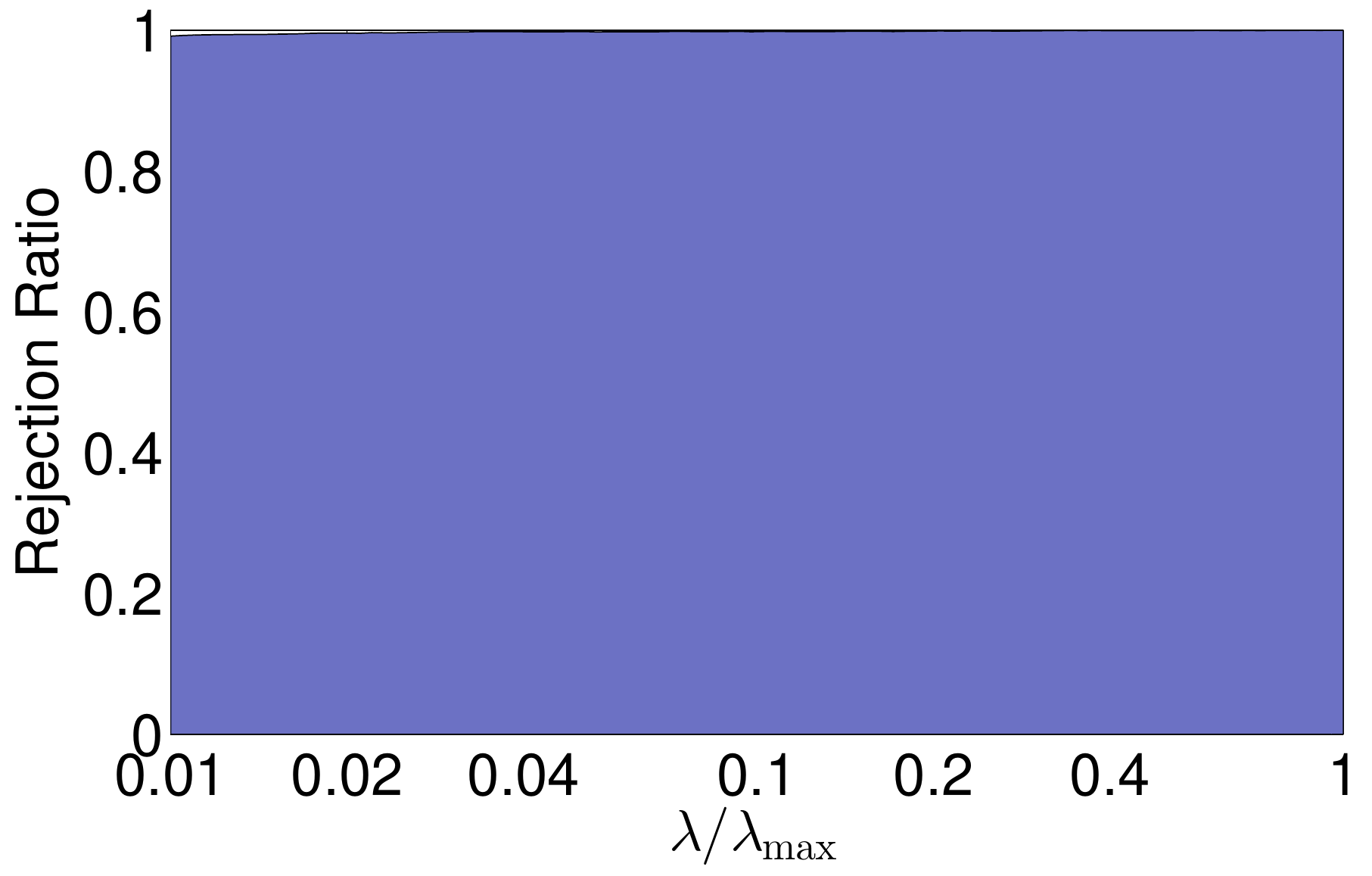}
		}
		\subfigure[Leukemia] { \label{fig:leukemia}
			\includegraphics[width=0.22\columnwidth]{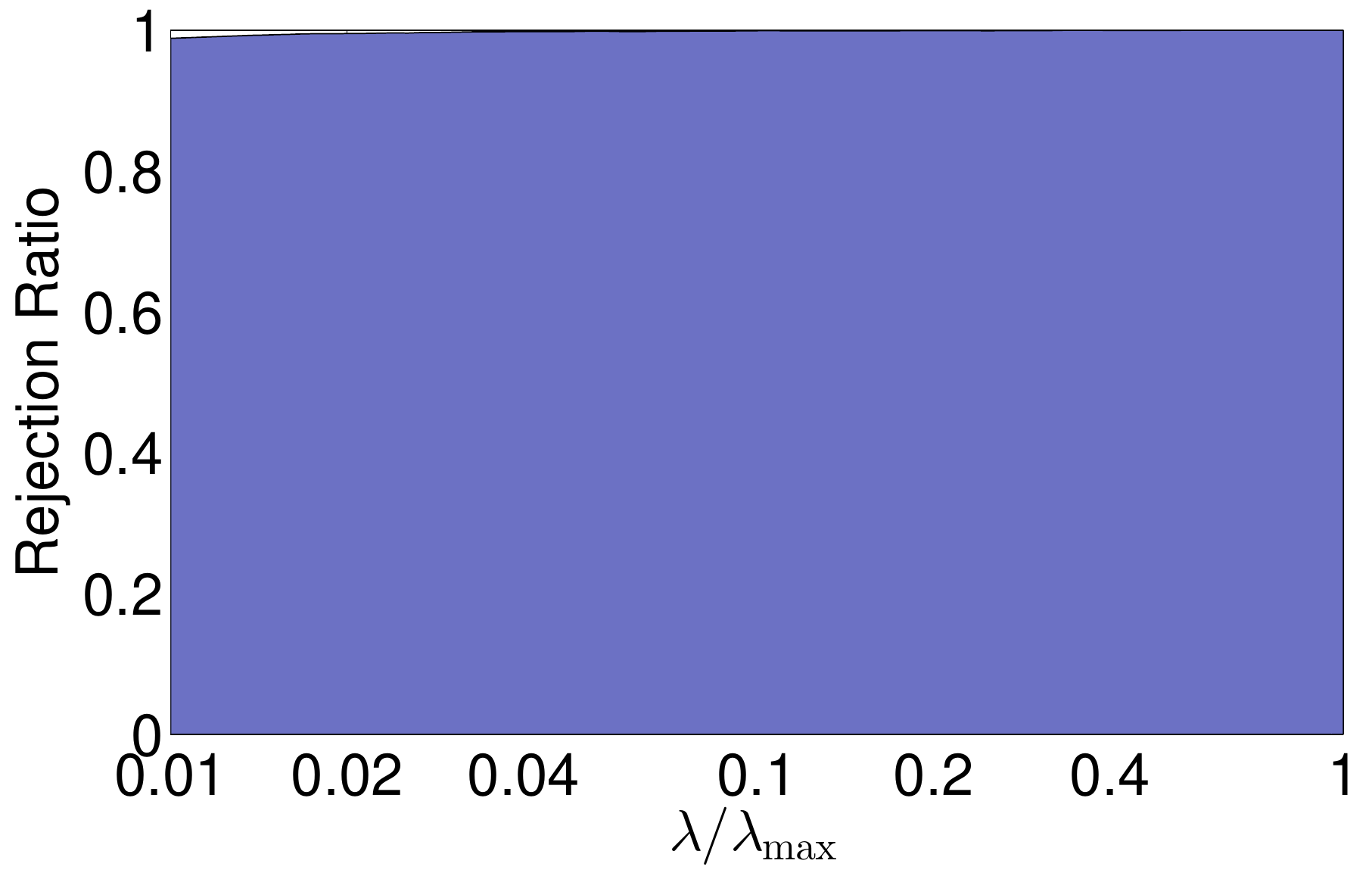}
		}\\[-0.3cm]
		\subfigure[Prostate Cancer] { \label{fig:prostatecancer}
			\includegraphics[width=0.22\columnwidth]{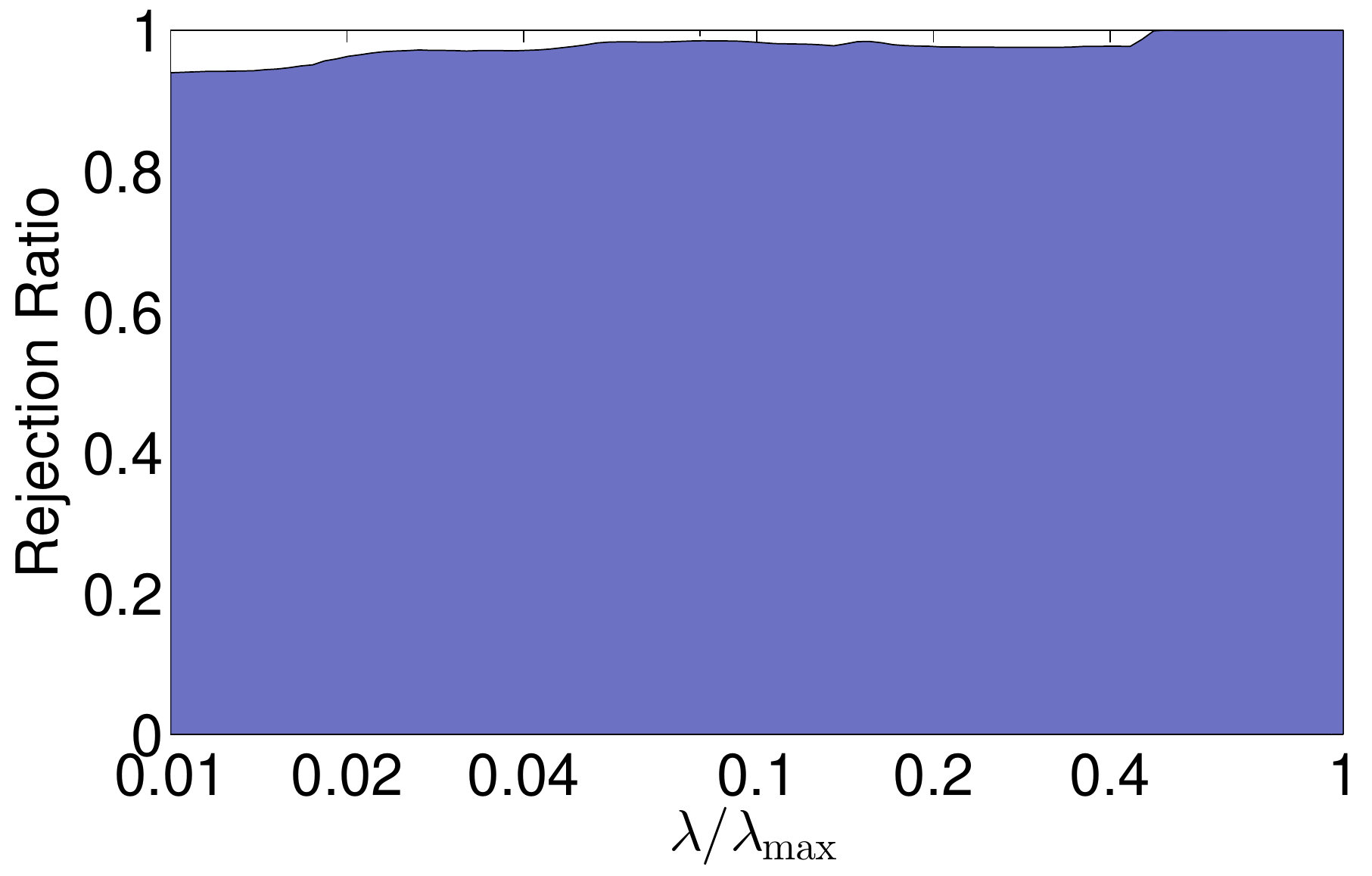}
		}
		\subfigure[PIE] { \label{fig:pie}
			\includegraphics[width=0.22\columnwidth]{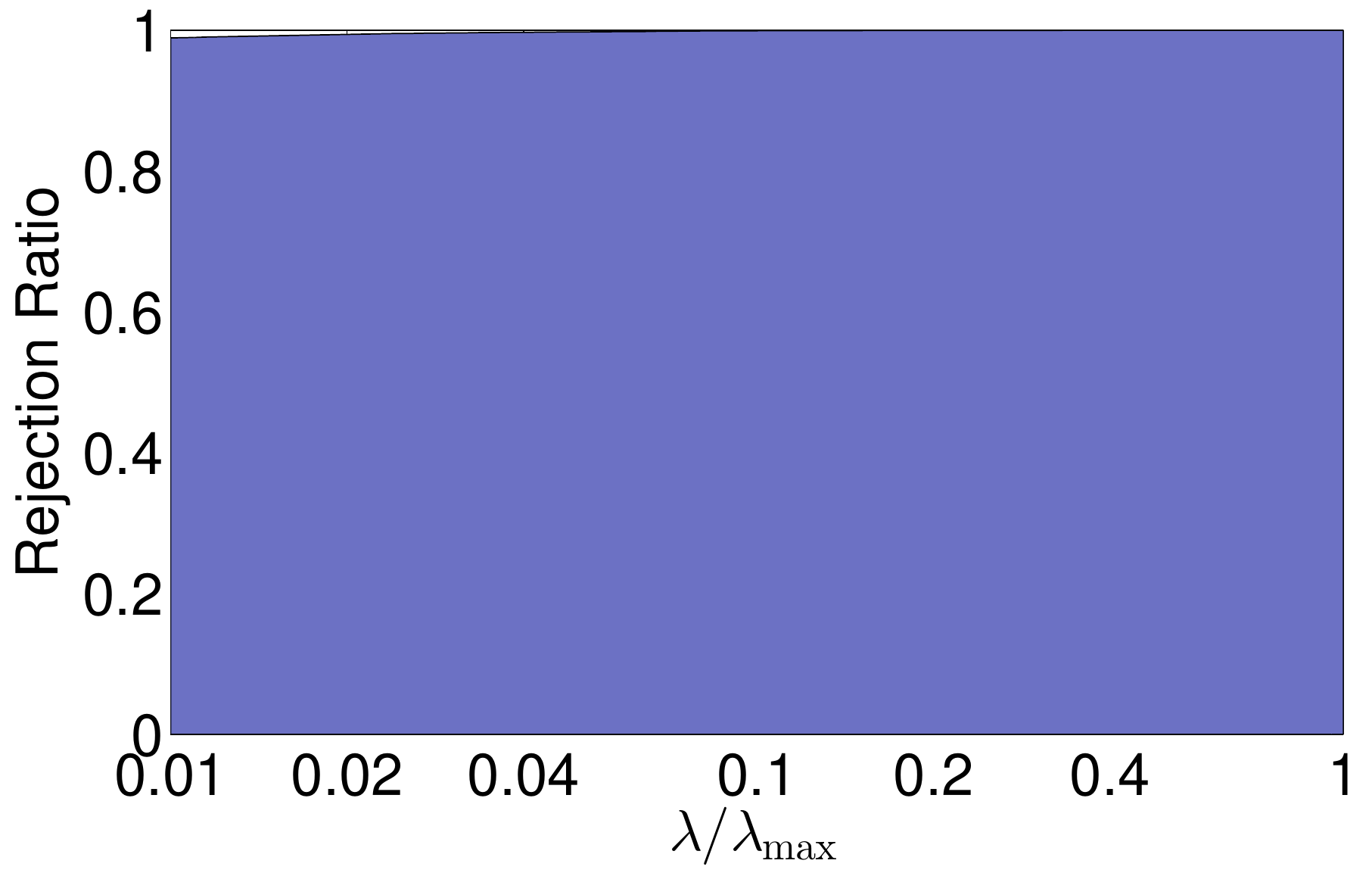}
		}
		\subfigure[MNIST] { \label{fig:mnist}
			\includegraphics[width=0.22\columnwidth]{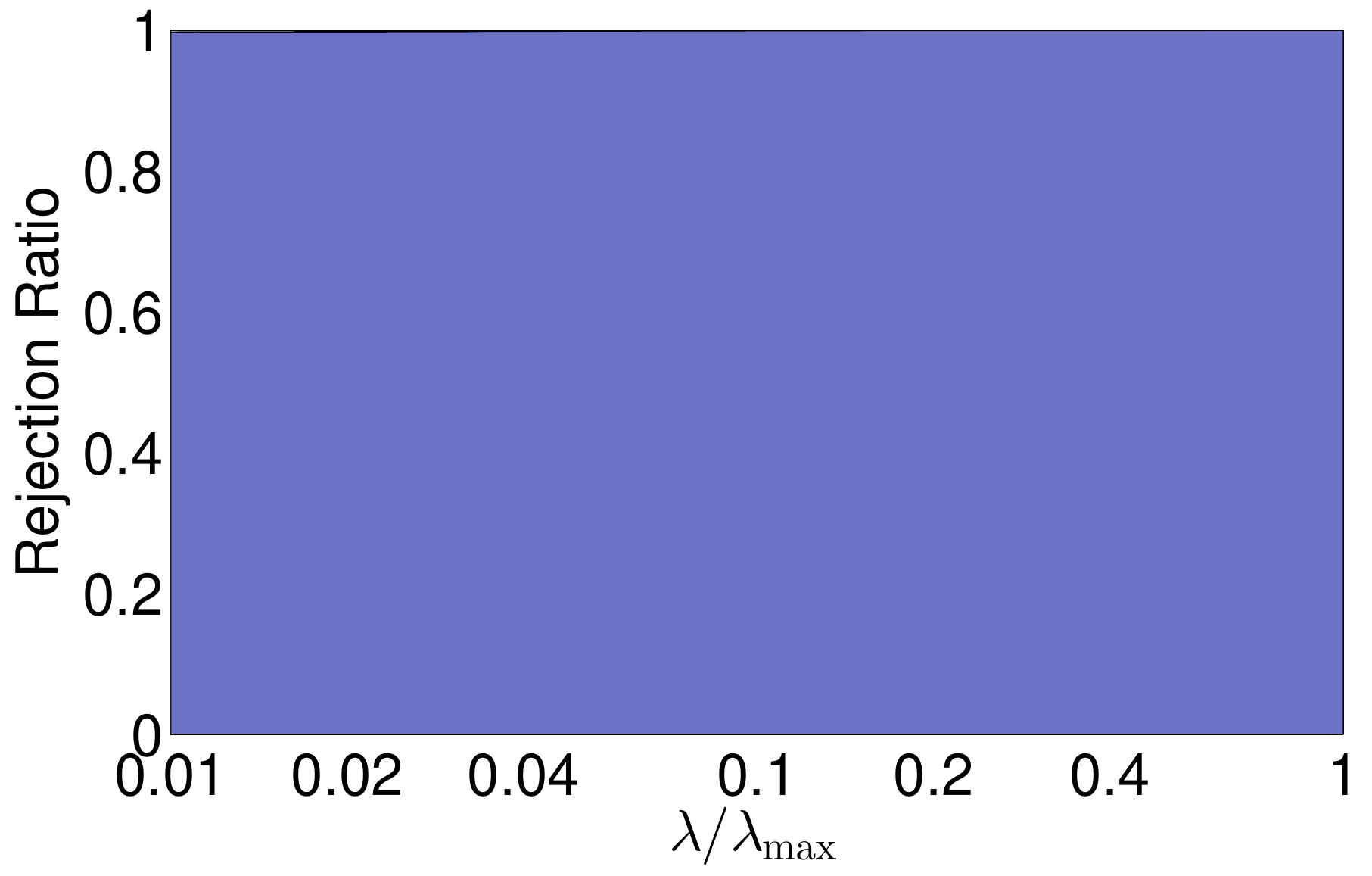}
		}
		\subfigure[SVHN] { \label{fig:svhn}
			\includegraphics[width=0.22\columnwidth]{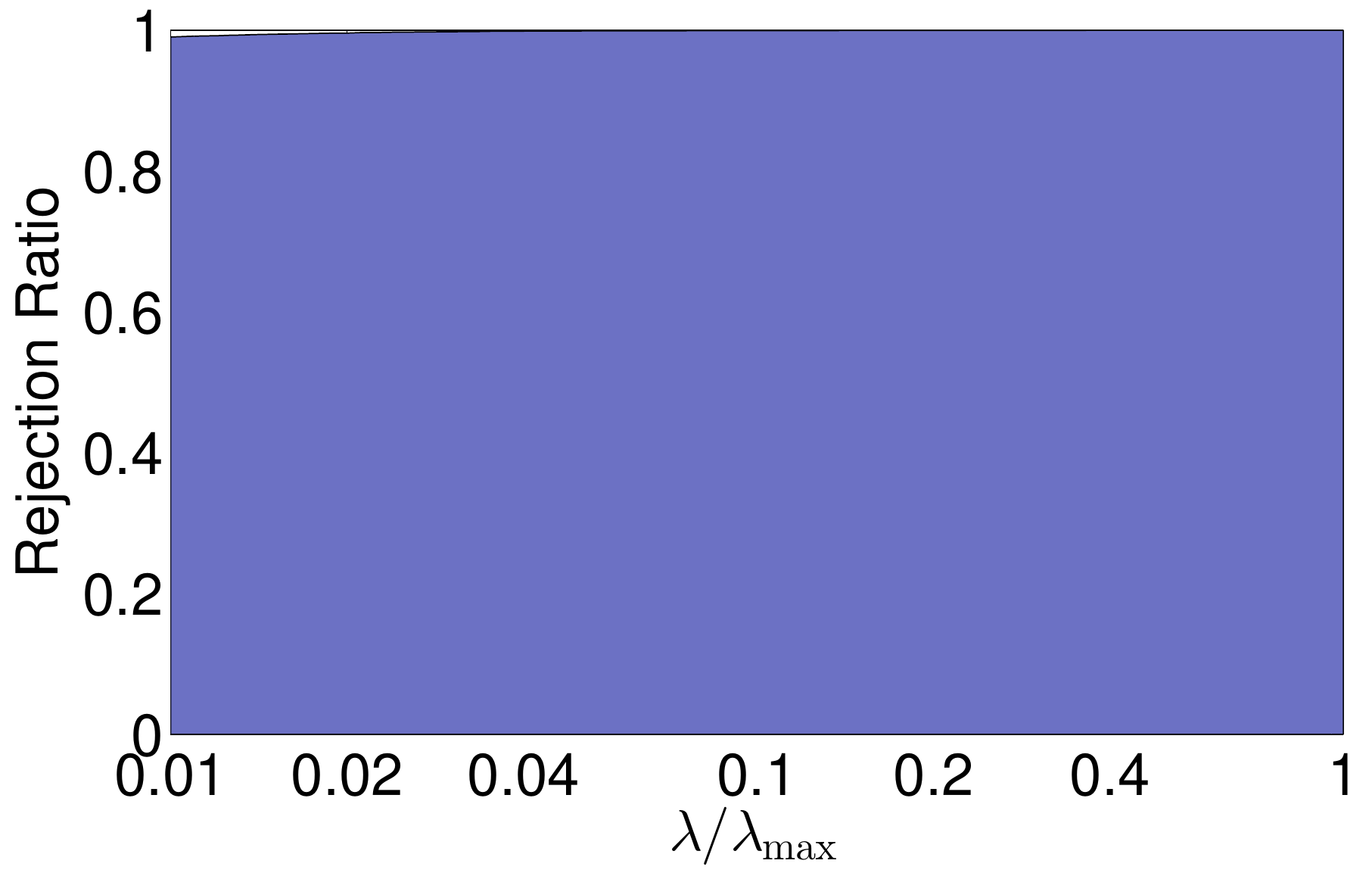}
		}
	}\\[-0.3cm]
	\caption{Rejection ratios of DPC on eight data sets. }
	\vspace{-0.15in}
	\label{fig:DPC_rej_ratio}
\end{figure*}

\setlength{\tabcolsep}{.18em}
\begin{table}
	\begin{center}
		\caption{Running time (in seconds) for solving nonnegative Lasso along a sequence of $100$ tuning parameter values of $\lambda$ equally spaced on the logarithmic scale of ${\lambda}/{\lambda_{\rm max}^{\alpha}}$ from $1.0$ to $0.01$ by (a): the solver \cite{SLEP} without screening; (b): the solver combined with DPC. 
		}\label{table:DPC_runtime}\vspace{3mm}
		\begin{scriptsize}
			\def\arraystretch{1.25}
			\begin{tabular}{ l c|c|c|c|c|c|c|c| }
				\cline{2-9}
				 &  \multicolumn{1}{|c|}{Synthetic 1} & Synthetic 2 &  Breast Cancer & Leukemia & Prostate Cancer & PIE & MNIST & SVHN\\
				\cline{2-9}\\ [-2.5ex]\hline
				\multicolumn{1}{|c|}{solver} & 218.37 & 204.06 & 23.40  & 34.04 & 187.82 & 674.04 & 3000.69 & 24761.07\\ \hline
				\multicolumn{1}{|c|}{DPC} & 0.31 & 0.29 & 0.03 & 0.06 & 0.23 & 1.16 & 3.53 & 30.59\\\hline
				\multicolumn{1}{|c|}{DPC+solver} & 5.52 & 6.10 & 2.18 & 3.37 & 6.37 & 5.01 & 9.31 & 104.93\\\hline
				\multicolumn{1}{|c|}{\textbf{speedup}} & \textbf{39.56} & \textbf{33.45} & \textbf{10.73} & \textbf{10.10} & \textbf{29.49} & \textbf{134.54} & \textbf{322.31} & \textbf{235.98}\\\hline
			\end{tabular}
		\end{scriptsize}
	\end{center}
	\vspace{-0.1in}
\end{table}

\section{Conclusion}
In this paper, we propose a novel feature reduction method for SGL via decomposition of convex sets. We also derive the set of parameter values that lead to zero solutions of SGL. To the best of our knowledge, TLFre is the first method which is applicable to sparse models with multiple sparsity-inducing regularizers. More importantly, the proposed approach provides novel framework for developing screening methods for complex sparse models with multiple sparsity-inducing regularizers, e.g., $\ell_1$ SVM that performs both sample and feature selection, fused Lasso and tree Lasso with more than two regularizers. To demonstrate the flexibility of the proposed framework, we develop the DPC screening rule for the nonnegative Lasso problem. Experiments on both synthetic and real data sets demonstrate the effectiveness and efficiency of TLFre and DPC. 
We plan to generalize the idea of TLFre to $\ell_1$ SVM, fused Lasso and tree Lasso, which are expected to consist of multiple layers of screening.

\clearpage
\newpage
{
	\small
	\bibliographystyle{abbrv}
	\bibliography{refs}
}
\clearpage
\newpage

\appendix

\section{Sparse-Group Lasso}

\subsection{The Lagrangian Dual Problem of SGL}\label{subsection:Lagrangian_supplement}

We derive the dual problem of SGL in (\ref{prob:dual_SGL_Lagrangian}) via the Lagrangian multiplier method.

By introducing an auxiliary variable
\begin{align}\label{eqn:auxiliary_SGL}
	\mathbf{z}=\mathbf{y}-\sum_{g=1}^G\mathbf{X}_g\beta_g,
\end{align}
the SGL problem in (\ref{prob:SGL}) becomes:
\begin{align*}
	\min_{\beta}\,\left\{\frac{1}{2}\|\mathbf{z}\|^2+\alpha\lambda\sum_{g=1}^G\sqrt{n_g}\|\beta_g\|+\lambda\|\beta\|_1:\mathbf{z}=\mathbf{y}-\sum_{g=1}^G\mathbf{X}_g\beta_g\right\}.
\end{align*}
Let $\lambda\theta$ be the Lagrangian multiplier, the Lagrangian function is
\begin{align}\label{eqn:Lagrangian_SGL}
	L(\beta,\mathbf{z};\theta)=&\frac{1}{2}\|\mathbf{z}\|^2+\alpha\lambda\sum_{g=1}^G\sqrt{n_g}\|\beta_g\|+\lambda\|\beta\|_1+\langle\lambda\theta,\mathbf{y}-\sum_{g=1}^G\mathbf{X}_g\beta_g-\mathbf{z}\rangle\\
	=&\alpha\lambda\sum_{g=1}^G\sqrt{n_g}\|\beta_g\|+\lambda\|\beta\|_1-\lambda\langle\theta,\sum_{g=1}^G\mathbf{X}_g\beta_g\rangle+\frac{1}{2}\|\mathbf{z}\|^2-\lambda\langle\theta,\mathbf{z}\rangle+\lambda\langle\theta,\mathbf{y}\rangle.
\end{align}
Let
\begin{align*}
	f_1(\beta)=&\sum_{g=1}^G\,f_1^g(\beta_g)=\sum_{g=1}^G\,\left(\alpha\lambda\sqrt{n_g}\|\beta_g\|+\lambda\|\beta_g\|_1-\lambda\langle\theta,\mathbf{X}_g\beta_g\rangle\right),\\
	f_2(\mathbf{z})=&\frac{1}{2}\|\mathbf{z}\|^2-\lambda\langle\theta,\mathbf{z}\rangle.
\end{align*}
To derive the dual problem, we need to minimize the Lagrangian function with respect to $\beta$ and $\mathbf{z}$. In other words, we need to minimize $f_1$ and $f_2$, respectively. We first consider
\begin{align*}
	\min_{\beta_g}\,f_1^g(\beta_g)=\alpha\lambda\sqrt{n_g}\|\beta_g\|+\lambda\|\beta\|_1-\lambda\langle\theta,\mathbf{X}_g\beta_g\rangle.
\end{align*}
By the Fermat's rule, we have
\begin{align}\label{eqn:opt_cond0_minf1_SGL}
	0\in\partial f_1^g(\beta_g)=\alpha\lambda\sqrt{n_g}\partial \|\beta_g\|+\lambda\partial \|\beta_g\|_1-\lambda\mathbf{X}_g^T\theta,
\end{align}
which leads to
\begin{align}\label{eqn:opt_cond_minf1_SGL}
	\mathbf{X}_g^T\theta=\alpha\sqrt{n_g}\zeta_1+\zeta_2,\hspace{2mm}\zeta_1\in\partial \|\beta_g\|,\hspace{1mm}\zeta_2\in\partial \|\beta_g\|_1.
\end{align}
By noting that
\begin{align*}
	\langle\zeta_1,\beta_g\rangle=\|\beta_g\|,\hspace{2mm}\langle\zeta_2,\beta_g\rangle=\|\beta_g\|_1,
\end{align*}
we have
\begin{align*}
	\langle\mathbf{X}_g^T\theta,\beta_g\rangle=\alpha\sqrt{n_g}\partial \|\beta_g\|+\partial \|\beta_g\|_1.
\end{align*}
Thus, we can see that
\begin{align}\label{eqn:minf1_SGL}
	0=\min_{\beta_g}\,f_1^g(\beta_g).
\end{align}
Moreover, because $\zeta_1\in\partial \|\beta_g\|,\hspace{1mm}\zeta_2\in\partial \|\beta_g\|_1$, \eqref{eqn:opt_cond_minf1_SGL} implies that
\begin{align}\label{eqn:dual_feasible_lagrangian_SGL}
	\mathbf{X}_g^T\theta\in \alpha\sqrt{n_g}\mathcal{B}+\mathcal{B}_{\infty}.
\end{align}

To minimize $f_2$, the Fermat's rule results in
\begin{align}\label{eqn:opt_cond_minf2_SGL}
	\mathbf{z}=\lambda\theta,
\end{align}
and thus
\begin{align}\label{eqn:minf2_SGL}
	-\frac{\lambda^2}{2}\|\theta\|^2=\min_{\mathbf{z}}\,f_2(z).
\end{align}

In view of \eqref{eqn:Lagrangian_SGL}, \eqref{eqn:minf1_SGL}, \eqref{eqn:minf2_SGL} and \eqref{eqn:dual_feasible_lagrangian_SGL}, the dual problem of SGL can be written as
\begin{align*}
	\sup_{\theta}\,\left\{\frac{1}{2}\|\textbf{y}\|^2-\frac{1}{2}\left\|\theta-\frac{\textbf{y}}{\lambda}\right\|^2:\mathbf{X}_g^T\theta\in \alpha\sqrt{n_g}\mathcal{B}+\mathcal{B}_{\infty},\,g=1,\ldots,G\right\},
\end{align*}
which is equivalent to (\ref{prob:dual_SGL_Lagrangian}).

Recall that $\beta^*(\lambda,\alpha)$ and $\theta^*(\lambda,\alpha)$ are the primal and dual optimal solutions of SGL, respectively. By \eqref{eqn:auxiliary_SGL}, \eqref{eqn:opt_cond0_minf1_SGL} and \eqref{eqn:opt_cond_minf2_SGL}, we can see that the KKT conditions are
\begin{align*}
	\lambda\theta^*(\lambda,\alpha)=&\mathbf{y}-\mathbf{X}\beta^*(\lambda,\alpha),\\
	\mathbf{X}_g^T\theta^*(\lambda,\alpha)\in& \alpha\sqrt{n_g}\partial \|\beta_g^*(\lambda,\alpha)\|+\partial \|\beta_g^*(\lambda,\alpha)\|_1,\hspace{1mm}g=1,\ldots,G.
\end{align*}

\end{document}